\documentclass[12pt]{article}
\usepackage[utf8]{inputenc}

\usepackage[english]{babel}

\usepackage[a4paper,top=3cm,bottom=2cm,left=3cm,right=3cm,marginparwidth=1.75cm]{geometry}

\usepackage{xcolor}
\definecolor{byw}{RGB}{255,0,0}

\usepackage{amsmath,amssymb,amsthm}
\usepackage{algorithm}
\usepackage{algpseudocode}
\usepackage{graphicx}
\usepackage{bbm,enumerate}
\usepackage[colorlinks=true, allcolors=blue]{hyperref}
\usepackage{bbm}
\usepackage{comment}
\usepackage[export]{adjustbox}
\usepackage{subcaption}
\usepackage{natbib}
\usepackage{array}
\renewcommand{\hat}{\widehat}

\newcommand{\bfm}[1]{\ensuremath{\mathbf{#1}}}

\def\be{\bfm e}     \def\EE{\mathbb{E}}

   \def\bI{\bfm I}

     \def\PP{\mathbb{P}}
     
     \def\RR{\mathbb{R}}

   \def\bW{\bfm W}  
   \def\bX{\bfm X}  
   \def\bY{\bfm Y}

\def\calA{{\cal  A}}

\def\calF{{\cal  F}}

\def\calI{{\cal  I}} 
\def\calJ{{\cal  J}}

\def\calN{{\cal  N}} 
\def\calO{{\cal  O}} 
\def\calP{{\cal  P}}

\def\calS{{\cal  S}}

\DeclareMathOperator{\var}{var}

\newcommand{\tP}{\Tilde{P}}

\newcommand{\tF}{\widetilde{F}}
\newcommand{\tQ}{\widetilde{Q}}
\newcommand{\tf}{\widetilde{f}}

\newcommand{\tb}{\widetilde{b}}

\newcommand{\tU}{\widetilde{U}}
\newcommand{\tV}{\widetilde{V}}
\newcommand{\tL}{\widetilde{L}}
\newcommand{\tR}{\widetilde{R}}
\newcommand{\tu}{\widetilde{u}}
\newcommand{\tv}{\widetilde{v}}
\newcommand{\tl}{\widetilde{l}}
\newcommand{\tr}{\widetilde{r}}
\newcommand{\ts}{\widetilde{s}}
\newcommand{\tx}{\widetilde{x}}
\newcommand{\ty}{\widetilde{y}}
\newcommand{\Tau}{\mathcal{T}}

\newcommand{\hL}{\widehat{L}}
\newcommand{\hS}{\widehat{S}}
\newcommand{\hP}{\widehat{P}}
\newcommand{\hU}{\widehat{U}}
\newcommand{\hV}{\widehat{V}}
\newcommand{\hF}{\widehat{F}}
\newcommand{\hSigma}{\widehat{\Sigma}}

\newcommand{\bars}{\bar{s}}
\newcommand{\barmu}{\bar{\mu}}

\newcommand{\DL}{\Delta_L}
\newcommand{\DS}{\Delta_S}
\newcommand{\barr}{\Bar{r}}

\newcommand{\Tr}{\operatorname{Tr}}
\newcommand{\inco}{\sqrt{\frac{\mu r}{p}}}
\newcommand{\invinco}{\sqrt{\frac{p}{\mu r}}}

\theoremstyle{plain}
\newtheorem{theorem}{Theorem}[section]
\newtheorem{definition}[theorem]{Definition}
\newtheorem{corollary}[theorem]{Corollary}
\newtheorem{assumption}[theorem]{Assumption}
\newtheorem{lemma}[theorem]{Lemma}
\newtheorem{remark}{Remark}
\newtheorem*{example}{Example}

\newtheorem{proposition}[theorem]{Proposition}

\title{Structured Matrix Learning under Arbitrary Entrywise Dependence and Estimation of Markov Transition Kernel\footnote{Partially supported by NSF grants DMS-2210833, DMS-2053832, DMS-2052926 and ONR grant N00014-22-1-2340.}}
\author{Jinhang Chai\thanks{ORFE, Princeton University}\and Jianqing Fan\footnotemark[\value{footnote}]}
\date{}
\begin{document}
\maketitle

\begin{abstract}
The problem of structured matrix estimation has been studied mostly under strong noise dependence assumptions. This paper considers a general framework of noisy low-rank-plus-sparse matrix recovery, where the noise matrix may come from any joint distribution with arbitrary dependence across entries. We propose an incoherent-constrained least-square estimator and prove its tightness both in the sense of deterministic lower bound and matching minimax risks under various noise distributions. To attain this, we establish a novel result asserting that the difference between two arbitrary low-rank incoherent matrices must spread energy out across its entries; in other words, it cannot be too sparse, which sheds light on the structure of incoherent low-rank matrices and may be of independent interest. We then showcase the applications of our framework to several important statistical machine learning problems. In the problem of estimating a structured Markov transition kernel, the proposed method achieves the minimax optimality and the result can be extended to estimating the conditional mean operator, a crucial component in reinforcement learning. The applications to multitask regression and structured covariance estimation are also presented. We propose an alternating minimization algorithm to approximately solve the potentially hard optimization problem. Numerical results corroborate the effectiveness of our method which typically converges in a few steps. 
\end{abstract}



\section{Introduction}
Over the past decade, the structures of low-rank-plus-sparse matrices and their variants have received widespread attention in statistical machine learning and been applied to a wide spectrum of scientific problems \citep{chen2021spectral}, including matrix denoising \citep{chandrasekaran2011rank,agarwal2012noisy}, robust PCA \citep{candes2011robust,wright2009robust}, multitask regression \citep{yuan2007dimension,rohde2011estimation}, factor models \citep{fan2013large,fan2021robust}, among others. Broadly speaking, the general goal of this problem is to recover the statistical signal, which consists of a low-rank matrix plus a sparse matrix from noisy observations. Nonetheless, prevailing works typically impose strong assumptions on the structure of noise. They either posit the noise matrix exhibits independence across its entries \citep{candes2010matrix,wang2022robust}, or that each entry possesses either a light tail distribution or a sub-Gaussian distribution~\citep{chen2021bridging}. Some also assume the signal-to-noise ratio is lower bounded \citep{chen2021bridging}.

Though near-optimal rates can be achieved, they rely on strong noise assumptions, which are likely to fail in the real world. To illustrate the limitations of the existing theory, let us take a detour to look at a particular problem that motivated our study. Specifically, we focus on Markov chain transition matrix estimation \citep{hao2018learning}, under the assumption that the transition matrix admits a low-rank-plus-sparse structure. Statistically, the goal is to estimate the ground-truth transition matrix $P^{\star}\in \mathbb{R}^{p\times p}$ given a single trajectory of Markov chain $\{X_0, X_1,\cdots X_n\}$\footnote{Our theory can be readily extended to the case with multiple trajectories.}, under the assumption that the ground-truth transition matrix follows the low-rank-plus-sparse form $P^{\star}=L^{\star}+S^{\star}$, where $L^{\star}$ is a low-rank matrix and $S^{\star}$ is a sparse matrix with a small number of nonzero elements. This structure is motivated by the fact that unstructured $P^{\star}$ becomes notably difficult to estimate when the dimension $p$ is large \citep{wolfer2019minimax,wolfer2021statistical} and exact low-rankness can be overly restrictive in many applications.

The problem of dependence of noises across entries arises in solving the estimation problem. To see this, let us consider a slightly simpler problem first:  estimating the frequency matrix $F^\star$.\footnote{For technical issues, we will focus on recovering the frequency matrix $F^{\star}$, whose $(i, j)$ entry is $P(X_t = i, X_{t+1}= j)$ and is given by $\operatorname{diag}(\pi)P^{\star}$ with $\pi$ being the invariant distribution of the Markov chain. The rationale behind using the frequency matrix is that it has a special form of empirical average, which facilitates the application of concentration inequalities. On the other hand, there is an approximate one-to-one correspondence between the enlarged frequency matrix class (defined in \eqref{def:enlarged-freq}) and the transition matrix class, and the results of estimating $P^{\star}$ can be readily deduced from the results of estimating $F^{\star}$ in the enlarged frequency matrix class.}  After constructing the empirical frequency matrix $\tF$, the problem can be recast as estimating the $F^{\star}$ from its noisy version $\tF = F^\star + W$ with the noise matrix $W = \tF-F^{\star}$. One can readily see that the entries of $W$ are dependent, invalidating most existing works on matrix denoising.

Dependence across entries of the noise matrix is not unique in this problem. It also appears in covariance matrix estimation and factor analysis, which will be further expounded later. In fact, independence across entries is too restrictive in many scientific problems.  The limitations of existing theory in solving these kinds of problems naturally give rise to the following challenging question:
\emph{
Can we establish general results of low-rank-plus-sparse matrix denoising that can account for arbitrary noise?}

To overcome this dependence issue, we aim to derive deterministic results with no assumption on the noise. While the main idea can be conveyed by focusing on a simple observation model, we resort to general observation models satisfying the restricted strong convexity condition as defined in \cite{agarwal2012noisy} to ensure a broader range of applicability. To tackle the estimation, we formulate a least-square type estimator by solving an incoherent-constrained optimization problem. The objective function is the squared error in Frobenius norm and the low-rank-plus-sparse structure is encoded as hard constraints in the optimization. This is a distinct deviation from previous literature.

We briefly go through the main techniques in proving the deterministic bound. First of all, by virtue of the optimality condition, the Frobenius loss of the optimal solution can be bounded by the Frobenius loss of the ground truth. However, a cross term appears in further deduction, which hinders a straightforward application of Cauchy-Schwarz inequality to get the desired bound. While \cite{agarwal2012noisy} directly bounds the cross term using the spikiness assumption which they impose, a trailing term emerges in their final result, leading to possible estimation inconsistency even when the sample size goes to infinity. It remains unclear whether consistent estimation is achievable in such a dependent noise setting. This work positively answers this question via a different route. To fulfill the goal, we prove a pivotal separation lemma indicating that the difference between two incoherent low-rank matrices cannot be too sparse. In other words, incoherent low-rank matrices are in $L_0$-norm sense well-separated. Therefore, the cross-term can be bounded by a fraction of the squared estimation error in the Frobenius norm, which leads to the desired results. 

At the core of our results lies an insightful yet intuitive Lemma~\ref{lemma:main} whose proof is quite technically involved. We divide the entire proof into several parts for clarity. The proof proceeds from simple and special cases to more general ones. Specifically, we first prove the case with equal singular values in Lemma C.5 in the supplementary material and then generalize it to allow for different singular values by utilizing techniques from linear programming. The result in the lemma is also optimal, in the sense that the minimum sparsity level of arbitrary difference of two incoherent matrices cannot be higher up to a constant, as illustrated in Proposition C.7 in the supplementary material.

Subsequently, we apply this overarching theory to several problem instances. For the Markov transition matrix estimation problem, we instantiate our deterministic bound to get a probabilistic bound, which provably attains the minimax lower bound \citep{zhang2019spectral}.  We further employed the framework to address other types of problems, encompassing multitask learning with low-rank-plus-sparse structure and robust structured covariance estimation.

To complement the acquired upper bounds, we turn to lower bounds. Specifically, in our deterministic result, we argue the rates of all terms are of optimal order by proving a matching deterministic lower bound. In various noise settings including the Markov transition noise setting, we establish minimax risk, thereby showcasing the results' tightness.

To summarize, our contributions are three-fold.  First of all, we formulate a general low-rank-plus-sparse matrix denoising problem that encompasses most statistical problems involving matrix estimation where the noise distribution across entries can be arbitrarily dependent and distributed. We develop an incoherent-constrained least-square optimization method that provably attains minimax risk in various scenarios, and the method can be implemented in practice using an alternating minimization algorithm.  Secondly, we find that the difference between two arbitrary low-rank incoherent matrices has energy spread across entries, namely, cannot be too sparse. This intermediate result sheds light on the geometric structure of low-rank incoherent matrix space as a subset of the entire matrix space. Thirdly, we employ our general theory to solve an important problem of estimating the Markov transition matrix. We prove that our method achieves the minimax lower bound. Further, we extend the study to the estimation of the conditional mean, which is an essential part in reinforcement learning. We prove that this structure promotes the statistical rate if the value vector is random and independent of the Markov chain data.  In addition, we apply our result to multitask regression and structured covariance estimation. An extension of our general theory to rectangular matrix estimand is presented in Appendix F in the supplementary material.

Most of the matrix completion or Robust PCA literature hinges on strong assumptions on the error matrix, such as independence across entries or sub-Gaussianity. It would be interesting to derive similar deterministic results in a partially observable setting, and we leave that to future work.

\subsection{Notation}
Before proceeding, we clarify some 
notations used in this paper. For an operator $\mathfrak{X}$, denote $\mathfrak{X}^{*}$ as its conjugate operator. Denote $\mathbb{R}^{p\times p}$ as the set of $p\times p$ matrices whose elements are real numbers. Denote $\mathbb{R}^{p\times p}_{\operatorname{diag}}$ as the diagonal matrix space in  $\mathbb{R}^{p\times p}$. Denote $\calO^{p\times p}$ as the orthogonal matrix space in $\mathbb{R}^{p\times p}$. We use $\|\cdot\|,\|\cdot\|_F,\|\cdot\|_{\max},\|\cdot\|_1,\|\cdot\|_*,\|\cdot\|_0$ to denote the spectral norm, Frobenius norm, elementwise $\ell_\infty$-norm, elementwise $\ell_1$-norm, nuclear norm and number of nonzero elements, respectively. We also denote $\|\cdot\|_{2,\infty}$ denote the $2$-to-$\infty$ norm, i.e., the maximum of the 2-norm of rows. For two matrices $F,G\in \mathbb{R}^{p\times q}$, denote $\langle F,G\rangle=\operatorname{Tr}(FG^{\top})$ as the matrix inner product in Euclidean distance. For every $p\ge r$, denote $\mathcal{O}_{p,r}$ as the set of $p \times r$ matrices whose columns are normalized and orthogonal, and $\mathcal{O}_{p,r}^\mu$ as in Definition \ref{def:incoherence}. By convention, vectors are usually column vectors, except that discrete probability distributions are row vectors. We use bold case $\mathbf{1}_p$ to denote the vector in $\mathbb{R}^p$ with all elements equal to 1 and $e_i$ to represent the $i^{th}$ standard basis vector whose elements are 0 except the $i^{th}$ coordinate is 1. For two vectors or matrices $u,v$, $u\ge v$ means for each entry $u_i\ge v_i$. We write $f(n)\lesssim g(n)$ or $f(n)=\operatorname{O}(g(n))$ if there exists a constant $C$ such that $f(n)\le Cg(n)$ for all $n\ge n_0$. Similarly, write $f(n)\gtrsim g(n)$ or $f(n)=\Omega(g(n))$ if there exists a constant $C$ such that $f(n)\ge Cg(n)$ for all $n\ge n_0$. $f(n)\asymp g(n)$ or $f(n)=\Theta(g(n))$ denotes that both $f(n)\lesssim g(n)$ and $f(n)\gtrsim g(n)$. We write $f(n,p)=\operatorname{\tilde{O}_p}(g(n,p))$ if there exists some $h(n,p)$ depending on logarithmic factors such that $\lim_{n\rightarrow \infty}\lim_{p\rightarrow\infty}\PP(f(n,p)\le g(n,p)h(n,p))=0$ We sometimes neglect logarithmic terms when specifying in the context. $\mathcal{U}[0,1]$ stands for uniform distribution over $[0,1]$. We denote $\{1,2,\cdots,p\}$ as $[p]$. We use $\calI_{\calA}$ or $\calI(\calA)$ to denote the indicator function of event $\calA$. Constants $C,c,c_0,c_1$ can change from line to line.

\subsection{Outline}
In Section \ref{sec:matrix-decomposition}, we present the problem setup and the method to solve general low-rank-plus-sparse matrix estimation. Section \ref{sec:markov} focuses on the application to the estimation of  frequency and transition matrix estimation.  Section \ref{sec:further-example}  illustrates further the versality of our method using the multi-task regression and robust stuctured matrix estimation.
Section \ref{sec:numerical-results} presents the numerical results. Section \ref{sec:conclusion} concludes the paper. The proofs and related works are delegated to the supplementary material.

\section{Low-rank-plus-sparse matrix estimation}
\label{sec:matrix-decomposition}
In this section, we formulate the low-rank-plus-sparse matrix denoising problem. Consider the general observation model:
\begin{align}
\label{equ:model}
Y&=\mathfrak{X}(\Theta^{\star})+W\\ \label{eq:lowrank}
\Theta^\star&=L^\star+S^\star
\end{align}
where $\Theta^\star\in \RR^{p\times p}$ is the parameter of interest\footnote{As the work is motivated by the transition matrix estimation problem, we focus on the case that the parameter of interest is a square matrix. The extension to the general aspect ratio is relegated to the supplementary material.} and can be written as the sum of a low-rank part $L^\star$ and a sparse part $S^\star$. The $\mathfrak{X}$ is a linear map. The output $Y$ can either be a vector or a matrix as long as it lives in an Euclidean space.


This observation model encompasses a wide range of problems, including the following two important special cases.
\begin{enumerate}
\item {\bf Matrix denoising}.  This corresponds to $\mathfrak{X}=\mathcal{I}$. The Markov transition matrix estimation problem is a specific instance with the parameter $\Theta^\star$ being the ground-truth frequency matrix $F^\star$. Letting $\hat{F}$ be the empirical frequency matrix obtained from the data, we can define noise term as $W=\hat{F}-F^\star$ and write the observation model as $\hat{F}=F^\star+W$.

\item {\bf Multitask regression}. Consider $p$ linear regression problems: $Y_j = \bX \theta_j^\star + W_j$ for $j \in [p]$, where $Y_j$ is the response vector for task $j$. Let $\Theta^\star = (\theta_1^\star, \cdots, \theta_p^\star)$, $\bY = (Y_1, \cdots, Y_p)$, and $\bW = (W_1, \cdots, W_p)$.  Then, $\bY = \bX \Theta^\star + \bW$.  To reduce the number of parameters, we assume a low-rank-plus-sparse structure on $\Theta^\star$ that allows us to pull the information from multiple tasks. This model corresponds to 
$\mathfrak{X}(\Theta)=\bX \Theta$ in \eqref{equ:model} with structure \eqref{eq:lowrank}.
\end{enumerate}

In order for the observation model to reveal enough information on $\Theta^{\star}$, we make the following restricted strong convexity (RSC) condition assumption in the Frobenius norm.  Following the same notation of \cite{agarwal2012noisy}, we define the regularizer $\Phi(\Delta)=\inf\limits_{L+S=\Delta}\|L\|_*+\lambda\|S\|_1$ where $\lambda$ is a given parameter, and the RSC condition is as follows. 
\begin{assumption}[RSC]
\label{asp:RSC}
    There exists a pair of parameters $(\kappa,\tau)$ such that 
    \[\frac{1}{2}\|\mathfrak{X}(\Delta)\|_{\mathrm{F}}^2 \geq \frac{\kappa}{2}\|\Delta\|_{\mathrm{F}}^2-\tau \Phi^2(\Delta)\] for every $\Delta\in \mathbb{R}^{p\times p}$.
\end{assumption}

\begin{remark}
    The RSC establishes a form of approximate identifiability. From the definition of $\Phi$ we can see that $\Phi(\Delta)\le \min(\|\Delta\|_*,\lambda\|\Delta\|_1)$. Therefore, Assumption \ref{asp:RSC} can be viewed as a strengthened variant of RSC with the restricted area being low-rank or sparse matrices. In later proofs, we make use of the fact that the nuclear norm and 1-norm are the dual norm of 2-norm and max-norm, respectively. 
\end{remark}

For the restricted convex condition to be informative, $\kappa$ should not be too small while $\tau$ should not be too large. Specifically, we require the following assumption.

\begin{assumption}
\label{asp:RSC-parameters}
    The parameters of RSC condition satisfies 
    $\kappa \geq 32 \tau \bars\max(1,\lambda^2)$, where $\bars$ is an upper bound of the sparsity level and $\kappa, \tau$ are defined in Assumption~\ref{asp:RSC}.
\end{assumption}

\begin{remark}
    Typically, we require ${\tau}/{\kappa}$ to be a small constant. When the observation map $\mathfrak{X}$ is identity, this assumption automatically holds since $\tau=0$.
\end{remark}

To guarantee the separability of low-rank matrices and sparse matrices, following \cite{candes2010matrix},  we define incoherent matrices in Definition~\ref{def:incoherence} and impose Assumption \ref{asp:L,S} on $L^{\star}$ and $S^{\star}$ in which we assume that the upper bounds are known for $\mu$, $r$ and $s$: $\barmu\ge \mu,\ \barr\ge r,\ \bars\ge s$. The incoherence condition ensures that the singular vectors of a low-rank matrix are not overly concentrated in any particular directions or entries, which is crucial for matrix completion \citep{candes2012exact,chen2015incoherence}. In our setting, it also facilitates separability from the sparse matrix.

\begin{definition}
\label{def:incoherence}
A matrix $U\in \mathcal{O}_{p,r}$ is defined to be $\mu$-incoherent if $\|U\|_{2,\infty}\le \sqrt{\frac{\mu r}{p}}$ and denote as $U\in \mathcal{O}_{p,r}^\mu$.
A rank-r matrix $M\in \RR^{p\times p}$ with SVD $M=U\Sigma V^{\top}$ is defined to be $\mu$-incoherent if both $U$ and $V$ are $\mu$-incoherent.
\end{definition}

\begin{assumption}
\label{asp:L,S}
There exists some constant $c>0$ such that the low-rank matrix $L^{\star}$ and sparse matrix $S^{\star}$ satisfy the following conditions:
\begin{enumerate}
    \item $L^{\star}=U^{\star} \Sigma^{\star} V^{*T}$ is the SVD, where $U^{\star},V^{\star}\in \mathcal{O}^{\mu}_{p,r}$;
    \item $\|S^{\star}\|_0\le \bars\le \frac{p}{c\barmu\barr^3}$.
\end{enumerate}
\end{assumption}

When $\barmu,\barr=\operatorname{O}(1)$, then there is a positive constant $c$ such that the third condition of Assumption \ref{asp:L,S} holds for $\bars\le cp$.

Now we  introduce our incoherent-constrained least-square estimator as follows:
\begin{align}
\label{alg:main}
    \min_{L,S,U,V,\Sigma}& \quad \frac{1}{2}\|Y-\mathfrak{X}(L+S)\|_F^2\\
    \text{s.t.}& \quad L=U\Sigma V^{\top}\nonumber\\
    &\quad U,V\in \mathcal{O}^{\Bar{\mu}}_{p,\barr}\nonumber\\
    &\quad \|S\|_0\le \bars, \nonumber
\end{align}
where $\bar{\mu}$ and $\bars$ are given hyperparameters. In the optimization formulation \eqref{alg:main}, we impose explicitly incoherent low-rankness and sparsity on the constraints.  This is a notable distinction from previous literature in which the incoherence often only appears in the theoretical analysis. All those constraints enable the effective usage of our separation lemma while the nuclear norm penalty can not.

\subsection{Results on low-rank-plus-sparse matrix estimation}
\label{sec:results-matrix-decomposition}
Let $(\hL,\hS,\hU,\hV,\hSigma)$ solve \eqref{alg:main}.  The following theorem states a deterministic and non-asymptotic result for the estimation accuracy of the low-rank and sparse matrices. 

\begin{theorem}
\label{thm:main}
    Let $\Delta_L=\hL-L^{\star}, \Delta_S=\hS-S^{\star}$.
    Under Assumptions \ref{asp:RSC}, \ref{asp:RSC-parameters}, and \ref{asp:L,S}, we have
    \begin{align}
    \label{equ:error-bound}
    \|\DL\|_F^2+\|\DS\|_F^2\le \frac{128}{\kappa^2}\left(\barr\|\mathfrak{X}^*(W)\|^2+\bars \|\mathfrak{X}^* (W)\|_{\max}^2\right).
    \end{align}
\end{theorem}

\begin{remark}
Inequality \eqref{equ:error-bound} holds for every possible realization of noise $W$.  Another notable generality of our results is that we don't make assumptions on the magnitude ratio of $S$ and $L$ compared to the existing literature, for instance, the pervasiveness assumption in the approximate factor model literature \citep{fan2008high,fan2013large}.  When $\mathfrak{X}=\mathcal{I}$, we can choose $\kappa=1$, $\tau=0$ and Assumption \ref{asp:RSC-parameters} holds automatically. Furthermore, when $\barr\lesssim r$ and $\bars\lesssim s$, we can deduce that 
\begin{equation}
\label{equ:error-bound-idendity-observation-model}
\|\DL\|_F^2+\|\DS\|_F^2\lesssim r\|W\|^2+s \|W\|_{\max}^2.
\end{equation}
\end{remark}

The proof of Theorem~\ref{thm:main} relies significantly on the following key separation lemma, with its detailed proof provided in the supplementary material.
\begin{lemma}
\label{lemma:main}
    Let $P=U \Sigma V^{\top}$ and $Q=L \Lambda R^{\top}$ be the SVD of $P,Q\in \mathbb{R}^{p\times p}$, $U,V,L,R\in \mathcal{O}_{p,r}$. Denote $\sigma_1\ge \sigma_2\ge\cdots\ge\sigma_r\ge 0$ and $\lambda_1\ge \lambda_2\ge\cdots\ge\lambda_r\ge 0$ as the diagonal elements of $\Sigma$ and $\Lambda$, respectively.
    Suppose $P$ and $Q$ both satisfy the $\mu$-incoherence condition \ref{def:incoherence}, then we have for $\Delta=P-Q$,
    \begin{equation}
    \label{equ:key-lemma}\frac{\|\Delta\|_{\max}^2}{\|\Delta\|_F^2}\le \frac{\Tilde{c}\mu r^3}{p}\end{equation}
where $\Tilde{c}$ is a fixed constant independent of both $P$ and $Q$.
\end{lemma}

A detailed comparison with existing works under similar settings is given in Appendix A in the supplementary material. For example, see Table 1 there.
Unlike previous works whose bound involves either the condition number or incoherence parameter, the clean deterministic bound \eqref{equ:error-bound-idendity-observation-model} is optimal in the following sense.

As the lower bound applies to general rectangular matrices, we define the class of parameters of interest as $\mathcal{C}=\{(L,S):L,S\in \mathbb{R}^{p\times q},\operatorname{Rank}(L)\le r,|\operatorname{Supp}(S)|\le s\}$ and the error metric $\rho(L,S,\hL,\hS)=\|L-\hL\|_F^2+\|S-\hS\|_F^2$.
For a given noise level $w>0$, define the deterministic lower bound as
\[
e(w):=\inf_{\hL,\hS=f(Y,w)}\sup_{(L,S)\in \mathcal{C},\|W\|_F\le w}\rho(L,S,\hL,\hS),
\]
where the $\inf$ is taken over arbitrary functions of $Y$ and $w$.

If we can find two ways of decomposing $Y=L_1+S_1+W_1=L_2+S_2+W_2$ such that $(L_1,S_1),(L_2,S_2)\in \mathcal{C}$ and $\|W_1\|_F,\|W_2\|_F\le w$, then we say $(L_1,L_2,S_1,S_2,W_1,W_2)$ is a ``valid confusion tuple" at the noise level $w$. As any algorithm only knowing $Y$ and $w$ is incapable of distinguishing between these two cases, we can bound the deterministic lower bound. More rigorously, for fixed $Y,w$, define \[\hat L,\hat S=\arg\inf_{\hL,\hS=f(Y,w)}\sup_{(L,S)\in \mathcal{C}:Y=L+S+W,\|W\|_F\le w}\rho(L,S,\hL,\hS),\] then we have
\begin{eqnarray*}
   e(w) &\ge& \max(\|\hL-L_1\|_F^2+\|\hS-S_1\|_F^2,\|\hL-L_2\|_F^2+\|\hS-S_2\|_F^2)\\
   &\ge & \{\|\hL-L_1\|_F^2+\|\hS-S_1\|_F^2+\|\hL-L_2\|_F^2+\|\hS-S_2\|_F^2\}/{2}\\
   &\ge & \{\|L_1-L_2\|_F^2+\|S_1-S_2\|_F^2\}/{4}=\rho(L_1,S_1,L_2,S_2)/{4}.
\end{eqnarray*}

\begin{theorem}
    For every $W$, let $w=\|W\|_F$. There exists a valid confusion tuple $(L_1,L_2,S_1,S_2,W_1,W_2)$ such that $\rho(L_1,S_1,L_2,S_2)=r\|W\|^2$; there also exists another valid tuple $(L_1',L_2',S_1',S_2',W_1',W_2')$ such that $\rho(L_1',S_1',L_2',S_2')=s\|W\|_{\max}^2$. Therefore, we have 
    $$
       		e(w)\ge \{r\|W\|^2+s\|W\|_{\max}^2\}/{8}.
    $$
\end{theorem}

\begin{proof}
Let $u,v$ be the left and right singular vectors corresponding to the singular value of $\|W\|$.
We construct the first valid confusion tuple as $(L_1,L_2,S_1,S_2,W_1,W_2)=(\|W\|uv^T,0,0,0,W-\|W\|uv^T,W)$. It is easy to verify that $(L_1,S_1),(L_2,S_2)\in \mathcal{C}$ and $\|W_1\|_F,\|W_2\|_F\le w$, as $W_1$ can be viewed as a projection residue of $W_2=W$ to the rank-$r$ matrix space.

Similarly, we construct the second valid confusion tuple. Define $\calI(|W|\ge t)\in \RR^{p\times q}$ with $\calI(|W_{ij}|\ge 0)$ as the $(i,j)$-element. Similarly define $W\calI(|W|\ge t)\in \RR^{p\times q}$ with $W_{i,j}\calI(|W_{ij}|\ge 0)$ as the $(i,j)$-element. Let  $(L_1',L_2',S_1',S_2',W_1',W_2')=(0,0,S_1',0,W-S_1',W)$, where $S_1'=W\mathcal{I}(|W|\ge t)$ for some $t$, is the projection of $W$ to the sparse matrix space with sparsity level at most $s$. It is easy to verify again that $(L_1',S_1'),(L_2',S_2')\in \mathcal{C}$ and $\|W_1'\|_F,\|W_2'\|_F\le w$.
\end{proof}

We argue above that our result is optimal in a deterministic sense. We now show that it is also minimax optimal in the conventional sense. For simplicity,  we again consider the identity observation model  \eqref{equ:model}:
\[Y=L+S+W\]
Recall that $\mathcal{C}=\{(L,S):L,S\in \mathbb{R}^{p\times q},\operatorname{Rank}(L)\le r,|\operatorname{Supp}(S)|\le s\}$ and 
$\rho(L,S,\hL,\hS)=\|L-\hL\|_F^2+\|S-\hS\|_F^2$ . Define the minimax risk
as \[\mathfrak{R}_\rho=\inf_{\mathcal{A}} \sup_{(L,S)\in\mathcal{C}} \mathbb{E}\rho(L,S,\hL^{\mathcal{A}},\hS^{\mathcal{A}}),\]
where the infimum is taken over all possible estimators $\mathcal{A}(Y)=(\hL^{\mathcal{A}},\hS^{\mathcal{A}})$, and the supremum is taken over all possible ground truths $(L,S)$ in $\mathcal{C}$.

In the following theorem, we consider two common cases on noise distributions.
\begin{theorem}
\label{thm:simple-noise-setting-upper-bound}   
The minimax risks for two noise distributions are given as follows.
\begin{enumerate}
\item Consider the case when each row of $W$ is sampled independently from multivariate Gaussian: $W_{i,:}\sim \calN(0,\Sigma_i)$ with $c\sigma^2I\le \Sigma\le C\sigma^2I$. Then the minimax risk in this setting, denoted by $\mathfrak{R}^{1}$,  satisfies
$\mathfrak{R}^{1}\asymp \max\{p,q\}r\sigma^2$ if $s \leq (p+q)r/\log(p+q)$, where we recall that $p\times q$ are the dimensions of $W$.
\item Consider the case when $W$ is supported on arbitrary predetermined $s_0\le s$ entries, and each entry is sampled independently from a Gaussian with variance parameter $\sigma^2$.  Let $E$ be the index matrix that takes $1$ on the predetermined support. If $\|E\|\lesssim \sqrt{\frac{s}{r}}$, the minimax risk in this setting, denoted by $\mathfrak{R}^{2}$, satisfies $\mathfrak{R}^{2}\asymp s\sigma^2$, up to logarithmic factors. In particular, when $E$ is an incomplete permutation matrix(each row and column has at most one $1$), then the condition on $s$ boils down to $s\gtrsim r$.  
\end{enumerate}
In both cases, the upper bound and lower bound match, and estimator \eqref{alg:main} is minimax optimal. 
\end{theorem}

\begin{remark}
    Note that the first case includes a special subcase where all entries of $W$ are independent Gaussian $\calN(0,\sigma^2)$. The independence assumption on the noise can be loosened to weakly independent continuous sub-Gaussian. An example is \cite{zhang2019spectral} where the noise has the form of empirical probability error. The technique should be modified only in calculating the KL divergence which is approximately of the same order as the independent case. Also, for the noise matrix with independent rows but dependent across columns~\citep{agterberg2022entrywise}, as long as the covariance matrices have upper and lower bounded eigenvalues, the upper bound is retained.
\end{remark}

\begin{remark}
For simplicity, we take $\mathfrak{X}$ to be identity. It can be readily extended to $\mathfrak{X}$ satisfying Assumption \ref{asp:RSC} and that $\|\mathfrak{X}\|_{op}\le C_{op}$. The proof of changes in bounding upper bound by $\|\mathfrak{X}^*\|_{op}\le C_{op}$ and calculating the KL divergence of constructed instances and is omitted.
\end{remark}

\subsection{A Block Alternating Minimization Algorithm}
\label{sec:practical-algorithm}
The optimization problem \eqref{alg:main} can be hard to solve when the problem size is large. In this section we propose to use a practical alternating minimization algorithm to solve it approximately. In the literature, a few algorithms have been created for various settings of nonconvex matrix recovery problems \citep{cherapanamjeri2017nearly,yi2016fast,gu2016low,cai2022generalized}, among which alternating minimization has been well known for a long time \citep{ortega1970iterative,bertsekas1989parallel}.
It falls into a broader optimization algorithm class called block coordinate descent and finds lots of statistical applications, such as robust PCA~\citep{gu2016low}, robust regression~\citep{mccullagh2019generalized}, sparse recovery~\citep{daubechies2010iteratively}. The prominent EM algorithm and iteratively reweighted least squares algorithm can both be regarded as variants of alternation minimization. For the sake of illustration, assume $\mathfrak{X}=\mathcal{I}$ throughout this section. 

Before delving into our proposed algorithm, let us search for alternative formulations. A potential drawback of \eqref{alg:main} is that it is not a continuous optimization problem because the support of $S$ is unknown. However, it is easily observed that given $L$, the optimal $S$ can be decided easily. In fact, for an arbitrary matrix $\Delta$, let $|\sigma_1(\Delta)|\ge |\sigma_2(\Delta)|\ge \cdots\ge|\sigma_{p^2}(\Delta)|$ be a ranking of $\Delta$'s elements in absolute value. Denote the convex function $\|\Delta\|_{s,F}:=\sqrt{\sum_{i=1}^{s}|\sigma_i(\Delta)|^2}$ as the squared sum of the $s$-largest elements. Then \eqref{alg:main} can be reformulated as a continuous nonconvex optimization problem by profiling $S$ away as follows:
\begin{align}
\label{alg:continuous-reformulation}
    \min_{U,V,\Sigma}& \quad \frac{1}{2}\|Y-L\|_F^2-\frac{1}{2}\|Y-L\|_{\bars,F}^2\\
    \text{s.t.}& \quad L=U\Sigma V^{\top}\nonumber\\
    &\quad U,V\in \mathcal{O}^{\Bar{\mu}}_{p,\barr}\nonumber
\end{align}

The above problem can be solved by using tools like the Augmented Lagrangian method or others that are guaranteed to find a stationary point.
However, in general \eqref{alg:continuous-reformulation}
is still hard to solve due to nonconvexity, and there are no guarantees for finding global optima. Here, we propose to use the alternating minimization method to solve \eqref{alg:main}, which is simpler to implement and exhibits good performance in our experiments. We slightly rewrite the original problem 
\eqref{alg:main} as
\begin{align*}
    \min_{S,U,V,\Sigma}& \quad f(U,\Sigma,V,S)=\frac{1}{2}\|Y-U\Sigma V^{\top}-S\|_F^2\\
    \text{s.t.}& \quad U,V\in \mathcal{O}^{\Bar{\mu}}_{p,\barr}\\
    & \quad\|S\|_0\le \bars
\end{align*}

Note that there are four blocks of variables $U,\Sigma,V,S$ in the above formulation, and each one is subject to light constraints\footnote{The constraints include orthogonal normality, incoherence, and sparsity.}. The alternating minimization algorithm proceeds as consecutively optimizing over one block variable while keeping the others fixed.

For $S$, it's immediate to see that 
\[\arg\min_{S}f(U,\Sigma,V,S)=\mathcal{T}_{\bars}(Y-U\Sigma V)\] where $\mathcal{T}_{\bars}$ stands for the hard thresholding operator that only keeps the $\bars$ elements of the matrix with largest magnitude and zeros out other elements. Mathematically, let $|X_{\sigma(1)}|\ge |X_{\sigma(2)}|\ge \cdots,|X_{\sigma(p^2)}|$ be the ranking of $X$'s elements in absolute value. Then \footnote{If there are multiple entries whose absolute value is equal to $|X_{\sigma(\bars)}|$, then we let the thresholding process keep exactly $\bars$ nonzero elements, where ties can be broken arbitrarily.
}\begin{equation}
    \label{def:operator-take-s-largest}
    \mathcal{T}_{\bars}(X)=[X_{ij}\mathbbm{1}(|X_{ij}|\ge |X_{\sigma(\bars)}|)]_{i,j=1}^p
\end{equation}

For $\Sigma$, using the orthogonality of $U$ and $V$, we have
\[f(U,\Sigma,V,S)=\frac{1}{2}\|Y-U\Sigma V^{\top}-S\|_F^2=\frac{1}{2}\|U^{\top}(Y-S)V-\Sigma\|_F^2\]
Therefore, we have
\[\arg\min_{\Sigma}f(U,\Sigma,V,S)=\operatorname{Diag}(U^{\top}(Y-S)V),\]
where $\operatorname{Diag}$ is the operator that returns the diagonal part of the matrix.

For the remaining two block variables, by symmetry, it suffices to consider the update formula of $U$. Specifically, 
$\arg\min_{U}f(U,\Sigma,V,S)$ can be formulated as 
\[\min_{U} \quad \frac{1}{2}\|Y-S-U\Sigma V^{\top}\|_F^2\quad 
    \text{s.t.} \quad \|U\|_{2,\infty}\le \barmu,\ U^{\top} U=I\\\]
which is also equivalent to 
\[\max_{U} \quad \langle (Y-S)V\Sigma,U\rangle \quad 
    \text{s.t.} \quad \|U\|_{2,\infty}\le \barmu,\ U^{\top} U=I.\]

One way is to view it as a standard QCQP(Quadratic Constraints Quadratic Programming), and we can solve it using standard methods. Alternatively, we can discard the incoherent constraint, which therefore admits an explicit iteration scheme, viz., \[U'=\operatorname{sign}((Y-S)V\Sigma)\] where the matrix sign function is defined as $\operatorname{sign}(X)=U_X V_X^{\top}$ for $X=U_X \Sigma_X V_X^{\top}$ as the SVD.

The details of the algorithm are elaborated below in Algorithm 1.

\begin{algorithm}[t]
    \caption{Alternating Minimization in solving low-rank-plus-sparse matrix estimation}
    \begin{algorithmic}[1]
    \label{alg:alternating-minimization}
    \State Input: iteration number K, hyperparameters $\barmu,\barr$, termination criteria $\mathcal{C}$.
    \State Initialize $U_0=\Sigma_0=V_0=0$.
    \If{using \emph{Init1}}
    \State Initialize $S_0=0$.
    \ElsIf{using \emph{Init2}}
    \State Initialize $S_0=\mathcal{T}_{\bar s}(Y)$ where $\mathcal{T}_{\bars}$ is given by \eqref{def:operator-take-s-largest}.
    \ElsIf{using \emph{Random Init}}
    \State Initialize $S_0=M_{\bar s}\circ Y$, where $M_{\bar s}$ is the mask matrix with $\bar s$ randomly selected elements set to $1$ and other elements set to $0$.
    \EndIf
    \For{$k=1,2,\cdots,K$}

    \If{using \emph{Method1}}
    \State Update $U$ as $U_k:=\operatorname{sign}((Y-S_{k-1})V_{k-1}\Sigma_{k-1})$.
    \State Update $V$ as $V_k:=\operatorname{sign}((Y-S_{k-1})^{\top} U_{k}\Sigma_{k-1})$.
    \ElsIf{using \emph{Method2}}
    \State Update $U$ as
$U_k=\arg\max_{U} \langle (Y-S_k)V_{k-1}\Sigma_{k-1},U\rangle \quad 
    \text{s.t.} \quad \|U\|_{2,\infty}\le \barmu,\ U^{\top} U=I$.
    \State Update $V$ as 
$V_k=\arg\max_{V} \langle (Y-S_{k-1})^{\top} U_k\Sigma_{k-1},V\rangle \quad 
    \text{s.t.} \quad \|V\|_{2,\infty}\le \barmu,\ V^{\top} V=I$.
    \EndIf
    \State Update $\Sigma$ as $\Sigma_k:=\operatorname{Diag}(U_{k}^{\top}(Y-S_{k-1})V_{k})$.
    \State Update $S$ as $S_{k}:=\mathcal{T}_{\bars}(Y-U_{k}\Sigma_k V_{k}^\top)$.
    \If{$\mathcal{C}$ holds}
    \State Set $U_K=U_k,\Sigma_K=\Sigma_k,V_K=V_k,S_K=S_k$ and terminate.
    \EndIf
    \EndFor
    \State Output: $\hat{L}=U_K \Sigma_K V_K^T$ and $\hat{S}=S_K$.

    \end{algorithmic}
\end{algorithm}

We use $\emph{Method1}$ to denote the one without incoherence constraint in each iteration and $\emph{Method2}$ to denote the one with incoherence constraint. In later numerical results, we find that they are similar in numerical performance.

To initialize, we typically begin with one of two methods: $S=0$ (\emph{Init1}), which corresponds to starting with incomplete SVD on $Y$, or $U=\Sigma=V=0$ (\emph{Init2}), which is equivalent to setting $S=\mathcal{T}_{\bar s}(Y)$. Alternatively, initialization can be performed randomly. Specifically, we set $S=M_{\bar s}\circ Y$, where $M_{\bar s}$ is the mask matrix with $\bar s$ randomly selected elements set to $1$ and other elements set to $0$. These initialization schemes are included in Algorithm 1 and will be compared in Figure \ref{fig:convergence} in Section \ref{sec:numerical-results}.

We remark in passing that while the alternating minimization algorithm (specifically, \emph{Method1}) is not guaranteed to find a global maximum, it is always non-increasing, leading to the identification of a local minimum. Random initializations can be utilized, as well as warm start variations, to heuristically get rid of local minima. One can also utilize the Gumbel approximation trick applied in solving the combinatorial optimization~\citep{gu2024causality}. We will discuss this further in Section \ref{sec:numerical-results}.

Our theory relies on knowledge of $\bar{r}$, which serves as an upper bound for $r$. For specific problems, domain knowledge can guide an informed estimate of $\bar{r}$. Alternatively, one could construct a grid of possible values for $\bar{r}$ and solve the optimization problem for each. If $\bar{r}$ is set below the true $r$, this would result in a large bias and objective function value. A comprehensive solution for selecting $\bar r$ is deferred to future work.

\section{Markov Transition Matrix Estimation}
\label{sec:markov}
This section showcases the application of our main result to an important problem in the estimation of the Markov Transition Matrix.  We begin by introducing some necessary notations.

\subsection{Markov Chain Preliminary}
\label{sec:preliminary}
A sequence of random variables $\{X_0,X_1,\cdots\}$ on $\Omega$ is called a Markov chain if $\PP(X_{n+1}\big|X_0,\cdots,X_n)=\PP_n(X_{n+1}\big|X_n)$ for every $n\ge 1$. In this paper, we consider finite state space, i.e., $\Omega=[p]$ and $p$ can be large. The Markov chain is called homogeneous if  $\PP_n(X_{n+1}\big|X_n)$ does not depend on $n$ and will be denoted by $\PP(X_{n+1}\big|X_n)$.

Let $P\in \mathbb{R}^{p\times p}$ be the transition matrix with $P_{i,j}=\PP(X_{n+1}=j\big|X_n=i)$. For simplicity, we denote the simplex $\mathcal{P}_p$ as the set of all possible transition matrices
\[
    \mathcal{P}_p=\{P\in \mathbb{R}^{p\times p}\big|P\mathbf{1}_p=\mathbf{1}_p, P\ge 0\}.
\]

Ergodicity is a common assumption for Markov chains. A Markov chain is called irreducible if, for each state pair $(i,j)$, there exists a $n_{i,j}\ge 1$ such that $\PP(X_{n_{i,j}}=j\big|X_0=i)>0$. A Markov chain is called aperiodic if for each state $i$, the greatest common divisor of the set $\{n\ge 1\big|\PP(X_n=i\big|X_0=i)>0\}$ is 1.  A Markov chain is ergodic if it is both irreducible and aperiodic. Throughout the paper, we assume $\{X_0,X_1,\cdots\}$ is ergodic so there exists a unique stationary(or invariant) distribution $\pi$ which satisfies $\pi P=\pi$ and $\pi\mathbf{1}_p=1$. We denote $\pi_{\min}=\min_i \pi_i$ and $\pi_{\max}=\max_i \pi_i$. Ergodicity implies $\pi_{\min}>0$.

This paper also uses the frequency matrix $F\in \mathbb{R}^{p\times p}$ which is given by $F_{i,j}=\pi_i P_{i,j}$ and represents the probability that a pair of states $(i,j)$ appears in the Markov chain in the long run, that is, $F_{i,j}=\lim_{n\rightarrow \infty}(1/n)\sum_{k=1}^n 1_{X_k=i,X_{k+1}=j}$. To facilitate theoretical analysis, we define the enlarged frequency matrix class as
\begin{align}
\label{def:enlarged-freq}
    \mathcal{F}_p:=\{F\in \mathbb{R}^{p\times p}\big|\mathbf{1}_p^{\top} F\mathbf{1}_p=1,\  F\ge 0, \ e_i^{\top}F\mathbf{1}_p>0,\ \forall i\}.
\end{align} 
Note that a proper frequency matrix further needs to satisfy  $\sum_{i=1}^p F_{i,k}=\pi_{k}=\sum_{j=1}^p F_{k,j}$. By incorporating these constraints, the proper frequency matrix class $\tilde{\calF}_p$, as a nontrivial subset of $\calF_p$,  can be defined as
\begin{align}
\label{def:freq-class}
\tilde{\mathcal{F}}_p:=\{F\in \calF_p\big|e_i^{\top}F\mathbf{1}_p=\mathbf{1}_p^\top Fe_i,\ \forall i\}.
\end{align}
One can verify that every $F\in \tilde{\calF}_p$ is a proper frequency matrix. In fact, by defining $\pi=e_i^{\top}F\mathbf{1}_p$ and $P$ as $P_{i,j}=\frac{e_i^{\top}Fe_j}{e_i^{\top}F\mathbf{1}_p}$, we can argue that $\pi$ is the stationary distribution and $P$ is a valid transition matrix, leading to $F_{i,j}=P_{i,j}\pi_i$.

Similarly, we can argue that any $F\in \calF_p$ corresponds to an induced transition matrix $P$. On the other hand, any $P$ in $\calP_p$ representing ergodic chain corresponds to an induced frequency matrix $F\in\tilde{\calF}_p$.

Reversibility is another important concept. A Markov chain is called reversible if for each pair of states $(i,j)$, $\pi_i P_{i,j}=\pi_j P_{j,i}$, or equivalently $F_{i,j}=F_{j,i}$.  Another important concept in the Markov chain is the mixing time. The $\epsilon$-mixing time is defined as 
\[
    \tau(\epsilon)=\min \left\{k: \max _{1 \leq i \leq p} \frac{1}{2}\left\|e_i^{\top} P^k-\pi\right\|_1 \leq\epsilon\right\},
\]
namely the minimum time that the distribution of the Markov chain is nearly stationary initialized at any value $i$.  We also call $\tau_*=\tau(\frac{1}{4})$ the mixing time for short.
The mixing time characterizes the speed of the Markov chain converges to the stationary distribution in 1-norm. It has a strong relation with the spectral gap of the transition matrix. Specifically, if the chain is reversible, then we have the following inequality \citep{levin2017markov}:
\[
    \frac{\lambda_2}{1-\lambda_2}\log(\frac{1}{2\epsilon})\le\tau(\epsilon)\le \frac{1}{1-\lambda_2}\log(\frac{1}{\pi_{\min}\epsilon}),
\]
where $\lambda_2$ is the second largest absolute eigenvalue of $P$.
For more about mixing time, we refer the reader to \cite{berestycki2016mixing}.

\subsection{Structured Markov Transition Matrix Estimation}
\label{sec:results-estimating-markov-transitions}

In this section, we present our results of estimating the Markov transition matrix given a single trajectory $\{X_0, X_1,\cdots, X_n\}$.  We also assume the initial distribution is stationary. This assumption is not critical as one typically regards the mixing time as a constant. Before proceeding, we formalize structural assumptions on the ground-truth Markov transition matrix $P^{\star}$.

\begin{assumption}
\label{asp:structural-P*}
We impose structural assumptions on $P^{\star}$:
\begin{enumerate}
    \item The Markov chain corresponding to $P^{\star}$ is ergodic;
    \item $F^{\star}=\operatorname{diag}(\pi^{\star}) P^{\star}=L^{\star}+S^{\star}$, where $L^{\star}$ and $S^{\star}$ satisfy Assumption \ref{asp:L,S};

    \item $\pi_{\min}\le \pi_i\le \pi_{\max}$ for every $i$.
\end{enumerate}
\end{assumption}

Here, conditions 1 and 3 in Assumptions \ref{asp:structural-P*} are standard regularity assumptions in the literature of Markov chain~\citep{norris1998markov,chung1967markov}. Ergodicity guarantees that the Markov chain eventually visits each state with enough time which enables good estimation. Technically, it also assures a unique stationary distribution. The second condition is the main structural assumption on the frequency matrix.
The third condition ensures that the probability of each state in the stationary distribution is sufficiently large, preventing some states from being inadequately explored~\citep{zhang2019spectral}. 

We provide additional illustrations to clarify the second condition. While \cite{zhang2019spectral} assumes a low-rank structure for the transition matrix, such as in the case of state aggregation, this assumption is often overly restrictive. In practice, even if states can be grouped, the interactions among groups are not always homogeneous. Moreover, certain pairs of nodes may exhibit stronger or weaker connections, leading to deviations from exact low-rankness. In principle, the low-rank part may capture systematic trends, while the sparse part characterizes anomalies or unusual events. Below, we present several concrete examples to demonstrate the idea.

\begin{enumerate}
\item {\bf Population Flow}:
In urban planning, ride-sharing services, or epidemiology, one may care about how people move between different regions. The low-rank part can represent the general movement trends, such as people commuting daily between residential and work areas. On the other hand, the sparse part can capture rare events, such as
movement spikes due to festivals and emergencies or unusual travel patterns during holidays.
\item {\bf Financial Markets}:
One can model the transition of financial instruments between investors using Markov chains. In this example, the low-rank part may capture typical trading patterns, such as systematic trading strategies that the investor deploys. On the 
other hand, the sparse part may represent rare events like market crashes or unexpected news affecting a single stock or sector.
\item {\bf Social Network Interactions}:
Consider modeling interactions (e.g., messages, likes, or posts) between users in a social network.
In this example, the low-rank component reflects regular interactions, such as users engaging primarily with close friends or family. On the other hand,
sparse component
captures irregular interactions, such as viral content, causing engagement spikes with unfamiliar users.
\end{enumerate}

\begin{example}
As an illustrative example, consider a Markov chain modeling population flow across four locations, $\calS=\{A,B,C,D\}$. Under normal conditions, one may impose the state aggregation structure (Definition 3 in \cite{zhang2019spectral}) with $\Omega_1=\{A,B\}$ and $\Omega_2=\{C,D\}$, resulting in the low-rank structure. Suppose the transition matrix under normal conditions is given by
 \[ 
 P^\star_N=\left[\begin{array}{c@{\hspace{10pt}}c@{\hspace{10pt}}c@{\hspace{10pt}}c}
1/6 & 1/6 & 1/3 & 1/3 \\
1/6 & 1/6 & 1/3 & 1/3 \\
1/3 & 1/3 & 1/6 & 1/6 \\
1/3 & 1/3 & 1/6 & 1/6
\end{array}\right].
\]
However, there may be aspects not fully captured by the state aggregation model. For example, suppose during some time period, the transportation from $D$ to other three places is under construction, while the transportation from other places to $D$ remains unchanged. In this case, the true transition matrix in this period is given by \begin{align*}
P^\star=\left[\begin{array}{c@{\hspace{10pt}}c@{\hspace{10pt}}c@{\hspace{10pt}}c}
1/6 & 1/6 & 1/3 & 1/3 \\
1/6 & 1/6 & 1/3 & 1/3 \\
1/3 & 1/3 & 1/6 & 1/6 \\
1/6 & 1/6 & 1/12 & 7/12
\end{array}\right]=\left[
\begin{array}{c@{\hspace{10pt}}c}
1 & 0 \\
1 & 0 \\
0 & 1 \\
0 & 1/2
\end{array}\right]\left[\begin{array}{c@{\hspace{10pt}}c@{\hspace{10pt}}c@{\hspace{10pt}}c}
1/6 & 1/6 & 1/3 & 1/3 \\
1/3 & 1/3 & 1/6 & 1/6\end{array}\right]+\left[\begin{array}{c@{\hspace{10pt}}c@{\hspace{10pt}}c@{\hspace{10pt}}c}
0 & 0 & 0& 0 \\
0&0& 0 & 0 \\
0&0& 0 & 0 \\
0&0& 0 & 1/2 \\
\end{array}\right] .
\end{align*}
Therefore, $P^\star$ has the low-rank-plus-sparse structure, and the same holds for $F^\star=\operatorname{diag}(\pi^\star)P^\star$.\qed
\end{example}

Like estimating a discrete distribution, most estimators of transition matrix start from the empirical frequency matrix $\tF$, whose $(i, j)$ element is defined by
\begin{equation}
\label{equ:empirical-frequency-matrix}
\tF_{i,j}=(1/n)\sum_{k=1}^n 1_{X_k=i,X_{k+1}=j}.
\end{equation}

In order to apply our deterministic result, we let the observation be $Y=\tF$, effective noise be $\tF-F$, and apply the incoherent-constrained least-squares \eqref{alg:main}. To establish its property, we need to bound $\|\tF-F^{\star}\|$ as well as  $\|\tF-F^{\star}\|_{\max}$ with high probability, which is stated in the following proposition. In particular, the first half of the proposition is the same as in Lemma 7 in \cite{zhang2019spectral}.
Denote $p_{\max}:=\max_{i,j} p_{i,j}$.

\begin{proposition}
\label{prop:error-bound}
    Under Assumption \ref{asp:structural-P*}, for $\forall c_0$, there exists a constant $C$ such that for all $n\ge C\tau_* p \log^2 n$, we have [Lemma 7 of  \cite{zhang2019spectral}]
    \begin{equation}
    \label{equ:error-bound-spectral-norm}
    \PP\left(\|\tF-F^{\star}\|\ge C\sqrt{\frac{\pi_{\max} \tau_* \log^2 n}{n}}\right)\le n^{-c_0}
    \end{equation}
    and 
    \begin{equation}
    \label{equ:error-bound-max-norm}
    \PP\left(\|\tF-F^{\star}\|_{\max}\ge C\sqrt{\frac{p_{\max}\pi_{\max} \tau_* \log^2 n}{n}}\right)\le n^{-c_0}.
    \end{equation}
\end{proposition}

\begin{remark}
    For the max norm error bound, one may attempt to apply Bernstein's inequality for Markov chain in \cite{jiang2018bernstein} on $Y$. However, it turns out that the absolute spectral gap is zero for this specific type of Markov chain (as shown in Lemma B.1 in the supplementary material), making this approach infeasible. 
\end{remark}

\begin{remark}
    The mixing time $\tau_*$ in the bound stems from the Markov chain concentration technique \citep{zhang2019spectral}. Even if the data starts from stationary distribution, in order to decorrelate, traditional concentration inequality is applied individually to groups of samples indexed by arithmetic sequences. As there are $\tau_*$ groups of them, there is a loss of factor $\tau_*$. The dependence on $\tau_*$ makes sense here because even though the chain starts from stationary distribution, there is only a single realized trajectory, and the chain has to take $\tau_*$ steps to approach stationarity again. 
\end{remark}

We present our algorithm for Markov transition estimation in detail in Algorithm 2, which originates from our main method \eqref{alg:main}. After solving \eqref{alg:main},
we project the solution to $\calF_p$ and then normalize each row to obtain the estimated transition matrix $\hat P$. It is worth noting that while $\hat P$ is a valid transition matrix, $\hat F$ may not be a valid frequency matrix; see \eqref{def:enlarged-freq} and discussions therein. One could either project the solution to the proper frequency class $\tilde{\calF}_p$ defined in \eqref{def:freq-class}, which forms a polytope, or transform $\hat P$ back. The technical details are omitted for brevity.

\begin{algorithm}[t]
    \caption{Estimating Markov transition through low-rank-plus-sparse matrix estimation}
    \begin{algorithmic}[1]
    \label{alg:transition-kernel}
        \State Input: A single trajectory of Markov chain $\{X_0,X_1,\cdots,X_n\}$.
        \State Calculate the empirical frequency matrix $\tF$ defined in \eqref{equ:empirical-frequency-matrix}.
        \State Solve optimization problem \eqref{alg:main} with $Y=\tF$ to get $\widehat{F}^0$. A practical algorithm is presented in Section \ref{sec:practical-algorithm}.
        \State  Truncate the elements of $\widehat{F}^0$ to [0,1], i.e., $\widehat{F}^1=\min(\max(\widehat{F}^0,0),1)$.
        \State Project $\widehat{F}^1$ onto the enlarged space of frequency matrices $\mathcal{F}_p$ in Euclidean distance(which is equivalent to projection in $L_1$ distance) to get $\hF$.  
        This can be done by the following.  
        Find a $t$ such that $\sum_{i=1}^p\sum_{j=1}^p \max\left(\widehat{F}^1_{i,j}-t,0\right)=1$, which could be implemented, e.g., by sorting and checking $p$ breakpoints or use bisection method. The projection is then given by $\hat F_{i,j}:=\max\left(\widehat{F}^1_{i,j}-t,0\right)$.
        
        \State Transform $\widehat{F}$ to $\widehat{P}$ by $\widehat{P}_{i,j}=\frac{\widehat{F}_{i,j}}{\sum_{j=1}^p \widehat{F}_{i,j}}$.
        \State Output $\widehat{F},\widehat{P}$.
    \end{algorithmic}
\end{algorithm}

Now, we are ready to present our main theorem in this section. Recall that $p_{\max}:=\max_{i,j} p_{i,j}$ and $\|\cdot\|_1$ stands for elementwise $L_1$ norm.

\begin{theorem}
\label{thm:markov-main}
    Under Assumption \ref{asp:structural-P*}, when $\barr=\operatorname{O}(r)$, $\bars=\operatorname{O}(s)$, and $n\ge C\tau_* p \log^2 n$, for every constant $c_0>0$, there exists some constant $c>0$ such that with probability at least $n^{-c_0}$, the estimator $\widehat{F},\widehat{P}$ given by Algorithm 2 satisfies the following:

    \begin{equation}
    \label{equ:main-1}
        \|\widehat{F}-F^{\star}\|_F\le \sqrt{c 
        \frac{\pi_{\max}\tau_*\log^2 n}{n}(r+p_{\max}s)}
    \end{equation}

    \begin{equation}
    \label{equ:main-2}
        \|\widehat{F}-F^{\star}\|_1\le p\sqrt{c
        \frac{\pi_{\max}\tau_*\log^2 n}{n}(r+p_{\max}s)}
    \end{equation}

    \begin{equation}
    \label{equ:main-3}
        \|\widehat{P}-P^{\star}\|_1\le \frac{p}{\pi_{\min}}\sqrt{c 
        \frac{\pi_{\max}\tau_*\log^2 n}{n}(r+p_{\max}s)}
    \end{equation}
\end{theorem}
\begin{remark}
The rate with respect to sample size $n$ scales as $\sqrt{\log^2 n/n}$. The logarithmic terms arise from applying the union bound. While we generally disregard the logarithmic terms in this paper, it is an interesting question whether they can be further eliminated. Another interesting question is whether multi-cover time instead of mixing time can be used to achieve tighter bounds~\citep{chan2021learning}.
\end{remark}

Further, under some mild assumptions, we have the following corollary.
\begin{corollary}
    Under the same assumptions of Theorem \ref{thm:markov-main},
    and assume further that
    $p_{\max}=\operatorname{O}(\frac{1}{p})$, $\pi_{\max}=\operatorname{O}(\frac{1}{p})$, $\tau_*=\operatorname{O}(1)$, we have
        \[\|\widehat{F}-F^{\star}\|_F=\operatorname{\Tilde{O}_p}\left(\sqrt{
        \frac{r}{np}}\right).\]
If in addition, $\pi_{\min}=\operatorname{\Omega}(\frac{1}{p})$, we have 
    \begin{equation}
        \|\widehat{P}-P^{\star}\|_1=\operatorname{\Tilde{O}_p}\left(\sqrt{
        \frac{rp^3}{n}}\right)
    \end{equation}
    which achieves the minimax lower bound as proved in \cite{zhang2019spectral}.
\end{corollary}

\begin{remark}
    Our results show that we can achieve the minimax lower bound even if some elements deviate from exact low-rankness arbitrarily. The maximum allowable number of such elements is proportional to dimension when $r$ and $\mu$ are constants. The lower bound in \cite{zhang2019spectral} is applicable to our setting because the low-rank case is subsumed by a broader low-rank-plus-sparse case. 
\end{remark}

\subsection{Implication on structured RL}
\label{sec:implication-rl}
In reinforcement learning, one is usually interested in estimating the conditional mean, which is an essential component of the Bellman operator. Linear Markov Decision Processes (MDPs) correspond to the case where the transition matrix is a low-rank matrix and the feature matrix (on one side of the decomposition) is known.  The assumption on low rankness is too restrictive and a natural question is whether the low-rank-plus-sparse structure also helps improve the efficiency in estimating conditional mean.

Specifically, given a (value) vector $v$, we hope to estimate $\Tau(v):= Pv$, the conditional mean function. The following simple corollary indicates that the low-rank-plus-sparse structure can be utilized to achieve better estimation in comparison with fully nonparametric estimation. 

\begin{corollary}
\label{cor:conditional-mean-bound}
    Let $v$ be a random vector in $\RR^p$ satisfying $\EE(vv^{\top})\preceq c_v I$. The observed data is a single trajectory of $X_1,X_2,\cdots,X_n$, where $n\ge C\tau_* p \log^2 n$. Suppose that Assumption \ref{asp:structural-P*} holds, and $v$ is independent with $X_1,X_2,\cdots,X_n$, then the estimator $\hat{\Tau}(v):=\hat{P}v$ where $\hat{P}$ is obtained by Algorithm 2 satisfies
    \[\EE\|\hat{\Tau}(v)-\Tau(v)\|_2^2\le c_v 
        \frac{c\pi_{\max}\tau_*\log^2 n}{n\pi_{\min}^2}(r+p_{\max} s).\]
        Further, under the  conditions that $p_{\max}=\operatorname{O}(\frac{1}{p})$, $\pi_{\max}\asymp \pi_{\min}$ and $\tau_*=\operatorname{O}(1)$, we have \[\EE\|\hat{\Tau}(v)-\Tau(v)\|_2^2=\operatorname{O} \left(c_v \frac{rp\log^2 n}{n}
        \right).\]
\end{corollary}

\begin{remark}
     From the above corollary, we observe that the structural assumption (low-rank-plus-sparse) reduces the estimation error of the expected conditional mean by a factor of ${p}/{r}$compared to the scenario without such an assumption. Notably, the proof utilizes the assumption that $v$ is a random variable. The next theorem is devoted to illustrating that even in the exact low-rank case, there is no benefit in the worst case for the deterministic $v$.
\end{remark}

In the above corollary, the randomness of $v$ is crucial as the improvement is in the average-case sense. We prove in the following theorem that even in the exact low-rank case where $S^\star=0$, there is no gain in estimation efficiency in the worst-case(or minimax) sense.

\begin{theorem}
\label{thm:rl-fixed-v}
    Let $v$ be an arbitrary vector in $\RR^p$. Let $\Bar{v}:=p^{-1}\sum_{i=1}^p v_i$ and $u:=v-\Bar{v}$. Define the minimax risk as $\mathfrak{R}_r(v)=\inf\limits_{\hat{\mathcal{T}}}\sup\limits_{P:\operatorname{rank}(P)\le r}\|\hat{\mathcal{T}}(v)-\mathcal{T}(v)\|_2^2$. We have 
    \[\mathfrak{R}_r(v)\ge \mathfrak{R}_1(v)\ge \frac{\|u\|_2^2}{36n}=\frac{\|v-\Bar{v}\|_2^2}{36n}\] as long as $n\ge \frac{4p\|u\|_{\max}^2}{9\|u\|_2^2}$.
\end{theorem}

\begin{remark}
    When $v\sim \mathcal{N}(0,I)$, we have $\|u\|_2\asymp \|v\|_2\asymp p$.  Therefore, the minimax risk is of order $\operatorname{O}({p^2}/{n})$. Compared with Corollary \ref{cor:conditional-mean-bound}, it's larger with a factor of ${p}/{r}$, and there is no gain even assuming the transition matrix is exactly low-rank.
\end{remark}

\begin{remark}
    We here provide intuition for the discrepancy between Corollary \ref{cor:conditional-mean-bound} and Theorem \ref{thm:rl-fixed-v}. The estimation error of the conditional mean is a one-dimensional projection of that of the transition matrix estimation. When one of those top singular vectors of $P$ aligns with $v$, the one-dimensional projection estimation error can be large, as shown in Theorem $\ref{thm:rl-fixed-v}$. By randomizing the one-dimensional directions, this unfavorable case of alignment is evened out. This is in analogy to denoising a vector under sparsity condition, where the estimating error in the Euclidean norm is improved by thresholding while the $L_{\infty}$ error typically can't benefit from the sparsity assumption.
\end{remark}

\begin{remark}
    In the linear MDP literature, not only the low-rankness assumption is imposed, but the feature is known \citep{jin2020provably}. Low-rank MDP corresponds to the unknown feature case, where feature learning is involved and assumptions on the function class of features are imposed \citep{agarwal2020flambe,modi2021model}.
\end{remark}

\section{Additional Applications}
\label{sec:further-example}

To demonstrate the wide applicability of our incoherent constrained least-squares method and its associated results, this section illustrates further our results in two important problems:  Multitask regression and robust covariance regularization. 

\subsection{Multitask Regression}
\label{sec:multi-task-regression}
Apart from the identity observation model, another simple observation model is the structured multitask regression \citep{ando2005framework,argyriou2008convex,kumar2012learning}. In particular, \cite{agarwal2012noisy} proposed the following low-rank-plus-sparse form:

\[Y=XP^{\star}+W\]
where $X,Y,W\in \RR^{n,p}$ and $P^{\star}=L^{\star}+S^{\star}$. A direct application of  Theorem \ref{thm:main} leads to the following corollary:

\begin{corollary}
\label{cor:multitask-regression}
    Let $\Delta_L=\hL-L^{\star}, \Delta_S=\hS-S^{\star}$. Assume the design matrix $X$ satisfies that $\sigma_{\max}(X)\le \sigma_{\max}$, $\sigma_{\min}(X)\ge \sigma_{\min}$, $\|X^{\top}\|_{2,\infty}\le B$.
    Under Assumptions \ref{asp:RSC},\ref{asp:RSC-parameters},\ref{asp:L,S}, we have 
    \begin{align}
    \label{equ:multitask-regression-error-bound}
    \|\DL\|_F^2+\|\DS\|_F^2\lesssim \frac{1}{\sigma_{\min}^4}\left(\barr\sigma_{\max}^2\|W\|^2+\bars B^2\|W\|_{2,\infty}^2\right).
    \end{align}
\end{corollary}

\begin{remark}
The Corollary 4 in \cite{agarwal2012noisy} gives a similar deterministic result, albeit with a trailing term $\calO(s/d^2)$ that doesn't disappear even $n\rightarrow\infty$. When $s=0$, the problem simplifies to low-rank multi-task learning or reduced-rank regression, and similar probabilistic results can be found in \cite{duan2023adaptive}. Multi-task learning is also closely related to trace regression \citep{anderson1951estimating,izenman1975reduced}. 
\end{remark}

\subsection{Robust Covariance Estimation}
\label{sec:robust-covariance-estimation}
In this section, we apply our framework to the problem of covariance estimation~\citep{fan2019robust}, i.e., given data $X_1,X_2,\cdots,X_T \in \RR^p$ which follows the same distribution, we aim to estimate the covariance matrix $\Sigma=\EE X_1 X_1^{\top}$. A structural assumption on the covariance estimation is that the covariance matrix can be decomposed as the sum of low-rank part and sparse part, i.e., 
\begin{equation}
\label{equ:structured-covariance-estimation}
\Sigma=\Sigma_L+\Sigma_S
\end{equation}
where $\operatorname{rank}(\Sigma_L)\le r $ and $\|\Sigma_S\|_0\le s$.   This model entails the classical factor model, which assumes
\begin{equation}
\label{equ:factor-model}
X_t=Bf_t+\epsilon_t,\quad f_t\in \RR^r
\end{equation}
where $f_t$ and $\epsilon_t$ are uncorrelated and $\var(\epsilon_t) = \Sigma_S$.

We consider a general case where we regard $X_1,X_2,\cdots,X_T$ as an i.i.d. sequence, where $X_t$ could potentially have heavy tails. In that case, a natural estimator for general covariance estimation is the Winsorization (truncation) estimator $\widehat{\Sigma}=\{\widehat{\sigma}_{ij}\}_{i,j=1}^p$, whose entries are given by\footnote{For notational simplicity, let $X_{ki}$ denote the $i$-th entry of $X_k$.} 
\begin{equation}
    \widehat{\sigma}_{ij}=\widehat{\EE X_{\cdot i} X_{\cdot j}}-\widehat{\EE X_{\cdot i}}\widehat{\EE X_{\cdot j}}, \quad i,j=1,2,\cdots,p
\end{equation}
where 
\begin{eqnarray*}
    \widehat{\EE X_{\cdot i} X_{\cdot j}} & = & \frac{1}{n}\sum_{k=1}^{T} \operatorname{sign}(X_{ki}X_{kj})\min(|X_{ki}X_{kj}|,\tau_1).\\
    \widehat{\EE X_{\cdot i}} & = & \frac{1}{n}\sum_{k=1}^{T}\operatorname{sign}(X_{ki})\min(|X_{ki}|,\tau_2). 
\end{eqnarray*}
After the pilot estimator $\widehat{\Sigma}$ is constructed, we apply the optimization program \eqref{alg:main} to $\widehat{\Sigma}$ as the observation to get estimators $\widehat{\Sigma}_L$ and $\widehat{\Sigma}_S$ for low-rank part and sparse part, respectively. The main benefit of using this estimator is that, as shown in the following, the pervasiveness assumption is no longer needed.  That is an advantage of the optimization-based method.

We impose the bounded fourth-moment assumption common in the literature~\citep{fan2017estimation,fan2021robust,fan2021shrinkage}.
\begin{assumption}
\label{asp:fourth-moment-bounded}
    Suppose the fourth moment of $X_{1i}$ is bounded for every $i$, i.e., $\EE(X_{1i}^4)\le C$.  
\end{assumption}

Taking the truncation level parameters $\tau_1,\tau_2\asymp \sqrt{n}$ in the pilot estimator, which strikes a balance between bias and variance, we have the following theorem.
\begin{theorem}
\label{thm:robust-covariance-estimation}
    Denote $\widehat{\Sigma}_L$ and $\widehat{\Sigma}_S$ as the output of \eqref{alg:main}. Under Assumptions \ref{asp:RSC},\ref{asp:RSC-parameters},\ref{asp:L,S} on $\widehat{\Sigma}_L$ and $\widehat{\Sigma}_S$, and Assumption \ref{asp:fourth-moment-bounded} on the distribution of $X_i$, it holds that
    \[\|\widehat{\Sigma}_L-\Sigma_L\|_F^2+\|\widehat{\Sigma}_S-\Sigma_S\|_F^2\lesssim r\|\widehat{\Sigma}-\Sigma\|^2+s \|\widehat{\Sigma}-\Sigma\|_{\max}^2=\mathcal{O}_p(\frac{rp^2\log^2 p}{n}).\]
\end{theorem}

\begin{remark}
   The rate is nearly minimax optimal when $r$ is a constant. In fact, we have $\inf\limits_{\widehat{\Sigma}}\sup\limits_{\Sigma}\|\widehat{\Sigma}-\Sigma\|_F^2\asymp \frac{p^2}{n}$ and $\|\widehat{\Sigma}_L-\Sigma_L\|^2_F+\|\widehat{\Sigma}_S-\Sigma_S\|^2_F\gtrsim \|\widehat{\Sigma}-\Sigma\|_F^2$. 
   \end{remark}


When common assumptions on the factor model \eqref{equ:factor-model} hold, we can derive the error of estimating loading factors $\{b_j\}_{j=1}^p$, and the proof basically follows Theorem 3 in \cite{fan2019farmtest}.
Specifically, due to the identifiability issue, we can assume $cov(f_i)=I$ and columns of $B$ are pairwise orthogonal \citep{fan2013large}. Denote $l^{th}$ column of $B$ as $\tb_l$. We know that $v_l=\frac{\tb_l}{\|\tb_l\|}$ and $\lambda_l=\|\tb_l\|^2>0$ are singular vectors and singular values of $B$, respectively. Further, $\lambda_1\ge \lambda_2\ge \cdots \lambda_r$ are in descending order. Now let $\widehat{v}_l$ and $\widehat{\lambda}_l$ be the singular vectors and singular values of $\widehat{\Sigma}_L$, respectively. We get an estimator of $B$ as $\widehat{B}=[\widehat{\lambda}_1^{\frac{1}{2}} \widehat{v}_1,\cdots,\widehat{\lambda}_r^{\frac{1}{2}} \widehat{v}_r]$. Denote the $j^{th}$ row of $\widehat{B}$ as $\widehat{b}_j$. The following theorem gives the estimation error of $\widehat{b}_j$.

\begin{theorem}
\label{thm:factor-loading-estimation}
    Under the same assumptions of Theorem \ref{thm:robust-covariance-estimation}, we have with probability at least $1-n^{-c}$ for some constant $c>0$,
    \begin{align}
    \label{equ:factor-loading-estimation-1}
    \min_{H\in \RR^{r\times r}}\|BH-\hat B\|^2\lesssim \frac{rp^2\log^2 (p)\log n}{n\lambda_r}.
    \end{align}
    Furthermore, if $\lambda_l-\lambda_{l+1}=\Omega(p)$ for all $1\le l\le r$ and $n\gtrsim \frac{p\log^2 p}{\mu}$, we have with probability at least $1-n^{-c}$,
    \begin{align}
    \label{equ:factor-loading-estimation-2}\min_{H\in \calO^{r\times r}:H^2=\bI}\sum_{j=1}^p\|b_j-H\widehat{b}_j\|^2\lesssim \frac{r^3 p\log^2 (p)\log n}{n}. \end{align}
\end{theorem}

\begin{remark}
    Note that there is rotational ambiguity regarding the direction factors. Hence, an orthogonal (or nearly orthogonal) transformation is used to eliminate this ambiguity, which in \eqref{equ:factor-loading-estimation-2} is taken to be a diagonal matrix with $\pm 1$ on the diagonal.
    \end{remark}
    
    \begin{remark}
    The inequality~\eqref{equ:factor-loading-estimation-1} is a subspace perturbation bound. For \eqref{equ:factor-loading-estimation-2}, $\lambda_j-\lambda_{j+1}=\Omega(p)$  is a common assumption on eigen gap \citep{fan2013large,fan2019farmtest}. Apart from this, we don't need the pervasiveness assumption as the magnitude of $\Sigma_L$ and $\Sigma_S$ can be arbitrary. 
    Whether it is possible to shave a factor of $\sqrt{p}$ for the bounding error of individual factor is left to future work.
\end{remark}

\section{Numerical Results}
\label{sec:numerical-results}
In this section, we conduct extensive numerical simulations and real data experiments to corroborate our theoretical results.

\subsection{Simulations}
We first conduct four simulations as follows.\\

\emph{Experiment 1}:
In the first experiment, we test the convergence of our alternating minimization algorithm in the case of weakly correlated noise. 

First of all, we generate the ground-truth matrix $M$ as follows.\footnote{To distinguish from the transition matrix $P$ which is considered in later experiments, we use notations $M$ to stand for general 'Matrix'.}
Set $p=1000$, $r=3$, $s=1000$, and let
\[M=ADB+S,\]
where $A\in \mathbb{R}^{p,r}$ and $B\in \mathbb{R}^{r,p}$ are sampled independently from $\mathcal{U}[0,1]$ in each entry, $D$ is a diagonal matrix whose diagonal entries are sampled independently from $\mathcal{U}[0,1]$, $S$ is supported on exactly $s$ entries, where the support is sampled uniformly in all $p^2$ entries. Conditional on the support location, entries are also drawn independently from $\mathcal{U}[0,1]$.
Numerically, we observe that this type of parameter generation process almost always results in the linear part with of low incoherence. 

The observation $Y$ is given by $Y=M+W$.
Apart from no-noise case ($W=0$), we consider two types of noise options as follows.
\begin{itemize}
    \item I.i.d. Gaussian: Each $W_{ij}$ is sampled independently from $\mathcal{N}(0,\sigma^2)$ where $\sigma=10^{-3}$.
    \item Empirical distribution errors: 
    Each row of $W$ is an independent realization of  $\left(n^{-1}\sum_{i=1}^n \be_{X_i} - {\mathbf{1}_p}/{p}\right)^{\top}$, where $X_i$ are i.i.d. discrete random variables uniform in $[p]$. \footnote{As can be seen from the definition, it resembles the noise arising from the Markov transition matrix estimation problem for a special Markov transition kernel.}
    Note that the entries in a row are weakly correlated, with variance in each entry given by $\frac{1}{np}(1-\frac{1}{p})$. 
\end{itemize}

We apply our alternating minimization algorithm on the generated data, with different initialization methods\footnote{We run the version without incoherence constraint as we later demonstrate that this constraint has minimal impact in practice.}, with results shown in Figure \ref{fig:convergence}. The first panel shows the case when there is no noise. The Frobenius error to the ground truth converges quickly to 0 in this case. The second panel shows the case when the noise is i.i.d. Gaussian with standard deviation $10^{-3}$, while the third case shows the case when each row of the noise matrix is sampled independently from empirical probability error distribution. To facilitate fair comparison, we set $n = 999$ such that the standard deviation of each entry is equal to $10^{-3}$. The latter two panels are based on 5 trials each. 

We make the following observations:\vspace*{-0.1 in}
\begin{enumerate}
\item All three initialization methods converge to the same optimal value. $\emph{Init1}$ and $\emph{Random Init}$ performs similarly and are better than $\emph{Init2}$.
\item In the middle panel of Figure \ref{fig:convergence}, the trajectories within each initialization method appear to collapse and overlap, which aligns with the concentration property of the Gaussian ensemble.
\item The middle panel and the right panel are only slightly different. This is expected, as the entrywise variance is controlled to be equal for two noise settings, and the empirical probability error exhibits weak correlations across entries.
\end{enumerate}

\begin{figure}[t]
\begin{center}
    \begin{tabular}{ccc}   \includegraphics[width=0.3\linewidth]{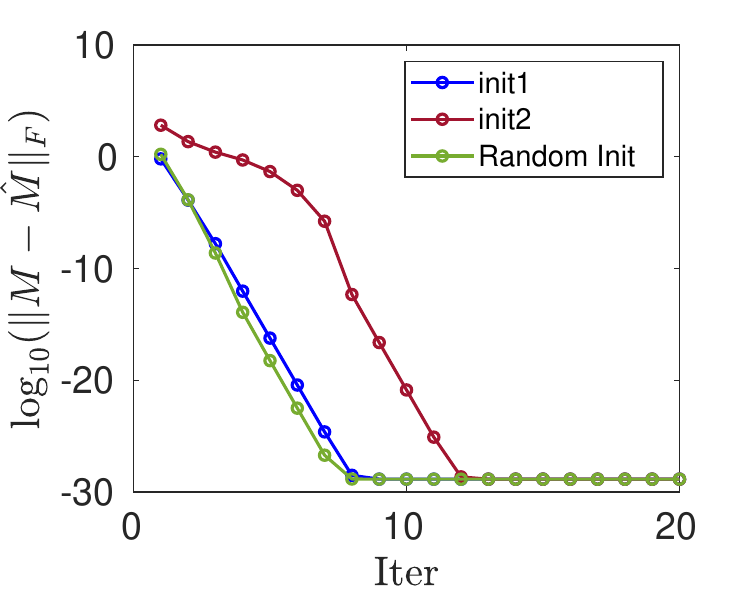}
    &\includegraphics[width=0.3\linewidth]{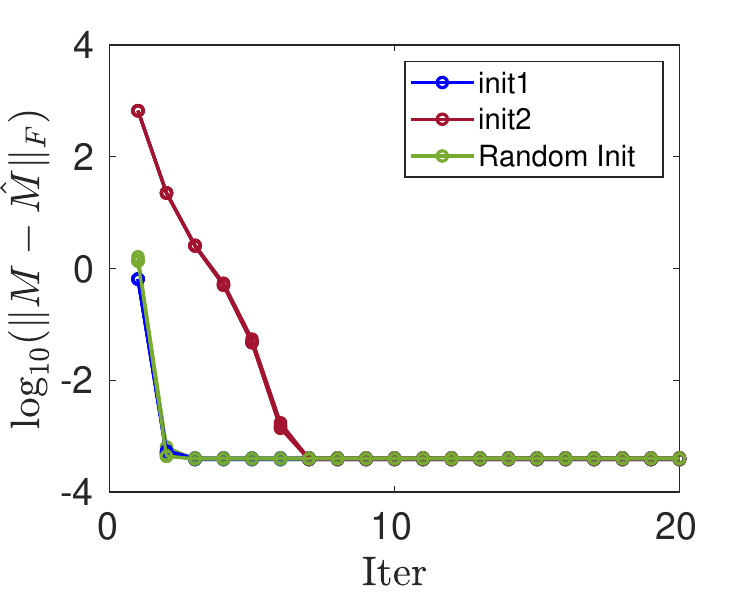}&\includegraphics[width=0.3\linewidth]{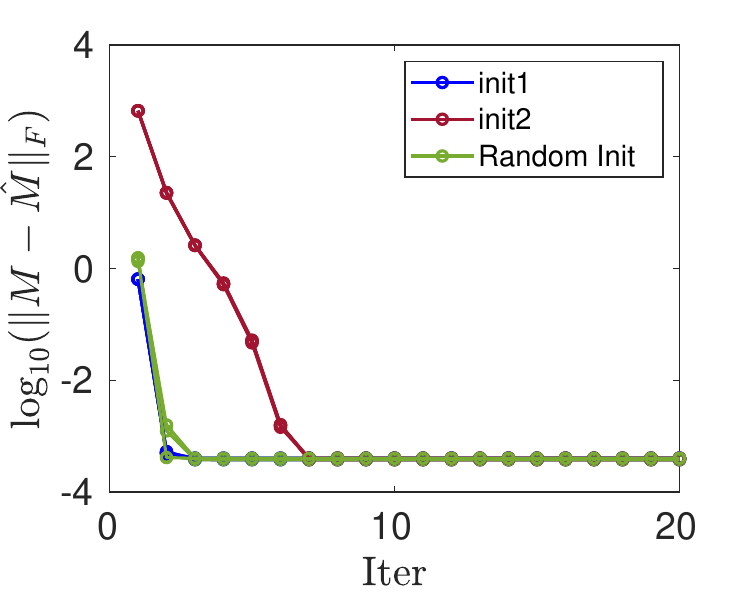}
    \tabularnewline
    
    no noise&i.i.d. Gaussian&empirical distribution error\tabularnewline
    
    \end{tabular}
    \end{center}
    \caption{\emph{Trajectories of optimization error for Algorithm 1 based on synthetic data with 3 noise scenarios}. The left panel corresponds to no-noise case ($W=0$). The middle panel corresponds to the i.i.d. Gaussian noise with the standard deviation of each entry being $10^{-3}$. The right panel corresponds to the noise sampled from empirical probability error distribution with the standard error of each entry being approximately $10^{-3}$. In each panel, blue curve is the result of $\emph{Init1}$, the red curve is the result of $\emph{Init2}$, and the green curve is the result of \emph{Random Init}. As can be seen, $\emph{Init1}$ and $\emph{Random Init}$ performs similarly and are better than $\emph{Init2}$. All three initialization methods converge the same optimal value.}
    \label{fig:convergence}
    \end{figure}

The next three experiments are conducted based on the Markov chain transition estimation problem that we consider in this paper. For better focus, we only use the $\emph{Init1}$ as the initialization method.
   \label{tab:Parameters-of-interest}
We generate the ground-truth transition matrix $P$ as follows\footnote{The ground-truth frequency matrix $F$ can be generated accordingly based on the transition matrix $P$.},
\[P_0=|tADB+S|, \qquad P=\operatorname{Diag}(P_0 \mathbf{1}_p)^{-1} P_0, \]
where $t$ characterizes the signal strength ratio of the linear part and the sparse part and $|\cdot|$ means taking absolute value entrywise.   $A,D,B,S$ are defined in the same way as in the first experiment. 

As $r$ is of minor importance, throughout the experiments, we set $r=3$ as in \cite{zhang2019spectral}. Also, we take the sparse level $s=p$. In the sequel, we will see that our results are nearly independent of $t$, therefore showing our method and results accommodate a wide range of scenarios with varying signal strengths of linear part and sparse part.

In implementing our proposed algorithms, we need a stopping rule. We set the maximum iteration number to be 3000, and when the Frobenius distance of two sparse part matrices in two consecutive iterations differs no more than $10^{-10}$ in the Frobenius norm of the sparse matrix in the iteration, we end the loop.
\\

\emph{Experiment 2}:
In the second experiment, we validate the dependence of Frobenius norm error on the sample size $n$. We fix $p=200$ and let $n$ take value from  $\{10000+5000k:0\le k\le 20\}$.

\begin{figure}[t]
\begin{center}
    \begin{tabular}{cc}   \includegraphics[width=0.4\linewidth]{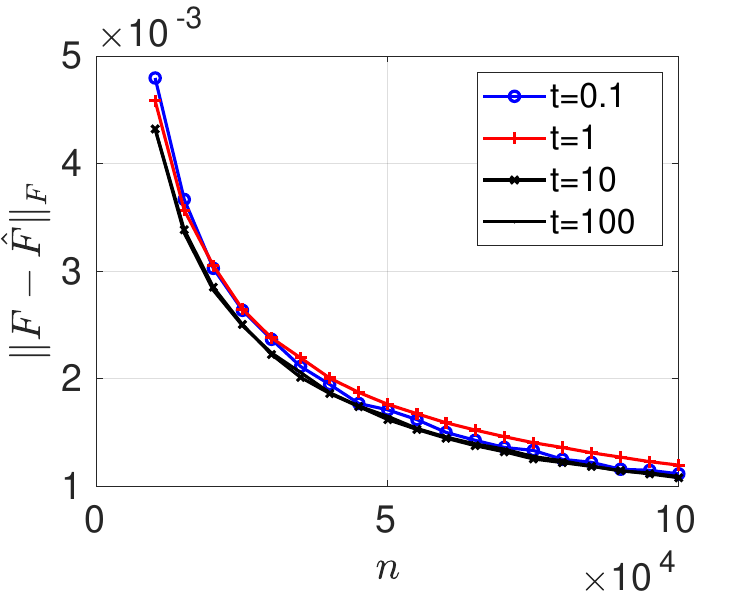}
    &\includegraphics[width=0.4\linewidth]{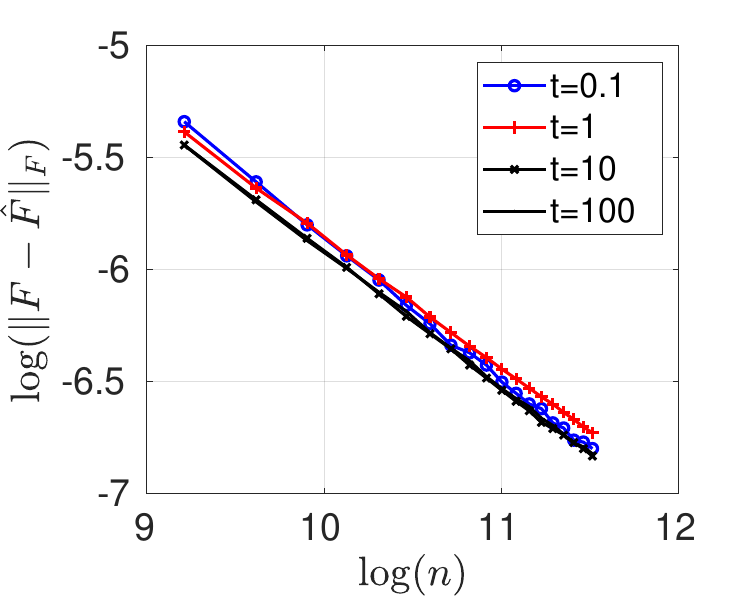}
    \tabularnewline
    
    $\Vert F-\widehat{F}\Vert_{F}\quad \text{v.s.}\quad n$&$\log\left(\Vert F-\widehat{F}\Vert_{F}\right)\quad \text{v.s.}\quad  \log(n)$\tabularnewline

    \includegraphics[width=0.4\linewidth]{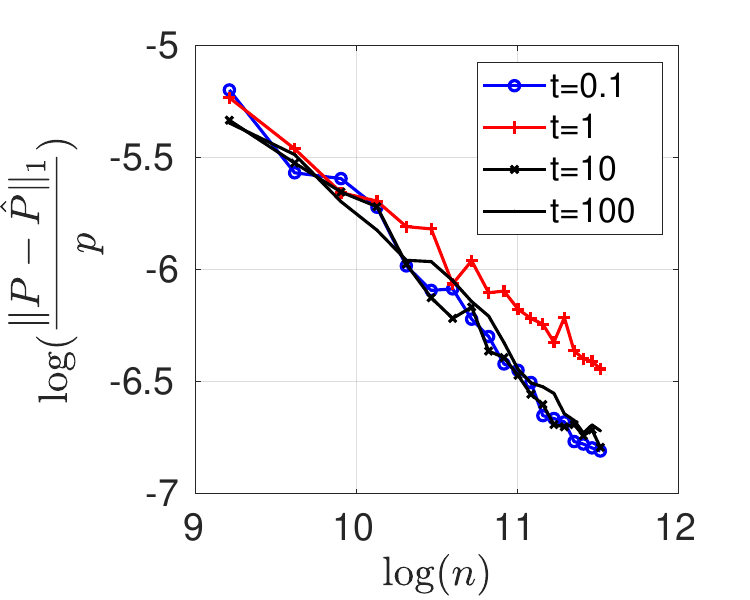}&\includegraphics[width=0.4\linewidth]{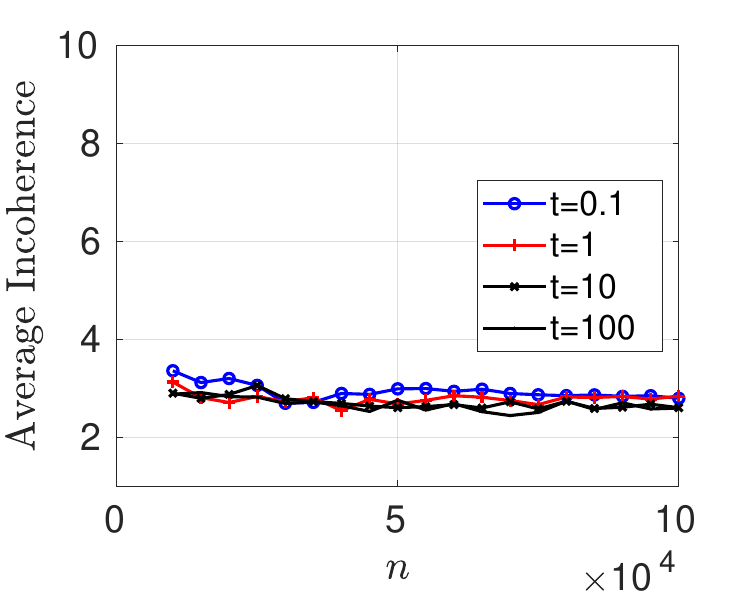}
    \tabularnewline
    $\log\left(\frac{\Vert P-\widehat{P}\Vert_{1}}{p}\right)\quad \text{v.s.}\quad \log(n)$&$\mu(\widehat{F})\quad \text{v.s.}\quad  n$\tabularnewline
    
    \end{tabular}

    \end{center}
    \caption{\emph{Alternating minimization method for estimating Markov transition matrix}. We show in the upper-left panel how error rates of estimating the frequency matrix in Frobenius norm decay with sample size. Not surprisingly, the decay rate is $\frac{1}{\sqrt{n}}$ as justified in Theorem \ref{thm:main}, which is clearly shown in the upper-right panel. The lower-left panel depicts how error rates of estimating the transition matrix in $L_1$ norm decay in the sample size. The curve is a little wiggling, reflecting the involvement of $\pi_{\min}$ in the bound given in the Theorem \ref{thm:main}. The lower right panel is shown to verify that our estimated matrix is indeed low-incoherent, consistent with our theory.}
    \label{fig:markov-error-dependence-on-n}
\end{figure}

Figure \ref{fig:markov-error-dependence-on-n} shows how the error of transition matrix estimation depends on the sample size. Each point in the graph represents an average of over 20 trials. The upper-left panel depicts how error rates of estimating the frequency matrix in Frobenius norm decay with sample size. From the upper-right panel, the decay rate is approximately $n^{-1/2}$ as justified in Theorem \ref{thm:main}. 


The lower-left panel depicts how error rates of the transition matrix estimation in $L_1$ norm decay in the sample size. The curve is a little wiggling from exact linear.  This is also reflected by the fact that the estimator $\widehat{P}$ is based on $\widehat{F}$ and it involves $\pi_{\min}$ in the bound given in the Theorem \ref{thm:main}. The lower left panel is to verify that our estimated matrix is indeed of low incoherence, in accordance with our theory.
\\

\emph{Experiment 3}:
In the third experiment, we examine the dependence of Frobenius norm error on the dimension $p$.
The number of samples is fixed to be $10^5$ while $p$ is taken from $\{30,35,40,45,50,100,200,500\}$. We also compare two slightly different algorithmic variants as specified in Section \ref{sec:practical-algorithm}. In particular, recall from the last section that $\emph{Method2}$ represents the one taking into account the incoherent constraint in each iteration while $\emph{Method1}$ is the one without the incoherent constraints, where we take $\mu=5$. Figure \ref{fig:markov-error-dependence-on-p-methods} depicts the results.  They show an approximately ${1}/{\sqrt{p}}$ decay rate of the error, which is coherent with our theory. Also both two panels show that including the incoherence condition in each iteration doesn't change the performance too much.
\\

\begin{figure}[t]
\begin{center}
    \begin{tabular}{cc}   \includegraphics[width=0.4\linewidth]{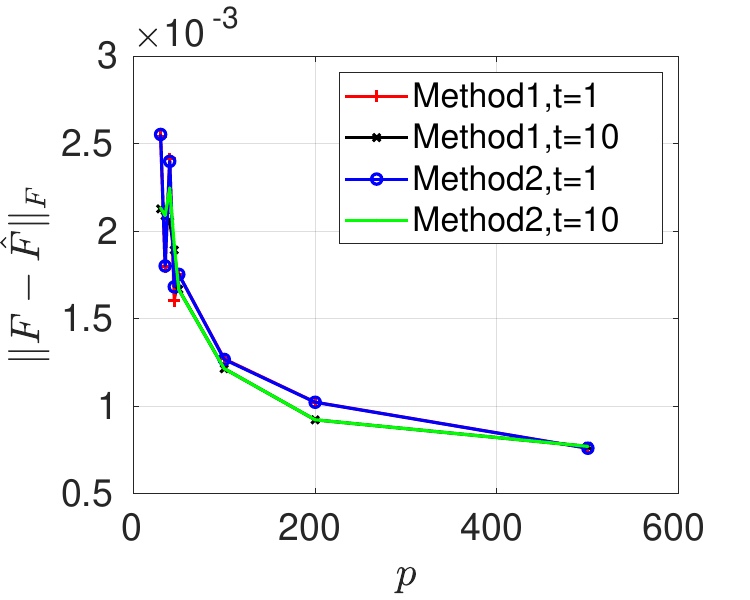}
    &\includegraphics[width=0.4\linewidth]{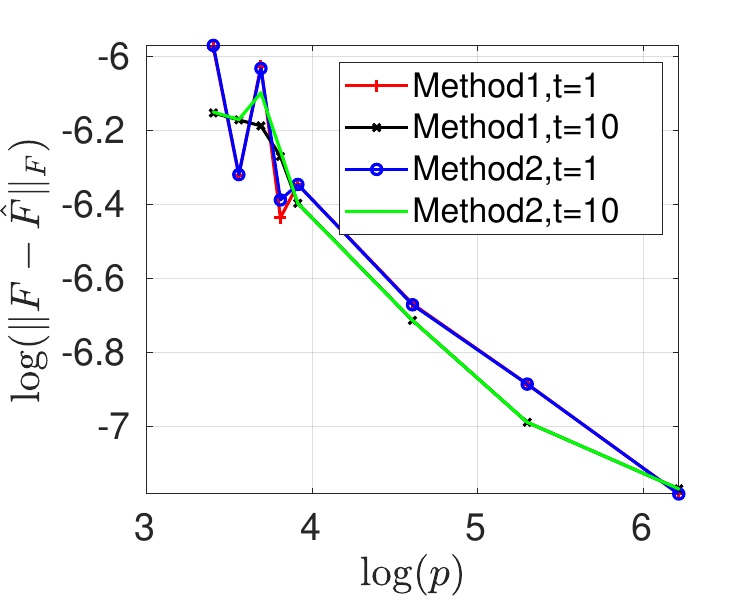}
    \tabularnewline
    
    $\Vert F-\widehat{F}\Vert_{F}\quad \text{v.s.}\quad p$&$\log\left(\Vert F-\widehat{F}\Vert_{F}\right)\quad \text{v.s.}\quad  \log(p)$\tabularnewline
    
    \end{tabular}

    \end{center}
    \caption{\emph{Alternating minimization method for estimating Markov transition matrix (continued)}. We plot in the left panel how error rates of estimating the frequency matrix in Frobenius norm decay with dimension $p$ using two methods. From the right panel, the decay is a typical square root $p$ as justifies in Theorem \ref{thm:main}. }
    \label{fig:markov-error-dependence-on-p-methods}
\end{figure}

\emph{Experiment 4}:
In the fourth experiment, we compare our method with the spectral method in \cite{zhang2019spectral}. The results are depicted in Figure \ref{fig:markov-error-comparison-spectral}.
To account for different signal strength  ratio, we let $t$ take value from $\{1,10\}$. The left panel presents the Frobenius norm error of frequency matrix while the right panel draws the $L_1$ norm error of the transition matrix. As our estimator is tailored for the low-rank-plus-sparse transition matrix, it's not surprising that our estimator by far outperform the spectral estimator. 

\begin{figure}[t]
\begin{center}
    \begin{tabular}{cc}   \includegraphics[width=0.4\linewidth]{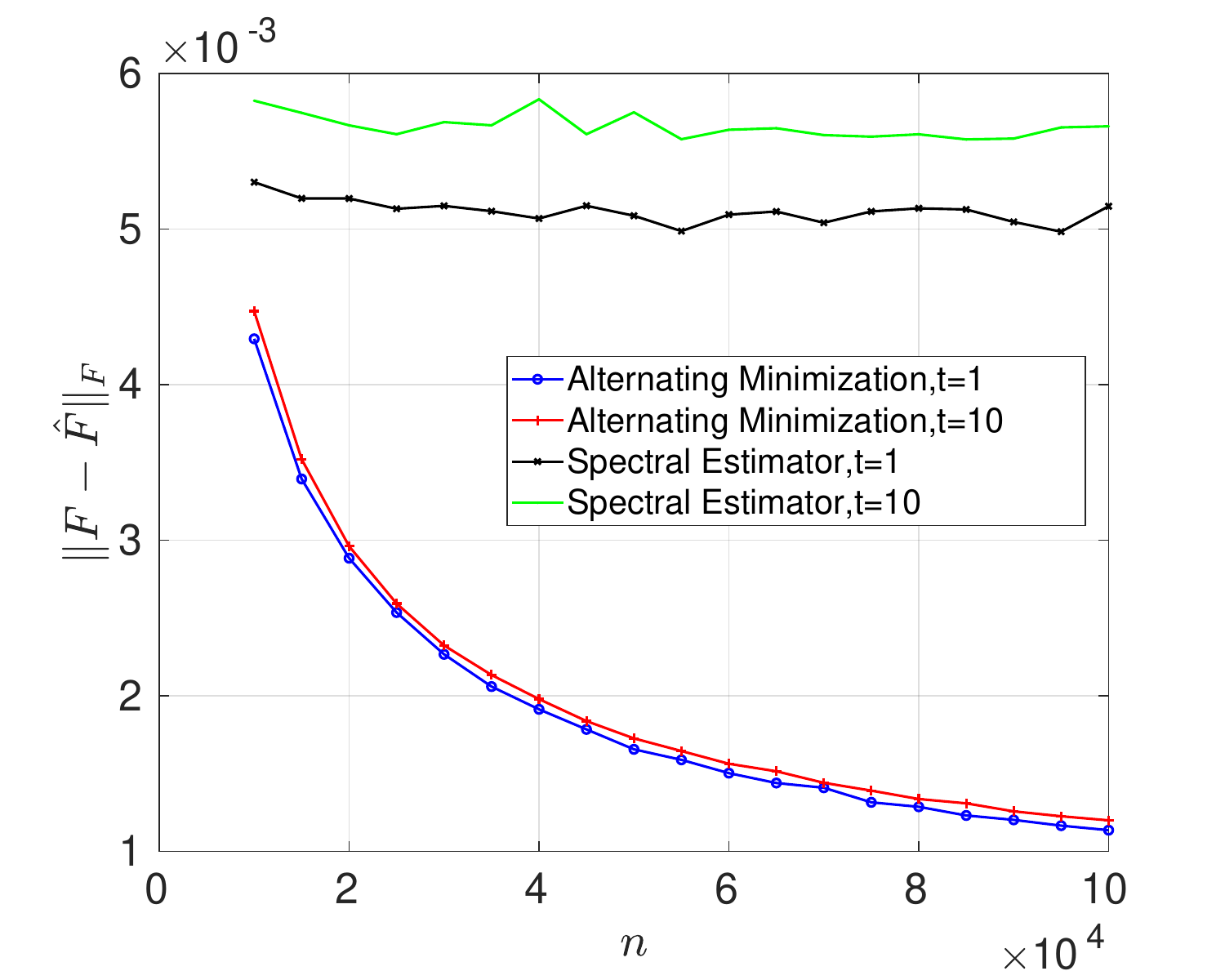}
    &\includegraphics[width=0.4\linewidth]{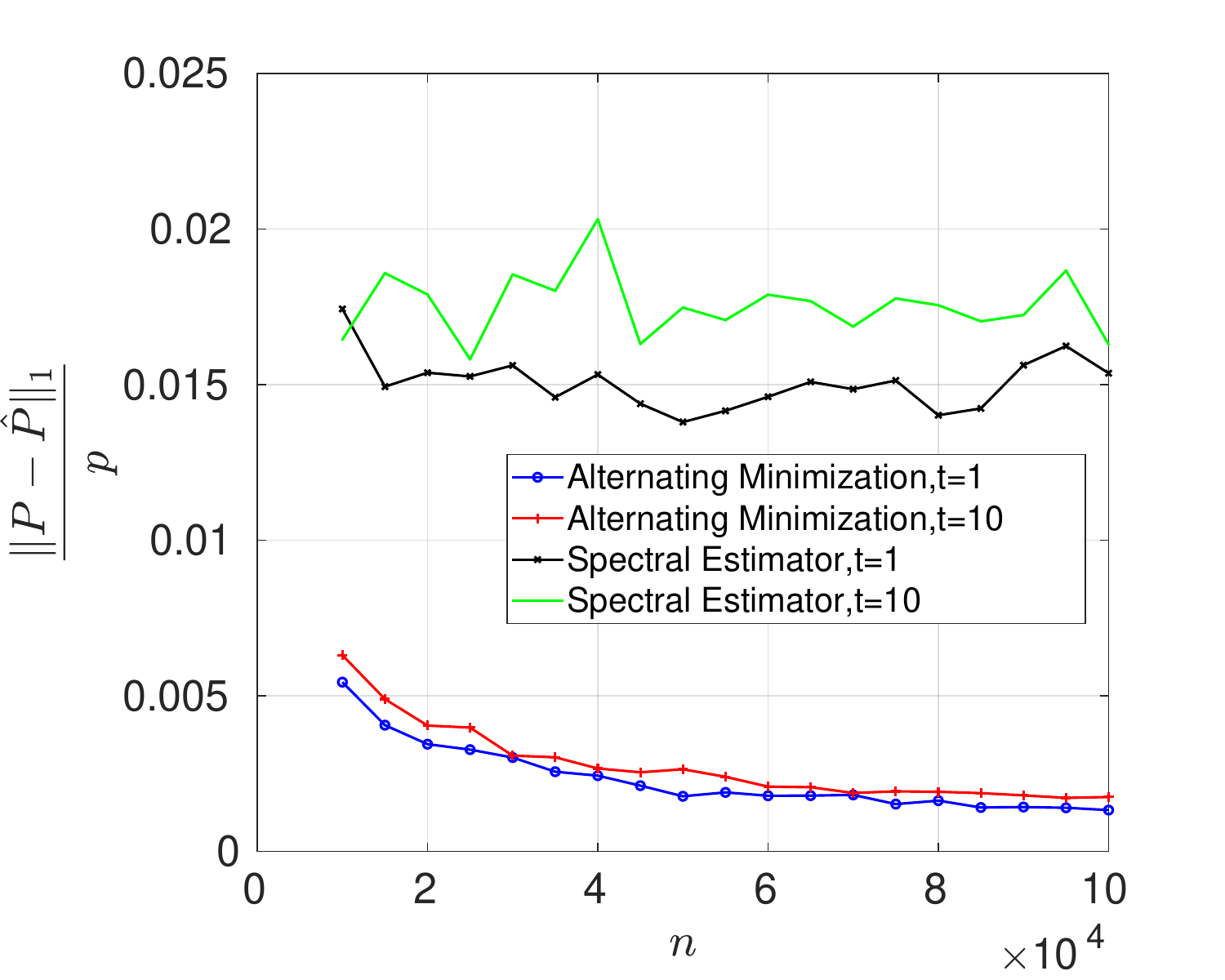}
    \tabularnewline
    
    $\Vert F-\widehat{F}\Vert_{F}\quad \text{v.s.}\quad n$&${\Vert P-\widehat{P}\Vert_{1}}/{p}\quad \text{v.s.}\quad  n$\tabularnewline
    
    \end{tabular}

    \end{center}
    \caption{\emph{The comparison of the spectral estimator and alternating minimization method in synthetic data}. The left panel presents the Frobenius norm error of the frequency matrix, while the right panel depicts the $L_1$ norm error of the transition matrix. The estimation error of the alternating minimization method decreases monotonically as sample size increases, whereas the spectral method does not due to irreducible bias.}
    \label{fig:markov-error-comparison-spectral}
\end{figure}

\subsection{Real Data Experiment}
We also evaluate our methodology using a real-world dataset of taxi trip records\footnote{The dataset is available at \href{https://www.nyc.gov/site/tlc/about/tlc-trip-record-data.page}{https://www.nyc.gov/site/tlc/about/tlc-trip-record-data.page}.}. This dataset contains taxi trip information collected and provided to the NYC Taxi and Limousine Commission (TLC), including details such as pick-up and drop-off locations, trip distances, itemized fares. 

For our experiment, we gathered all trips in January 2024, resulting in a total number of $2964624$ trips. For each trip, we have a one-step transition from the pick-up location to the drop-off location, both of which are districts already categorized. In preprocessing, we exclude rarely visited districts---specifically, those appearing fewer than $100$ times---leaving $74$ districts. Using this filtered data, we construct an empirical transition matrix of dimension $p=74$. Given the massive number of trips, we can assume the empirical transition matrix is sufficiently precise. Therefore, we denote that as $P^\star$ and treat as the ground truth for subsequent semi-synthetic experiments.
\begin{figure}[h]
    \centering
    \includegraphics[width=0.4\linewidth]{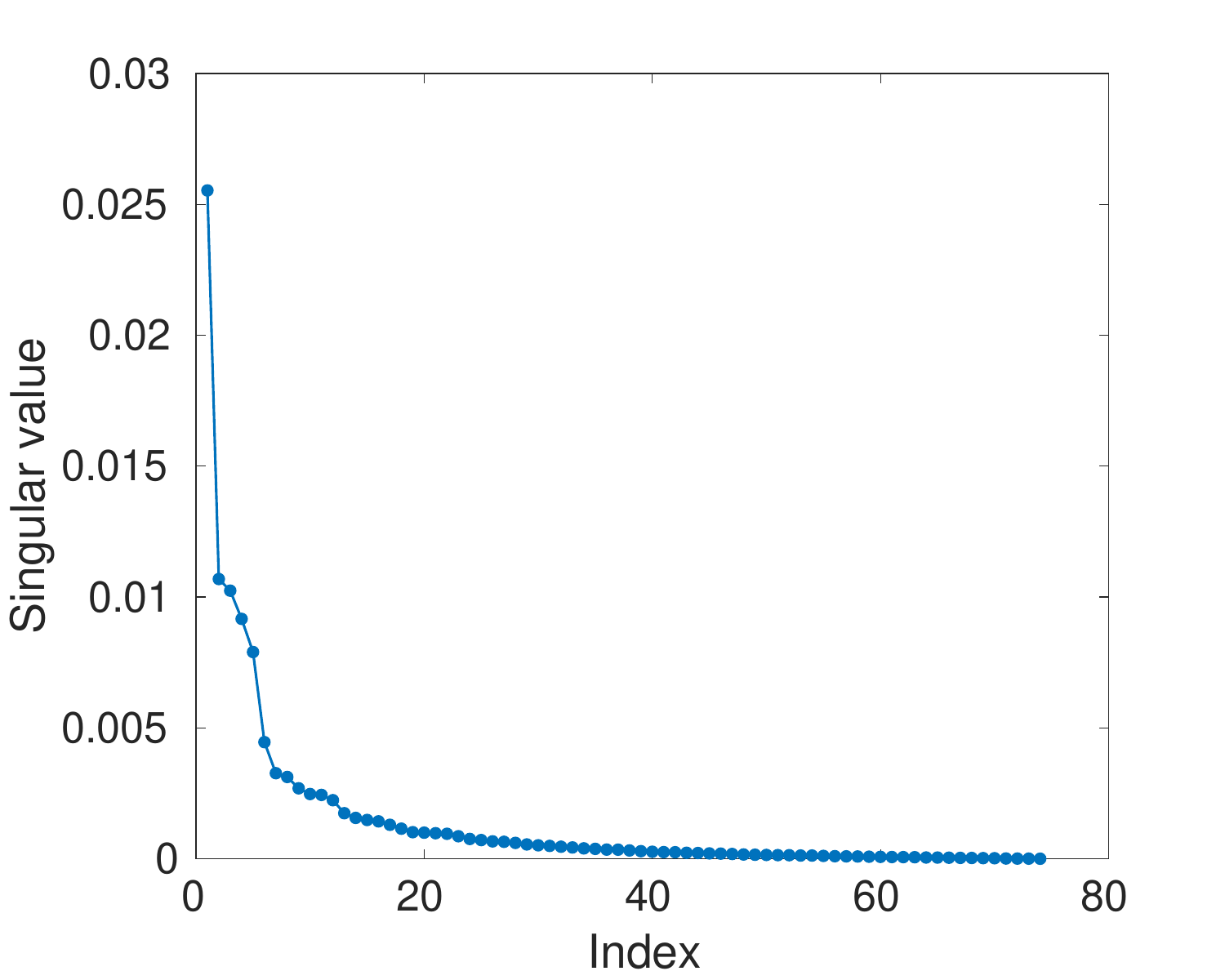}
    \caption{Singular value plot of $P^\star$ in real data}
    \label{fig:real-data-singular-values}
\end{figure}

To continue, we examine the approximation capability of our low-rank-plus-sparse model, in comparison with low-rank model~\citep{zhang2019spectral}. 
We begin by plotting the singular values of $P^\star$ in Figure~\ref{fig:real-data-singular-values}. The plot shows that the singular values decay to zero at around $30$-th component, suggesting that a purely low-rank assumption may not be appropriate.
For our analysis, we set $r=10$ and $s=50$. For the low-rank model, we input $F^\star$ and compute the best approximator $F_L$, which is equivalent to the top $r$ singular components of $F^\star$. We can then conduct projection and row-wise normalization to derive $P_L$. For the low-rank-sparse model, we input $F^\star$ to Algorithm 1 to obtain the best approximator $F_{LS}$, from which we derive $P_{LS}$ in a similar way. The approximation results are shown in Table \ref{tab:real-data-approximation}. Compared to the low-rank model, the low-rank-plus-sparse model achieves a significantly better fit to the real data, substantially reducing the approximation error.

\begin{table}[h!]
	\centering
	\begin{tabular}{c|c||c|c}
		\hline
		$\|F^\star-F_{LS}\|_F$ & 0.0037  &  $\|F^\star-F_{L}\|_F$ & 0.0349\\ \hline
		$\|P^\star-P_{LS}\|_1/p$& 0.0075 & $\|P^\star-P_{L}\|_1/p$ & 0.0363\\ \hline
	\end{tabular}
	\caption{Approximation results in real data (subscripts $L$ and $LS$ represent low-rank, and low-rank-plus-sparse model.)}
	\label{tab:real-data-approximation}
\end{table}

Furthermore, we conduct semi-synthetic experiments, generating Markov chain data regarding $P^\star$ as the ground truth transition matrix. For sample size $n$ in $\{5000k:1\le k \le 20\}$, we generate a single trajectory of length $n$. We compare our alternating minimization method to the spectral method, with results averaged over $100$ trials. The outcomes are plotted in Figure \ref{fig:real-data-comparison}. Similar to Figure~\ref{fig:markov-error-comparison-spectral}, our alternating minimization method significantly outperforms the spectral method. Notably, the estimation error of the spectral method does not decrease as the sample size increases, indicating the presence of irreducible bias. In contrast, the low-rank-plus-sparse model, along with our method, aligns more closely with the data, achieving better overall performance.

\begin{figure}[t]
\begin{center}
    \begin{tabular}{cc}   \includegraphics[width=0.4\linewidth]{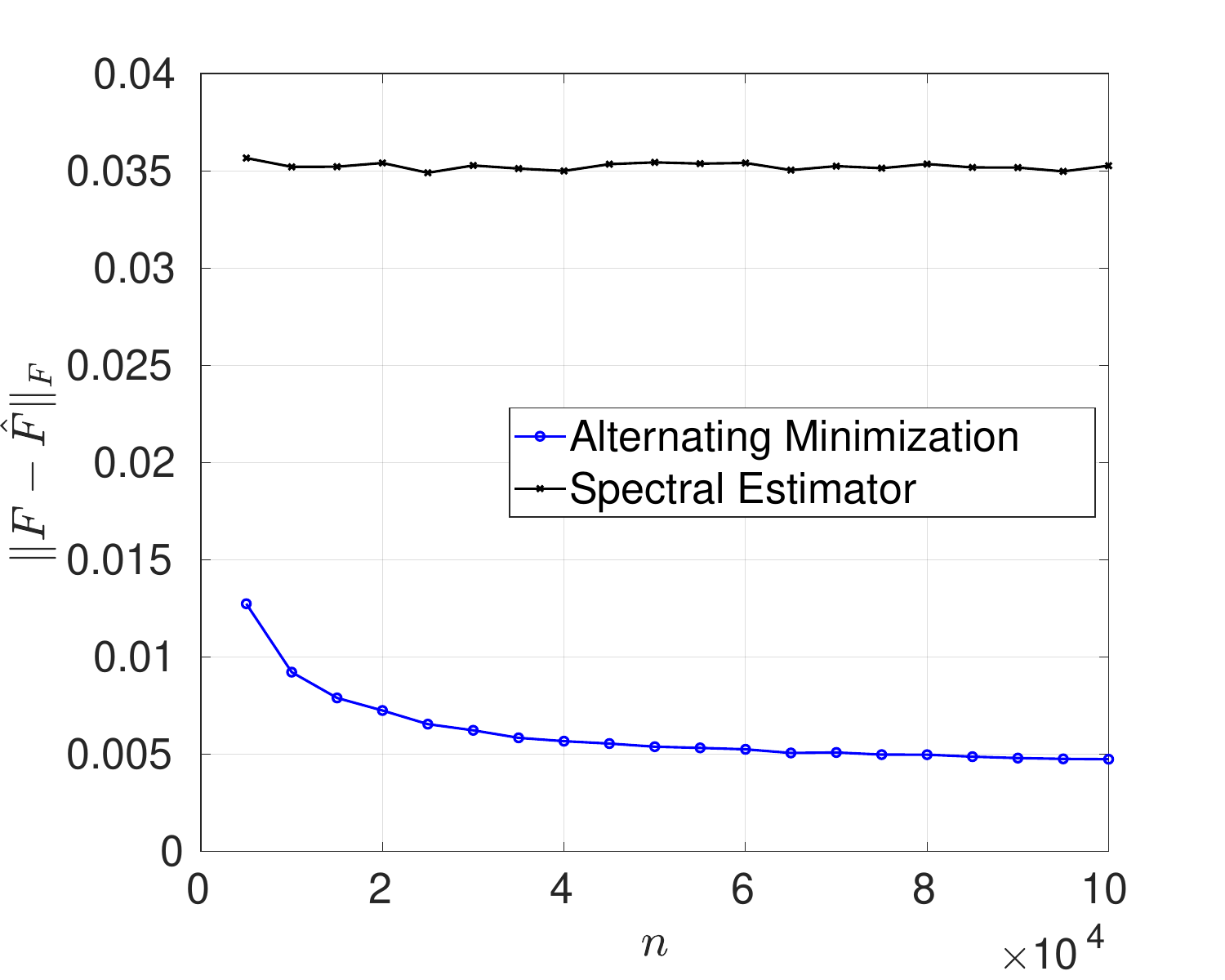}  &\includegraphics[width=0.4\linewidth]{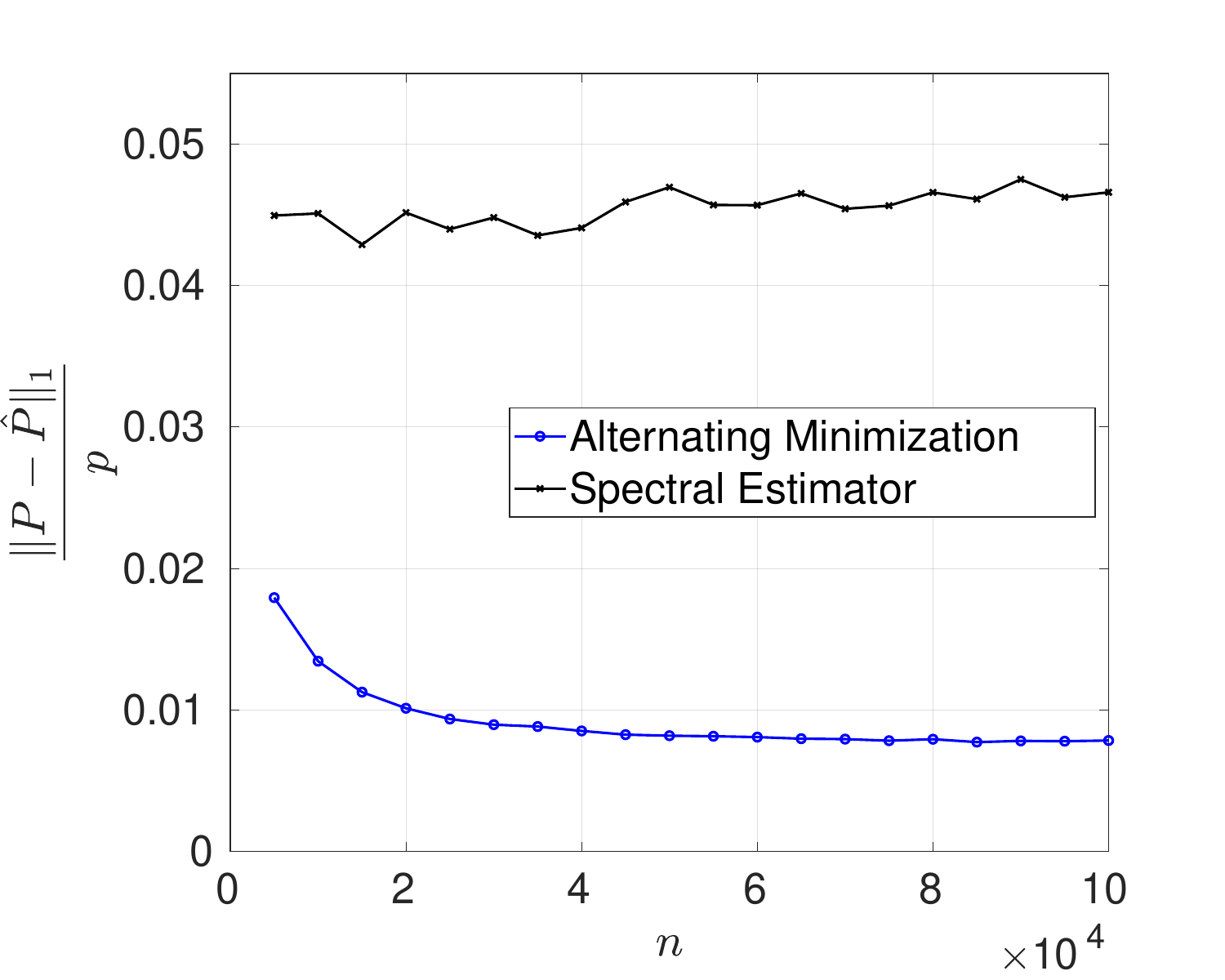}
    \tabularnewline
    
    $\Vert F-\widehat{F}\Vert_{F}\quad \text{v.s.}\quad n$&${\Vert P-\widehat{P}\Vert_{1}}/{p}\quad \text{v.s.}\quad  n$\tabularnewline
    
    \end{tabular}
    \end{center}
    \caption{\emph{The Comparison of spectral estimator (in blue) and alternating minimization method (in black) in real data}. The left panel presents the Frobenius norm error of the frequency matrix, while the right panel depicts the $L_1$ norm error of the transition matrix. The estimation error of the alternating minimization method decreases approximately monotonically as sample size increases, whereas the spectral method does not due to irreducible bias.}
    \label{fig:real-data-comparison}
\end{figure}

\section{Conclusion}
\label{sec:conclusion}
In this paper, we consider the structured matrix denoising problem. In particular, the structured matrix can be written as the sum of a low-rank matrix and a sparse matrix. Our primary contribution lies in developing a deterministic error bound that does not rely on distributional assumptions, a departure from prevalent reliance on such assumptions in existing literature. The methodology we propose is an optimization procedure centered on an incoherence-constrained least-squares objective. We prove the method is optimal in a wide range of noise settings. The key intermediate step underlying the success of the constrained least-square optimization hinges on showing that the difference between two arbitrary incoherent-low-rank matrices has energy spread out across entries or cannot be too sparse. The proof of the lemma is composed of several steps, starting from simple cases. Our primary motivating example revolves around the estimation of the Markov chain transition kernel. The difference of the ground-truth transition matrix and empirical transition matrix, which can be regarded as the noise, is not independent across entries. Therefore, we find it necessary to relax these assumptions and consider deterministic recovery instead. The method provably attains the minimax-lower bound in the class of low-rank-plus-sparse transition matrices. To solve the optimization problem, we propose an alternating minimization algorithm that alternates between optimizing the low-rank component and the sparse component. Other related problems are considered, for example, multitask regression and robust structured covariance estimation, where assumptions are relaxed compared to the factor model. Numerical results show this algorithm typically converges to the solution in a few steps, highlighting its practical utility and efficacy.

\newpage
\begin{appendix}
\begin{center}
{\LARGE Supplementary Material}
\end{center}

In this supplementary material, Appendix \ref{sec:related-work} provides a discussion of the related works. Appendix \ref{sec:main-proof} contains the main proofs of results in Section \ref{sec:matrix-decomposition}, while Appendix \ref{sec:proof-section-3} includes the proofs related to Section \ref{sec:markov}. We then elaborate on the proof of the key separation lemma in Appendix \ref{sec:key-lemma}. We include in Appendix \ref{sec:proof-further-example} the proofs in Section \ref{sec:further-example}. Appendix \ref{sec:extension-rectangular} is devoted to the extension of rectangular estimand. Several technical lemmas are listed in the Appendix \ref{sec:technical}.

\section{Related work}
\label{sec:related-work}
Our work is related to the following lines of research directions. 
\begin{itemize}
    \item Noisy Matrix Recovery: Noisy matrix recovery has been studied intensively in different settings such as noisy matrix completion~\citep{candes2010matrix,candes2012exact,gross2011recovering,chen2019inference,chen2020noisy} and robust PCA~\citep{candes2011robust,chandrasekaran2011rank,chen2021bridging}. There are mainly two popular ways to formulate the estimation problem, namely  convex relaxation~\citep{negahban2012restricted,koltchinskii2011nuclear,chen2020noisy,chen2021bridging} and nonconvex Burer–Monteiro approach with gradient descent~\citep{sun2016guaranteed,ma2018implicit,chen2020nonconvex} or projected descent~\citep{chen2015fast,zheng2016convergence}. See overview papers~\citep{chi2019nonconvex,chen2018harnessing}. With no missing entries, low-rank-and-sparse noisy matrix recovery has been studied in \cite{hsu2011robust,agarwal2012noisy,chandrasekaran2011rank,netrapalli2014non}. Among those, one of the most relevant works to us is \cite{agarwal2012noisy}, where the problem of study is under the name of noisy matrix decomposition. Other related matrix or vector sensing problems include blind deconvolution~\citep{ma2018implicit,chen2023convex}, phase retrieval~\citep{ma2018implicit}, tensor decomposition~\citep{ge2017optimization}, etc.

    \item Markov Transition Estimation: There has been a flurry of research on transition estimation. Some works consider the plain tabular case with no statistical structures imposed, such as \cite{hao2018learning}, \cite{wolfer2019minimax}, and \cite{wolfer2021statistical}. However, the rate inevitably depends on the size of the state space, which can be huge. To mitigate this dimension issue, \cite{zhang2019spectral} imposes low-rankness and derives the minimax rate of estimation; see also \cite{zhu2022learning}. Section 3 in \cite{wolfer2024empirical} considers a sparse transition matrix and presents improved estimation results. To the best of our knowledge, our work is the first to consider the low-rank-plus-sparse structure on the transition matrix.

    \item Spectral methods: Spectral methods have become a paradigm for many statistical machine problems~\citep{chen2021spectral}, and have been a go-to choice for various kinds of problems, including PCA and factor analysis~\citep{johnstone2018pca, fan2018principal}, community detection~\citep{abbe2017community}, clustering~\citep{ndaoud2018sharp,loffler2021optimality}, ranking~\citep{chen2015spectral}, covariance estimation~\citep{fan2013large}, etc. Although our approach does not utilize eigenvector perturbation theory due to the existence of a potentially large sparse part, the techniques bear a lot of similarity with the spectral methods, which intensively use singular value decomposition and matrix analysis. 

    \item Other related works include regularized M-estimator~\citep{negahban2012unified}, discrete distribution estimation~\citep{kamath2015learning}, tail inequality and concentration inequality regarding Markov Chains~\citep{paulin2015concentration,adamczak2008tail}.
\end{itemize}

For clearness of comparison, we list results of previous works under a similar deterministic low-rank-plus-sparse matrix recovery setting to ours in Table \ref{tab:related-work}. As can be seen in the table, our results are the best among all, the tightness of which will be argued later. Here $r$ is the rank of the low-rank part $L^*$, $s$ represents the sparsity of the sparse part $S^*$, $W$ denotes the noise matrix, while $p$ is the dimension of $L^*$ and $S^*$. Full details of these notations and the definition of $\mathfrak{X}$ can be found at the beginning of Section 2.  

\begin{table}[!h]
\scriptsize
    \centering
    \begin{tabular}{c ccc}
       \hline &Obs model &Low-rank Asp&Error bound $\|\widehat{L}-L^{\star}\|_F^2+\|\widehat{S}-S^{\star}\|_F^2$\\\hline
       \cite{agarwal2012noisy}&General&Spikiness&
       $\mathcal{O}(r\|\mathfrak{X}^{\star}(W)\|^2+s \|\mathfrak{X}^{\star} (W)\|_{\max}^2+s \|L^{\star}\|_{\max}^2)$\\
       \cite{chandrasekaran2011rank}&Identity&Incoherence&$\mathcal{O}(\|W\|_1^2+p^2\| W\|_{*}^2)$\\
       \cite{netrapalli2014non}&Identity&Incoherence&$\mathcal{O}(r^2\|W\|_2^2+ {p^2}r^{-1}\|W\|_{\max}^2)$\\
       This work&General&Incoherence&$\mathcal{O}(r\|\mathfrak{X}^{\star}(W)\|^2+s \|\mathfrak{X}^{\star} (W)\|_{\max}^2)$ \\ \hline
    \end{tabular}
    \caption[Comparison of deterministic results for low-rank-plus-sparse matrix estimation]{Comparison of deterministic results for low-rank-plus-sparse matrix estimation. For ease of comparison, we assume RSC parameter $\kappa$ is a constant. Here, 'Obs' stands for 'Observation', 'Asp' stands for 'Assumption'.} 
   \label{tab:related-work} 
\end{table}

Here we give a detailed comparison of a few more related literature. \cite{agarwal2012noisy} bears the most resemblance to our work, which studied in a similar setting. Notably, their intermediate deterministic bound also accounts for arbitrary noise. However, they imposed a spikiness condition that each entry of the low-rank part of the ground-truth matrix is upper bounded in magnitude and the final deterministic bound involves a trailing term as appeared in Table \ref{tab:related-work}. A huge drawback is that this bound results in inconsistency where the estimation error is lower bounded from zero even with infinite samples. Our result improves this.

\cite{hsu2011robust} also considered the arbitrary noisy setting and gave bounds on estimating errors. However, compared to our work, there are several differences. Firstly, their bound is not tight in general. In particular, their bound on low-rank part is not optimal in their Theorem 3 and they can't get optimal Frobenius norm bound while ours can. Secondly, our results apply to every linear observation model with restricted strong convexity while their results only apply to the standard low-rank-plus-sparse matrix estimation. Thirdly, unlike their work, we don't need prior knowledge about the error scale. We have less hyperparameters to tune than theirs. Again, our conditions on the sparse part are more relaxed than theirs.

There is a more recent work on Robust PCA \citep{chen2021bridging}, which 
considered the setting where noise is i.i.d. and sub-Gaussian across
entries. As previously mentioned, our results don't need such an
assumption. Their techniques of analyzing the rate convergence of gradient
descent using leave-one-out highly rely on the assumption of independence
of noise across entries. Moreover, unlike theirs, we can discard the 
assumption on the randomness of the sparse locations, which is considered
fixed in our setting.

In contrast to the aforementioned papers that considered noisy observations, \cite{chandrasekaran2011rank} and \cite{candes2011robust} studied under the noiseless setting where they proved with high probability that exact recovery is possible. In contrast,  our setting involves noise, and therefore, fully recovering the true signal is impossible. Our conditions on the sparse part are more stringent than theirs.

On the other hand, as our setting handles more general settings such as multitask regression, our constraint on sparsity level is more stringent in that we require the sparsity to be at most the same order as the dimension, while some of the other works can allow for a nearly constant fraction of the entire entries. Whether our general recipe allows for less stringent sparsity is an interesting research direction.
Furthermore, our methods only require the approximate prior knowledge of the rank as well as the sparsity level, which can be of the order of the ground-truth value and won't impair our rate of convergence.

Note that the key separation lemma \ref{lemma:main} characterizes the sparsity of arbitrary difference of two incoherent low-rank matrices. One can impose a more fine-grained structure on sparsity such as row-wise or column-wise sparsity \citep{chandrasekaran2011rank} and prove similar separation results. This is also an interesting future direction.

\section{Main Proofs in Section \ref{sec:matrix-decomposition}}
\label{sec:main-proof}
\subsection{Proof of Theorem \ref{thm:main}}
\begin{proof}
The proof highly relies on the key separation lemma \ref{lemma:main} whose detailed proof is given in Appendix \ref{sec:key-lemma}.

As $\barmu\ge \mu,\ \barr\ge r,\ \bars\ge s$, the feasibility of the ground-truth $P^{\star}$ follows from the form of optimization problem \eqref{alg:main}. Using the optimality condition, we have the following,

\begin{equation}
    \frac{1}{2}\|Y-\mathfrak{X}(\hL+\hS)\|_F^2\le \frac{1}{2}\|Y-\mathfrak{X}(L^{\star}+S^{\star})\|_F^2.
\end{equation}

Denote $\Delta_L=\hL-L^{\star}$ and $\Delta_S=\hS-S^{\star}$. Use the linearity of $\mathfrak{X}$ and recall the definition of $W$ in \eqref{equ:model}, some calculation yields
\begin{align}
\frac{1}{2}\|\mathfrak{X}(\DL+\DS)\|_F^2&\le \langle W,\mathfrak{X}(\DL+\DS)\rangle\nonumber\\
&=\langle \mathfrak{X}^* (W),\DL+\DS\rangle\nonumber\\
&=\langle \mathfrak{X}^* (W),\DL\rangle+\langle\mathfrak{X}^*(W),\DS\rangle\nonumber \\
&\le \|\mathfrak{X}^*(W)\|\cdot \|\DL\|_*+\|\mathfrak{X}^* (W)\|_{\max} \|\DS\|_1\label{equ:basic-inequality-mid},
\end{align}
where the first equality uses the definition of the conjugate operator, and the last inequality uses the Cauchy-Schwarz inequality.

Notice by the Assumption \eqref{asp:L,S} and the definition of the feasible set of optimization problem \eqref{alg:main}, $\DL$ has rank at most $r+\barr$, $\|\DS\|_0\le s+\bars$.  Therefore we have 
\[\|\DL\|_*\le \sqrt{r+\barr}\|\DL\|_F\le \sqrt{2\barr}\|\DL\|_F,\quad \|\DS\|_1\le \sqrt{s+\bars}\|\DS\|_F\le \sqrt{2\bars}\|\DS\|_F.\]
Plugging in \eqref{equ:basic-inequality-mid}, we get
\begin{align}
\label{equ:basic-inequality-mid2}
\frac{1}{2}\|\mathfrak{X}(\DL+\DS)\|_F^2
&\le \|\mathfrak{X}^*(W)\| \sqrt{2\barr}\|\DL\|_F+\|\mathfrak{X}^* (W)\|_{\max} \sqrt{2\bars}\|\DS\|_F\nonumber\\
&\le \sqrt{\|\DL\|_F^2+\|\DS\|_F^2}\sqrt{2\barr\|\mathfrak{X}^*(W)\|^2+2\bars \|\mathfrak{X}^* (W)\|_{\max}^2}.
\end{align}

On the other hand, by the RSC Assumption \eqref{asp:RSC}, we have 
\begin{align}
    \frac{1}{2}\|\mathfrak{X}(\DL+\DS)\|_{\mathrm{F}}^2 &\geq \frac{\kappa}{2}\|\DL+\DS\|_{\mathrm{F}}^2-\tau \Phi^2(\DL+\DS)\nonumber\\
    &\ge \frac{\kappa}{2}\|\DL+\DS\|_{\mathrm{F}}^2-2\tau (\|\DL\|_*^2+\lambda^2\|\DS\|_1^2)\nonumber\\
    &\ge\frac{\kappa}{2}\|\DL+\DS\|_{\mathrm{F}}^2-2\tau(2\barr\|\DL\|_F^2+2\bars\lambda^2 \|\DS\|_F^2)\label{equ:RSC-mid},
\end{align}
where the second inequality follows from the definition of $\Phi$.

Now notice that $\Delta_L=\hL-L^{\star}=\hU \hSigma \hV^{\top}-U^{\star} \Sigma^{\star} V^{*T}$, where $\hU,\hV\in \mathcal{O}_{p,\barr}^{\barmu}$ and $U^{\star},V^{\star}\in \mathcal{O}_{p,r}^\mu\subset \mathcal{O}_{p,\barr}^{\barmu}$.
By Lemma \ref{lemma:main}, we have that
\begin{equation*}
    \frac{\|\Delta_L\|^2_{\max}}{\|\Delta_L\|_F^2}\le \frac{c\barmu \barr^3}{p}.
\end{equation*}

Let $\operatorname{Supp}_L$ and $\operatorname{Supp}_S$ be the support of $\Delta_L$ and $\Delta_S$ respectively, denote $\mathcal{J}:=\operatorname{Supp}_L\bigcap \operatorname{Supp}_S$ and $\mathcal{T}_\mathcal{J}$ as the projection operator that zeros out the elements in $\mathcal{J}^c$. By the definition of $\calJ$ and Assumption \ref{asp:L,S} we have  $|\mathcal{J}|\le 2\bars\le \frac{p}{4c\barmu\barr^3}$ where we take $\bar c=8c$.
Therefore, we can estimate the proportion of elements of $\Delta_L$ in $\mathcal{J}$ as follows,
\begin{equation*}
    \frac{\|\mathcal{T}_\mathcal{J}(\Delta_L)\|_F^2}{\|\Delta_L\|_F^2}\le |\mathcal{J}|\frac{\|\Delta_L\|^2_{\max}}{\|\Delta_L\|_F^2}\le \frac{c\barmu \barr^3}{p} \frac{p}{4c\barmu\barr^3}= \frac{1}{4}.
\end{equation*}
As a result, we have 
\begin{align*}
    |\langle \Delta_L,\Delta_S\rangle| &= |\langle \mathcal{T}_\mathcal{J}(\Delta_L),\mathcal{T}_\mathcal{J}(\Delta_S)\rangle|\\
    &\le \|\mathcal{T}_\mathcal{J}(\Delta_L)\|_F^2+\frac{1}{4}\|\mathcal{T}_\mathcal{J}(\Delta_S)\|_F^2  \\
    &\le \|\mathcal{T}_\mathcal{J}(\Delta_L)\|_F^2+\frac{1}{4}\|\Delta_S\|_F^2  \\
    &\le \frac{1}{4}\|\Delta_L\|_F^2+\frac{1}{4}\|\Delta_S\|_F^2,
\end{align*}
and hence 
\begin{equation*}
\|\DL+\DS\|_{\mathrm{F}}^2=\|\DL\|_F^2+\|\DS\|_F^2+2\langle \Delta_L,\Delta_S\rangle\ge \frac{1}{2}(\|\DL\|_F^2+\|\DS\|_F^2).
\end{equation*}

Recalling \eqref{equ:RSC-mid}, and also by Assumption \eqref{asp:RSC-parameters} on RSC parameters, we have

\begin{align}
\label{equ:basic-inequality-mid3}
    \frac{1}{2}\|\mathfrak{X}(\DL+\DS)\|_{\mathrm{F}}^2 
    &\ge\frac{\kappa}{4}(\|\DL\|_F^2+\|\DS\|_F^2)-2\tau (2\barr\|\DL\|_F^2+2\bars\lambda^2 \|\DS\|_F^2)\nonumber\\
    &\ge\frac{\kappa}{8}(\|\DL\|_F^2+\|\DS\|_F^2).
\end{align}

Combining \eqref{equ:basic-inequality-mid2} and \eqref{equ:basic-inequality-mid3}, we arrive at 

\begin{align}
    \|\DL\|_F^2+\|\DS\|_F^2\le \frac{128}{\kappa^2}\left(\barr\|\mathfrak{X}^*(W)\|^2+\bars \|\mathfrak{X}^* (W)\|_{\max}^2\right),
\end{align}
thus completing the proof.
\end{proof}

\subsection{Proof of Theorem \ref{thm:simple-noise-setting-upper-bound}}
\begin{proof}
\emph{Upper Bound:}
For the upper bound, we apply the deterministic result of our estimator for rectangular matrices, which we denote here as $\mathcal{E}(Y)=(\hL^{\mathcal{E}},\hS^{\mathcal{E}})$. In the first noise scenario, we have by the condition on the sparsity level, the first term in the bound \eqref{equ:error-bound-idendity-observation-model} dominates the second one and therefore, up to a constant,
\[\mathfrak{R}^{1}_\rho\le \mathbb{E}\rho(L,S,\hL^{\mathcal{E}},\hS^{\mathcal{E}})\lesssim  \mathbb{E}r\|W\|^2_2\lesssim \max\{p,q\}r\sigma^2, \]
where the last inequality bounds $\EE\|W\|_2$ and  follows in the same vein from Theorem 5.39 of \cite{vershynin2010introduction}, except now we have different covariance matrices across rows.

For the second case, denote the nonzero index set of $E$ as $\calJ_E$. For $t>0$, we have by union bound that with probability at least $1-2se^{-t^2/2}$ that $|W_{i,j}|\le \sigma t$ for every $(i,j)\in \calJ_E$. We denote this event as $\calA$. It holds that under $\calA$, $\|W\|\le \sigma t\|E\|$. 
Hence we have 
\begin{align*}
\EE[\|W\|^2]&=\EE[\|W\|^2\calI_\calA]+\EE[\|W\|^2\calI_{\calA^c}].
\end{align*}
Moreover, we have the bound $\EE[\|W\|^2\calI_\calA]\le \sigma^2 t^2 \|E\|^2$, as well as
\begin{align*}
\EE[\|W\|^2\calI_{\calA^c}]&\le \EE[\|W\|_F^2\calI_{\calA^c}]\\&\le
\sum_{(i,j)\in \calJ_E}\left\{\EE[W_{i,j}^2\calI(|W_{i,j}|\ge \sigma t)]+\sum_{(i',j')\in \calJ_E\backslash(i,j)}\EE[W_{i,j}^2\calI(|W_{i',j'}|\ge \sigma t)]\right\}\\&\lesssim
s^2\sigma^2(t+1)e^{-t^2/2},
\end{align*}
where we used that $\int_t^\infty x^2e^{-x^2/2}\lesssim (t+1)e^{-t^2/2}$. Take $t=C\log(s)$ with sufficiently large constant $C$, we have $\EE[\|W\|^2]\lesssim \sigma^2\|E\|^2$ up to log terms.
And hence $r\EE[\|W\|^2]\lesssim s\sigma^2\le s\EE[\|W\|_{\max}^2]$ by the assumption on $\|E\|$.

Therefore, the second term in the bound \eqref{equ:error-bound-idendity-observation-model} dominates the first one and, therefore, up to logarithmic factors,
\[\mathfrak{R}^{2}_\rho\le\mathbb{E}\rho(L,S,\hL^{\mathcal{E}},\hS^{\mathcal{E}})\lesssim  \mathbb{E}s\|W\|_{\max}^2\lesssim s\sigma^2 \]
where the bound on $\|W\|_{\max}$ follows from a similar union bound.

\emph{Lower Bound:}
We are left with the proof of the lower bounds.
The proof is based on the combination of Fano's method along with the construction of local packing using Gilbert-Varshamov bound.

For the first noise scenario, without loss of generality, assume $p>q$. Let $l_1=\lfloor \frac{q}{r}\rfloor$ and $l_2=q-rl_1$. Define the map $f(R)=\Big[
        \overbrace{R\ R\ \cdots\ R}^{l_1}\ 0_{p\times l_2}
    \Big]$. Consider the set of instances as 
\begin{align}
    \{(L,S):L=f(R); S=0; R\in \{-\eta,\eta\}^{p,r} \}
\end{align}
where $\eta$ is a parameter to be determined.

We invoke the following Gilbert-Varshamov bound in Lemma F.1 in the supplementary material which is included in Appendix F in the supplementary material for completeness (for a proof of this lemma, see \cite{yu1997assouad}). Therefore, we get a subset $A_1$ of $\{-\eta,\eta\}^{p,r}$ (flattened to vectors), such that $|A_1|\ge \exp(c_0 pr)$ and 
\[\|R_1-R_2\|^2_F\ge \eta^2\frac{4pr}{3}\]

Consider the subset of instances 
$\{(L,S):L=f(R); S=0; R\in A_1\}$, each pair of which are far apart in Frobenius norm in the sense that 
\[\|(f(R_1)-f(R_2))\|\\_F^2=l_1  \|R_1-R_2\|^2_F\ge l_1\eta^2\frac{4pr}{3}\ge \eta^2\frac{2pq}{3}.\] 
On the other hand, the KL-divergence between the isotropic Gaussian distributions with mean $f(R_1)$ and mean $f(R_2)$ is upper bounded by $\sum_{i=1}^p [f(R_1)-f(R_2)]_i^T\Sigma_i^{-1}[f(R_1)-f(R_2)]_i\lesssim \frac{pq\eta^2}{\sigma^2}$, where we use the chain rule for KL-divergence and KL-divergence for two multivariate Gaussian random variables with same covariance $\Sigma$ is given by $KL(\calN(\mu+\delta,\Sigma)|\calN(\mu,\Sigma))=\mu^T\Sigma^{-1}\mu/2$.

Applying Fano's method on $\{(L,S):L=f(R); S=0; R\in A_1\}$, we get that

\begin{align*}
\mathfrak{R}^{d_1}_\rho&\ge \eta^2\frac{2pq}{3}\left(1- \frac{pq\frac{4\eta^2}{\sigma^2}+\log 2}{\log(|A_1|)}\right)\\
&\ge \eta^2\frac{2pq}{3}\left(1- \frac{pq\frac{4\eta^2}{\sigma^2}+\log 2}{c_0 pr}\right).
\end{align*}

When $p\ge \frac{3\log 2}{c_0}$ and take $\eta=\sigma\sqrt{\frac{c_0 r}{12q}}$, we get 
\[\mathfrak{R}^{d_1}_\rho\ge \eta^2\frac{2pq}{9}\gtrsim pr\sigma^2=\max\{p,q\}r\sigma^2.\]

For the second noise scenario, let the predetermined support be $\Omega$. Construct the set of instances as 
\[\{(L,S):L=0;S=h(R);R\in \{-\eta,\eta\}^s\}\]
where the $h(R)$ is a matrix supported on $\Omega$ and whose entries on $\Omega$ corresponds to $R$.

Similarly using Gilbert-Varshamov bound we can find a subset $A_2$ of $\{-\eta,\eta\}^{s}$ such that $|A_1|\ge \exp(c_0 s)$ and 
\[\|R_1-R_2\|^2_F\ge \eta^2\frac{4s}{3}.\]

Similarly, consider $\{(L,S):L=0;S=h(R);R\in A_2\}$. Each pair is far apart in Frobenius norm in the sense that 
\[\|(h(R_1)-h(R_2))\|\\_F^2=\|R_1-R_2\|_F^2\ge \eta^2\frac{4s}{3}\]
and the KL-divergence between the distributions induced by $(h(R_1)$ and $h(R_2))$ is upper bounded by $s \frac{4\eta^2}{\sigma^2}$.

Applying Fano's method on $\{(L,S):L=0;S=h(R);R\in A_2\}$ and set $\eta=\sigma\frac{c_0}{12}$, we get that 
\begin{align*}
    \mathfrak{R}^{d_2}_\rho&\ge \eta^2\frac{4s}{3}\left(1- \frac{s\frac{4\eta^2}{\sigma^2}+\log 2}{\log(|A_2|)}\right)\\
    &\ge \eta^2\frac{4s}{3}\left(1- \frac{s\frac{4\eta^2}{\sigma^2}+\log 2}{c_0 s}\right)\\
    &\gtrsim s\sigma^2.
\end{align*}

\end{proof}

\section{Proofs in Section \ref{sec:markov}}
\label{sec:proof-section-3}

\subsection{Proof of Proposition \ref{prop:error-bound}}
\begin{proof}
    
    Recall that $\tau_*=\tau(\frac{1}{4})$. By (73) of the Lemma 7 of \cite{zhang2019spectral}, we have for $\forall c_0$, there exists a constant $C$ such that for all $n\ge C\tau_* p \log^2 n$, we have that \eqref{equ:error-bound-spectral-norm} holds.

    To prove the second part,
    consider a new Markov chain that takes the consecutive states in the original chain as the new state. Formally, define $Y=\{Y_0,Y_1,\cdots,Y_{n-1}\}$ as the new Markov chain taking values in the product space $[p]\times [p] \backslash Z$, with $Q$ its transition matrix, where 
    \begin{equation}
    \label{def:consecutive-markov-chain}
    Y_i=(X_i,X_{i+1}),\quad  Q_{(i,j),(k,l)}=1_{j=k}P^{\star}_{k,l}.
    \end{equation} and $Z=\{(i,j)\big|P_{i,j}=0\}$. To continue, we will need the following technical lemma whose proof is relegated to Section \ref{sec:technical}.
\begin{lemma}
    \label{lemma:new-consecutive-markov-chain}
    Let $Y=\{Y_0,Y_1,\cdots,Y_{n-1}\}$ be the new Markov chain defined in \eqref{def:consecutive-markov-chain}. Denote the $\epsilon$-mixing time of the original Markov chain $X=\{X_0,X_1,\cdots,X_{n-1}\}$ as $\tau_X(\epsilon)$ and transition matrix as $P$. Similarly, denote $\epsilon$-mixing time of $Y$ as $\tau_Y(\epsilon)$ and transition matrix as $Q$. Suppose $X$ is ergodic, then $Y$ is ergodic and $\tau_Y(\epsilon)=\tau_X(\epsilon)$.
    \end{lemma}

    Now by Lemma \ref{lemma:new-consecutive-markov-chain}, it holds that $Y$ is ergodic. Also by \eqref{equ:stationary-new-consecutive-mc}, \[\mu_{\max}= \max_{k,l}\pi_kP_{k,l}\le \pi_{\max}p_{\max}.\] Applying (74) of Lemma 7 of \cite{zhang2019spectral} on the new Markov chain $Y$, we have for all $n\ge C\tau_* p \log^2 n$,
    \[\PP\left(\|\tF-F^{\star}\|_{\max}\ge C\sqrt{\frac{p_{\max}\pi_{\max} \tau_Y(\frac{1}{4}) \log^2 n}{n}}\right)\le n^{-c_0}\]

    By the Lemma \ref{lemma:new-consecutive-markov-chain}, we have $\tau_Y(\frac{1}{4})=\tau_X(\frac{1}{4})=\tau_*$, and \eqref{equ:error-bound-max-norm} holds, thus completing the proof.
\end{proof}

\subsection{Proof of Theorem \ref{thm:markov-main}}

\begin{proof}
First of all, notice that it suffices to prove \eqref{equ:main-1}. As a matter of fact, \eqref{equ:main-2} follows from $\|\hF-F^{\star}\|_1\le p\|\hF-F^{\star}\|_F$. And by applying Lemma 2 in \cite{zhang2019spectral} to the rows of $\hP-P^{\star}$ and summing up, we have $\|\hP-P^{\star}\|_1\le \frac{\|\hF-F^{\star}\|_1}{\pi_{\min}}$ which implies \eqref{equ:main-3}. 

Combine Theorem \ref{thm:main} and Proposition \ref{prop:error-bound}, we have that 
\begin{equation*}
        \|\widehat{F}^0-F^{\star}\|_F^2\le c 
        \frac{\pi_{\max}\tau_*\log^2 n}{n}(r+p_{\max}s)
\end{equation*}
    
Now note that the truncation and projection step cannot increase the distance to the ground truth. Specifically, we have $|\hF^1_{i,j}-F^{\star}_{i,j}|=|\min(\max(\widehat{F}^0,0),1)-F^{\star}_{i,j}|\le |\hF^0_{i,j}-F^{\star}|$ and therefore, $\|\hF^1-F^{\star}\|_F\le \|\hF^0-F^{\star}\|_F$.
Similarly, by the definition of projection in Euclidean distance and the fact that $F^{\star}\in\mathcal{F}_p$, we have $\|\hF-F^{\star}\|_F\le \|\hF^1-F^{\star}\|_F$. Therefore,

\begin{equation*}
        \|\widehat{F}-F^{\star}\|_F^2\le\|\widehat{F}^0-F^{\star}\|_F^2\le c 
        \frac{\pi_{\max}\tau_*\log^2 n}{n}(r+p_{\max}s)
\end{equation*}
which finishes the proof.
\end{proof}

\subsection{Proof of Corollary \ref{cor:conditional-mean-bound}}
\begin{proof}
    Let $\hat{\Tau}(v):=\hat{P}v$, where $\hat{P}$ is computed by Algorithm 2. Then we have 
    \begin{align*}
    \EE\|\hat{\Tau}(v)-\Tau(v)\|_2^2&=\EE \left(v^{\top} (\hat{P}-P)^{\top}(\hat{P}-P) v\right)=\operatorname{Tr}\left(\EE(vv^{\top})(\hat{P}-P)^{\top}(\hat{P}-P)\right)\\&\le c_v\operatorname{Tr}\left(\EE(\hat{P}-P)^{\top}(\hat{P}-P)\right)=c_v\EE\|\hat{P}-P\|_F^2
    \end{align*}
    Now by the previous result on estimating the transition kernel, we get
    \[\EE\|\hat{\Tau}(v)-\Tau(v)\|_2^2\le c_v 
        \frac{c\pi_{\max}\tau_*\log^2 n}{n\pi_{\min}^2}(r+p_{\max}s)\]as desired.
        Further, if $p_{\max}=\operatorname{O}(\frac{1}{p})$, $\pi_{\max}\asymp \pi_{\min}$ and $\tau_*=\operatorname{O}(1)$, we have $c_v 
        \frac{c\pi_{\max}\tau_*\log^2 n}{n\pi_{\min}^2}(r+p_{\max}s)=\operatorname{O}\left(c_v \frac{rp\log^2 n}{n}
        \right)$.
\end{proof}

\subsection{Proof of Theorem \ref{thm:rl-fixed-v}}
\begin{proof}
    We construct two fuzzy low-rank transition matrices which are hard to statistically distinguish and use Le Cam's theorem.

    First of all, note that 
    \[\hat{\mathcal{T}}(v)-\mathcal{T}(v)=(\hP-P)v=(\hP-P)(v-\Bar{v})\]

    Consider the following two transition kernels $P$ and $Q$. Let $P_{ij}=\frac{1}{p}$ for every $i,j$. Let $Q_{ij}=\frac{1}{p}+\eta u_j$ for every $i,j$, where $\eta$ is to be determined. Note that $P$ and $Q$ are both rank-1 transition matrices.

    Define the distribution of trajectory $X_0,X_1,X_2,\cdots,X_n$ induced by $P$ as $P^n$, starting from uniform distribution.
    The KL divergence between $P^n$ and $Q^n$ is calculated as 

    \begin{align*}
        D_{KL}(P^n\|Q^n)&=\sum_{i_0,\cdots,i_n\in [p]}P^n(i_0,\cdots,i_n)\log\left(\frac{P^n(i_0,\cdots,i_n)}{Q^n(i_0,\cdots,i_n)}\right)\\
        &=\sum_{i_0,\cdots,i_n\in [p]}\frac{1}{p}P_{i_0,i_1}\cdots P_{i_{n-1},i_n}\log\left(\frac{P_{i_0,i_1}\cdots P_{i_{n-1},i_n}}{Q_{i_0,i_1}\cdots Q_{i_{n-1},i_n}}\right)\\
        &=\sum_{i_0,\cdots,i_n\in [p]}\frac{1}{p}P_{i_0,i_1}\cdots P_{i_{n-1},i_n}\left(\log\left(\frac{P_{i_0,i_1}\cdots P_{i_{n-2},i_{n-1}}}{Q_{i_0,i_1}\cdots Q_{i_{n-2},i_{n-1}}}\right)+\log\left(\frac{P_{i_{n-1},i_n}}{Q_{i_{n-1},i_n}}\right)\right)\\
        &=D_{KL}(P^{n-1}\|Q^{n-1})+\sum_{i_0,\cdots,i_{n-1}\in [p]}\frac{1}{p}P_{i_0,i_1}\cdots P_{i_{n-2},i_{n-1}}\sum_{i_n}P_{i_{n-1},i_n}\log\left(\frac{P_{i_{n-1},i_n}}{Q_{i_{n-1},i_n}}\right)\\
        &=D_{KL}(P^{n-1}\|Q^{n-1})+D_{KL}(\Bar{p}\|\Bar{q})
        \\&=nD_{KL}(\Bar{p}\|\Bar{q})
    \end{align*}
    where $\Bar{p}=(\frac{1}{p},\frac{1}{p},\cdots,\frac{1}{p})$ and $\Bar{q}=(\frac{1}{p}+\eta u_1,\frac{1}{p}+\eta u_2,\cdots,\frac{1}{p}+\eta u_p)$. We also use the probability vector to denote the discrete distribution. The second to last equality is because $P_{ij}$ and $Q_{ij}$ both don't depend on $i$, and 
    the last equality is by reduction. 

In addition, we have by Lemma 4 of \cite{zhang2019spectral}, 
\[D_{KL}(\Bar{p}\|\Bar{q})\le 3p \|\Bar{p}-\Bar{q}\|_2^2=3p\eta^2 \|u\|_2^2\] given the condition that $\eta\le \frac{1}{2p\|u\|_{\max}}$.

By Le Cam's theorem, see \cite{wainwright2019high} for example,
\begin{align*}
\mathfrak{R}_1(v)&\ge \frac{\|(P-Q)v\|_2^2}{2}\left(1-\sqrt{\frac{D_{KL}(P^n\|Q^n)}{2}}\right)\\&=\frac{\|(P-Q)v\|_2^2}{2}\left(1-\sqrt{\frac{nD_{KL}(\Bar{p}\|\Bar{q})}{2}}\right)\\&\ge
\frac{p\eta \|u\|_2^2}{2}\left(1-\sqrt{\frac{n\cdot 3p\eta^2 \|u\|_2^2}{2}}\right)
\end{align*}

Let $\eta=\frac{1}{3\|u\|_2 \sqrt{np}}$. We have $\eta\le \frac{1}{2p\|u\|_{\max}}$ as long as $n\ge \frac{4p\|u\|_{\max}^2}{9\|u\|_2^2}$.

Plug in the value of $\eta$ we have that
\[\mathfrak{R}_r(v)\ge \mathfrak{R}_1(v)\ge\frac{\|u\|_2^2}{36n}\]
as desired.
\end{proof}

\section{Key Separation Lemma}
\label{sec:key-lemma}
In this section, we aim to prove the key separation Lemma \ref{lemma:main} which claims the difference between two incoherent low-rank matrices should distribute relatively evenly across elements. In other words, incoherent low-rank matrices are in some sense well-separated. Section \ref{sec:proof-sketch} lists a number of intermediate lemmas to achieve Lemma \ref{lemma:main}. Section \ref{sec:proof-key-lemma} is devoted to the proof of Lemma \ref{lemma:main}. In Section \ref{sec:optimality-key-lemma}, we discuss the optimality of Lemma \ref{lemma:main} with respect to the order of dimension $p$.
With slight abuse of notation, the $L$ in this section is referred to as the left orthogonal part of $Q's$ SVD instead of the low-rank part in the main text. To prevent confusion, all notations used in this section do not apply to other sections except Section \ref{sec:extension-rectangular}. 

\begin{lemma}
\label{lemma:main}
    Let $P=U \Sigma V^{\top}$ and $Q=L \Lambda R^{\top}$ be the SVD of $P,Q\in \mathbb{R}^{p\times p}$, $U,V,L,R\in \mathcal{O}_{p,r}$. Denote $\sigma_1\ge \sigma_2\ge\cdots\ge\sigma_r\ge 0$ and $\lambda_1\ge \lambda_2\ge\cdots\ge\lambda_r\ge 0$ as the diagonal elements of $\Sigma$ and $\Lambda$, respectively.
    Suppose $P$ and $Q$ both satisfy the $\mu$-incoherence condition \ref{def:incoherence}, then we have for $\Delta=P-Q$,
    \begin{equation}
    \label{equ:key-lemma}\frac{\|\Delta\|_{\max}^2}{\|\Delta\|_F^2}\le \frac{\Tilde{c}\mu r^3}{p}\end{equation}
where $\Tilde{c}$ is a fixed constant independent of both $P$ and $Q$.
\end{lemma}

\subsection{Proof sketch}
\label{sec:proof-sketch}
We first prove a simple discrepancy lemma \ref{lemma:norm-discrepancy}. Then, we do a more refined discrepancy analysis in Lemma \ref{lemma:discrepancy-ab} and \ref{lemma:discrepancy-rv} in the sense of norm discrepancy and direction discrepancy, respectively. Combining them, we get Lemma \ref{lemma:orthogonal-case}, which is a special case of Lemma \ref{lemma:main} in which every singular values are the same. We then tackle the general case utilizing a technical lemma \ref{lemma:linear-programming}, which is proved using linear programming technique.

\begin{lemma}
\label{lemma:norm-discrepancy}
    For every $F,G\in \mathcal{O}_{p,r}$, denote $u=F^{\top} e_1$ and $v=G^{\top} e_1$
    , then we have $\| F^{\top} G\|_F^2\le r-\big(\|u\|-\|v\|\big)^2$.
\end{lemma}

\begin{remark}
    It is straightforward to see that $\| F^{\top} G\|_F^2 \le r$. This lemma claims that if there is a gap between the 2-norm of the first row of $F$ and $G$, then there is a gap between $\| F^{\top} G\|_F^2$ and $r$.
\end{remark}

\begin{lemma}
\label{lemma:discrepancy-ab}
    Let $p>r\ge 2$, $\ U,V,L,R\in \mathcal{O}_{p,r}$, and assume they are all $\mu$-incoherent, denote $u=U^{\top} e_1$ and similarly for $v,l,r$. Also, denote $r_1$ as the first element of $r$. Suppose $u=ae_1$ and $l=be_1$, $0\le a\le b$ then we have 
    \[r-\langle U^{\top} L,V^{\top} R\rangle \ge \frac{p}{2\mu r} \cdot r_1^2(a-b)^2.\]
\end{lemma}

\begin{lemma}
\label{lemma:discrepancy-rv}
   Under the same assumptions of Lemma \ref{lemma:discrepancy-ab},
we have 
\[r-\langle U^{\top} L,V^{\top} R\rangle\ge  \frac{p}{32\mu r^2} a^2(r_1-v_1)^2.\]
\end{lemma}

Combining the previous two lemmas \ref{lemma:discrepancy-ab} and \ref{lemma:discrepancy-rv} gives the next lemma which characterizes how the discrepancy between the first rows of $U,V,L,R$ impact the inner product $\langle U^{\top} L,V^{\top} R\rangle$.

\begin{lemma}
\label{lemma:orthogonal-case}
    Let $p>r\ge 1$, $U,V,L,R\in \mathcal{O}_{p,r}$, and suppose they are all $\mu$-incoherent, denote $u=U^{\top} e_1$ and similar for $v,l,r$, then there exists a constant $c>0$ such that
    \[r-\langle U^{\top} L,V^{\top} R\rangle \ge c\frac{p}{\mu r^2} \cdot (u^{\top} v-l^{\top} r)^2.\]
\end{lemma}

Lemma \ref{lemma:orthogonal-case} solves the case where all singular values are identical for both $P$ and $Q$. Now we utilize the special case to prove the general case. We pause for a bit and state a technical lemma utilizing linear programming.

\begin{lemma}
\label{lemma:linear-programming}
    Consider the optimization problem, which is essentially an LP problem.
    \begin{align}
        \max_{X\in \Omega}&\quad \langle A,X\rangle 
    \end{align}

where \begin{equation}
\label{equ:definition-linear-coeff}
A= \begin{bmatrix}
\sigma_1 \\\sigma_2 \\ \cdots\\\sigma_r
\end{bmatrix}\cdot\begin{bmatrix}
\lambda_1 &\lambda_2 &\cdots&\lambda_r
\end{bmatrix}\end{equation}
 for $\sigma_1\ge \sigma_2\ge\cdots\ge\sigma_r\ge 0$ and $\lambda_1\ge \lambda_2\ge\cdots\ge\lambda_r\ge 0$. And $\ \Omega$ is specified by $3r$ constraints:
\begin{equation}
\label{equ:definition-linear-feasible-set}
    \Omega=\Big\{X\in \mathbb{R}^{r\times r}\Big|\sum_{i=1}^r |X_{i,k}|\le 1,\ \sum_{i=1}^r |X_{k,i}|\le 1,\ \sum_{i=1}^k\sum_{j=1}^k X_{i,j}\le k-\gamma_k \quad \forall 1\le k \le r\Big\}.
\end{equation}

Suppose $0 \le\gamma_k\le \frac{1}{3},\ \forall 1\le k\le r$, then the optimal value can be achieved at the $X_0\in\mathbb{R}^{r\times r}$ where $X_0(i,j)=0$ if $|i-j|\ge 2$; $X_0(1,1)=1-\gamma_1$, $X_0(i,i)=1-\gamma_i-\gamma_{i-1}$ for $i\ge 2$; $X_0(i,i+1)=X_0(i+1,i)=\gamma_i$ for $i\ge 1$.
\end{lemma}

\subsection{Proof of Lemma \ref{lemma:norm-discrepancy}}
\begin{proof}
    There is no loss of generality in assuming $\|u\|\ge \|v\|$. Let $R\in \mathcal{O}_{r}$ be the orthogonal transform such that $Ru=\|u\|e_1$.

    Denote $\mathcal{P}_{G}$ as the projection matrix that projects a vector into the subspace spanned by columns of $G$, then $\mathcal{P}_{G}(x)=GG^{\top} x$, $\|\mathcal{P}_{G}(x)\|=\|G^{\top} x\|$.

    Denote $\Tilde{F}=FR^{\top}$ and $\Tilde{f}_j=\Tilde{F} e_j$ as the $j^{th}$ column of $\Tilde{F}$, we have
    \begin{align*}
    \|F^{\top} G\|_F^2
    &=\|R F^{\top} G\|_F^2=\|G^{\top} \Tilde{F}\|_F^2\\
    &=\sum_{j=1}^r\|G^{\top} \Tilde{f}_j\|^2 = \sum_{j=1}^r\|\mathcal{P}_{G}(\tf_j)\|^2
\stepcounter{equation}\tag{\theequation}\label{equ:r-1+}
    \end{align*}
    where the first equality is due to the orthogonality of $R$ and the unitary invariance of $\|\cdot\|_F$. 
    
    Also as $|\mathcal{P}_{G}(\tf_j)\|\le \|\tf_j\|\le 1$, we have
    \[\|F^{\top} G\|_F^2\le r-1+\|\mathcal{P}_{G}(\tf_1)\|^2.\]

    Denote the projected vector $\tf_0=\mathcal{P}_{G}(\tf_1)=GG^{\top} \tf_1$, it suffices to bound $\|\tf_0\|^2$. Note that by the definition of $R$, the first row of $\Tilde{F}$ is $Ru=\|u\|e_1$, hence $e_1^{\top} \Tilde{f}_1=\Tilde{F}_{1,1}=\|u\|$.

    On the other hand, we have 
    \[e_1^{\top}\tf_0=e_1^{\top} GG^{\top} \tf_1=v^{\top} G^{\top} \tf_1\le \|v\|\cdot \|G^{\top} \tf_1\|\le \|v\|,\]
    from which we obtain \[\|\tf_1-\tf_0\|\ge |e_1^{\top}\tf_0-e_1^{\top}\tf_1|\ge |e_1^{\top}\tf_1|-|e_1^{\top}\tf_0|\ge \|u\|-\|v\|.\]

    Therefore, $\|\mathcal{P}_{G}(\tf_1)\|^2=\|\tf_0\|^2=\|\tf_1\|^2-\|\tf_1-\tf_0\|^2\le 1-\big(\|u\|-\|v\|\big)^2$, combined with \eqref{equ:r-1+} gives $\| F^{\top} G\|_F^2\le r-\big(\|u\|-\|v\|\big)^2$.
\end{proof}

\subsection{Proof of Lemma \ref{lemma:discrepancy-ab}}
\begin{proof}
    The crucial step underlying the proof is to characterize the geometric structure of the set of all possible $\{U^{\top} L e_1\}$, i.e. all possible outcomes of the first column of $U^{\top} L$, which is included in an ellipse $\mathcal{E}_r$ to be defined later. 

    Let us bring in more notations. Denote \[\begin{bmatrix}
       0\\x_i  \end{bmatrix}=(I-e_1 e_1^{\top})Ue_i,\quad \begin{bmatrix}        0\\y_i\end{bmatrix}=(I-e_1 e_1^{\top})Le_i,\]that is, \[U=\begin{bmatrix}a & 0 & \cdots & 0 \\ x_1 & x_2 & \cdots & x_r  \end{bmatrix},\quad L=\begin{bmatrix}b & 0 & \cdots & 0 \\ y_1 & y_2 & \cdots & y_r  \end{bmatrix}.\]

    Then we have \begin{equation}
    \label{equ:first-column}U^{\top} L e_1=\begin{bmatrix}
        ab+x_1^{\top} y_1\\x_2^{\top} y_1\\ \cdots\\x_r^{\top} y_1
    \end{bmatrix}.\end{equation}

    On the other hand, by orthogonality of columns of $U$, we have $x_i$ are pairwise orthogonal. As columns of $U$ and $V$ are normalized, we have $\|x_1\|^2=1-a^2$, $\|y_1\|^2=1-b^2$ while $\|x_i\|^2=\|y_i\|^2=1$ for $i\ge 2$, we have \[\begin{bmatrix}\frac{x_1}{\sqrt{1-a^2}},&x_2,&x_3,&\cdots,&x_r\end{bmatrix}\in \mathcal{O}_{p-1,r}.\]

    In other words, they form an incomplete basis of $\RR^{p-1}$, hence
    \begin{equation}
    \label{equ:ellipse}(\frac{x_1^{\top} y_1}{\sqrt{1-a^2}})^2+\sum_{i=2}^r (x_i^{\top} y_1)^2\le \|y_1\|^2=1-b^2.
    \end{equation}

    Now define the ellipse to be \[\mathcal{E}_r=\Bigg\{\begin{bmatrix}
        s_1\\s_2\\ \cdots\\s_r
    \end{bmatrix}\in \RR^{r}\Bigg| \frac{(s_1-ab)^2}{1-a^2}+\sum_{i=2}^r s_i^2\le 1-b^2\Bigg\}.\]

    Here we only indicate the dependence on $r$ and suppress the dependence on $a,b$. By \eqref{equ:first-column} and \eqref{equ:ellipse}, we have $U^{\top} L e_1\in \mathcal{E}_r$. We will at times use $\begin{bmatrix}
        s_1\\s_2\\ \cdots\\s_r
    \end{bmatrix}$ to denote $U^{\top} L e_1$ when there is no confusion.

    Now we turn to $V^{\top} R$; it turns out it suffices to bound $e_1^{\top} V^{\top} R e_1$ to prove this lemma. In fact, using Lemma \ref{lemma:norm-discrepancy} with $F=Ve_1,G=Re_1,r=1$, we get  
    \begin{equation}
    \label{equ:stripe}
    |e_1^{\top} V^{\top} R e_1|=\|F^{\top} G\|_F\le \sqrt{1-(v_1-r_1)^2}\le 1-\frac{(v_1-r_1)^2}{2}.
    \end{equation}

    Define the set (which is a ball intersecting a stripe) \[\mathcal{T}_r=\Bigg\{\begin{bmatrix}
        t_1\\t_2\\ \cdots\\t_r
    \end{bmatrix}\in \RR^{r}\Bigg| |t_1|\le 1-\frac{(v_1-r_1)^2}{2}, \sum_{i=1}^r t_i^2\le1\Bigg\},\]
    then by \eqref{equ:stripe}, we have $V^{\top} Re_1\in \mathcal{T}_r$.

    It turns out the structures of $\mathcal{E}_r$ and $\mathcal{T}_r$ are enough to prove the lemma. Specifically, we can prove for all $s\in \mathcal{E}_r$ and $t\in\mathcal{T}_r$, it holds that \begin{equation}  \label{equ:local}
    1-s^{\top} t\ge \frac{p}{8\mu r} \cdot r_1^2(a-b)^2.
    \end{equation}

    First let's prove the lemma using \eqref{equ:local}. In fact,
    \begin{align}
        r-\langle U^{\top} L,V^{\top} R\rangle&=\big(1-(U^{\top} Le_1)^{\top} V^{\top} Re_1\big)+\sum_{i=2}^r \big(1-(U^{\top} Le_i)^{\top} V^{\top} Re_i\big)\nonumber\\
        &\ge \big(1-(U^{\top} Le_1)^{\top} V^{\top} Re_1\big)+\sum_{i=2}^r \big(1-\|U^{\top} Le_i\| \cdot\|V^{\top} Re_i\|\big)\nonumber\\
        &\ge 1-(U^{\top} Le_1)^{\top} V^{\top} Re_1,\label{equ:local-technique}
    \end{align}
    where the second inequality follows from $\|U^{\top} Le_1\|\le \|Le_1\|\le 1$  .
    
    Further notice that $U^{\top} L e_1\in \mathcal{E}_r, V^{\top} Re_1\in \mathcal{T}_r$, by \eqref{equ:local}, we have that
    
        \[r-\langle U^{\top} L,V^{\top} R\rangle\ge \frac{p}{8\mu r} \cdot r_1^2(a-b)^2.\]

    To prove \eqref{equ:local}, we argue it suffices to consider the case $r=2$. In fact, let $\ts_1=s_1, \ts_2=\sqrt{\sum_{i=2}^r s_i^2}$ and $\Tilde{t}_1=t_1, \Tilde{t}_2=\sqrt{\sum_{i=2}^r t_i^2}$, then 
    \begin{equation}
    \label{equ:reduce-to-r=2}\begin{bmatrix}
        \ts_1\\\ts_2
    \end{bmatrix}\in \mathcal{E}_2, \begin{bmatrix}
        \Tilde{t}_1\\ \Tilde{t}_2
    \end{bmatrix}\in \mathcal{T}_2,\end{equation}
    while by Cauchy-Schwarz Inequality, $s^{\top} t\le \ts_1 \Tilde{t}_1+ \ts_2 \Tilde{t}_2=\ts^{\top} \Tilde{t}$. Therefore we reduce to the case $r=2$.

    Now we consider this simple optimization problem. 
    \begin{align}
        \max_{s,t\in \RR^2}&\quad s^{\top} t \nonumber\\
        \text{s.t.}&\quad  s\in \mathcal{E}_2, t\in \mathcal{T}_2.\nonumber 
    \end{align}

    As both $\mathcal{E}_2$ and $\mathcal{T}_2$ are convex, it's straightforward to see that the optimum is attained at the border. Denote the optimum as $(s^{\star},t^{\star})$ at this moment. We first show that the optimal $t^{\star}$ takes the form of $\begin{bmatrix}
        \cos(\phi)\\ \sin(\phi)
    \end{bmatrix}$. 
    
    In fact, it suffices to rule out the case that the optimal $t^{\star}$ lies on the vertical line part of the border of $\mathcal{T}_2$. Suppose without loss of generality that $|t^{\star}_1|=1-\frac{(v_1-r_1)^2}{2},\ t^{\star}_2\ge 0$, we immediate have the optimal $s^{\star}_2\ge 0$. Now define $\Tilde{t}:=(t^{\star}_1,\sqrt{1-t_1^{*2}})$. Let $\delta_t=\sqrt{1-t_1^{*2}}-t^{\star}_2>0$. We have $\Tilde{t}\in \mathcal{T}_2$ and $s^{*T} \Tilde{t}=s^{*T} t^{\star}+s^{\star}_2 \delta_t> s^{*T} t^{\star}$, which is in contradiction with the optimality of $(s^{\star},t^{\star})$. Therefore, $t^{\star}$ does take the form of $\begin{bmatrix}
        \cos(\phi)\\ \sin(\phi)
    \end{bmatrix}$. 

    As $s^{\star}$ lies on the ellipse, it takes the form of $\begin{bmatrix}
        ab+\sqrt{(1-a^2)(1-b^2)}\cos{\theta}\\ \sqrt{1-b^2} \sin{\theta}
    \end{bmatrix}$.
Therefore, we have 
    \begin{align*} s^{*T}t^{\star}&=\cos(\phi)\cdot\big(ab+\sqrt{(1-a^2)(1-b^2)}\cos{\theta}\big)+\sin(\phi)\cdot\big(\sqrt{1-b^2} \sin{\theta}\big)\\&\le \sqrt{\big(ab+\sqrt{(1-a^2)(1-b^2)}\cos{\theta}\big)^2+\big(\sqrt{1-b^2} \sin{\theta}\big)^2}\\&=\sqrt{a^2b^2+1-b^2+2ab\sqrt{1-a^2}\sqrt{1-b^2}\cos \theta-a^2(1-b^2)\cos^2\theta}
    \end{align*}

As $\frac{b\sqrt{1-a^2}}{a\sqrt{1-a^2}}=\sqrt{\frac{b^2-a^2b^2}{a^2-a^2b^2}}\ge 1$, the above display achieves maximum at $\cos\theta=1$ and hence \[s^{*T}t^{\star}\le ab+\sqrt{1-a^2}\sqrt{1-b^2}, \]
which leads to
\begin{align*}
1-s^{\top} t&\ge 1-s^{*T}t^{\star}\\&\ge 1-ab-\sqrt{1-a^2}\sqrt{1-b^2}\\&=\frac{(a-b)^2}{1-ab+\sqrt{1-a^2}\sqrt{1-b^2}}\\&\ge\frac{(a-b)^2}{2}\\&\ge\frac{p}{2\mu r} \cdot r_1^2(a-b)^2.\end{align*}
where the last inequality is due to $|r_1|\le \sqrt{\frac{\mu r}{p}}$ by the $\mu$-incoherence condition.

\end{proof}

\subsection{Proof of Lemma \ref{lemma:discrepancy-rv}}
\begin{proof}
    We start the proof by exploring more geometric structure of $V^{\top} Re_1=\begin{bmatrix}
        t_1\\t_2\\ \cdots\\t_r
    \end{bmatrix}$. 
    Denote $v_1-r_1=\lambda$, then $a^2(r_1-v_1)^2=a^2\lambda^2$.
    Consider expansion of $Re_1$ on the incomplete basis $\{Ve_1,Ve_2,\cdots,Ve_r\}$,
    \[Re_1=\sum_{i=1}^r t_i Ve_i+\alpha,\]
    where $V^{\top}\alpha=0$. Comparing the first coordinate, we get
    \begin{align}
    v_1-\lambda=r_1&=\sum_{i=1}^r t_i e_1^{\top} Ve_i+\alpha_1\nonumber\\
    &=t_1 v_1+\sum_{i=2}^r t_i e_1^{\top} Ve_i+\alpha_1. \nonumber
    \end{align}

    By the $\mu$-incoherence condition, we have 
    \[\sqrt{\sum_{i=1}^r \big( e_1^{\top} Ve_i\big)^2}\le\inco.\]

    Thus, by Cauchy-Schwartz inequality, \begin{align}
    \label{equ:compare-first-component}
    |\lambda|&=|(1-t_1)v_1-\sum_{i=2}^r t_i e_1^{\top} Ve_i-\alpha_1|\nonumber\\
    &\le |\alpha_1|+|1-t_1|\inco+\sqrt{\sum_{i=2}^r |t_i|^2}\cdot \inco,
    \end{align}, which will be crucial for the later proof. It is worth noting that the $\mu$-incoherence condition plays a nontrivial role in the proof exclusively through this equation.

    In addition, \begin{equation}
    \label{equ:norm-bound}
    1=\|Re_1\|^2=\sum_{i=1}^r t_i^2+\|\alpha\|^2\ge\sum_{i=1}^r t_i^2+\alpha_1^2. 
    \end{equation}

    Also, $\mu$-incoherence condition yields 
    \begin{equation}
    \label{equ:incoherence-discrepancy-u-v}|a|\le \inco,\quad |\lambda|=|v_1-r_1|\le |v_1|+|r_1|\le 2\inco
    \end{equation}, which will be extensively used later.

    We consider five cases in the sequel:
    \begin{itemize}
        \item {\bf Case 1}: $|\alpha_1|\ge\frac{1}{4}|\lambda|$: By the same technique as in \eqref{equ:local-technique}, we have 
        \begin{align}
        r-\langle U^{\top} L,V^{\top} R\rangle
        &=
        \big(1-(U^{\top} Le_1)^{\top} V^{\top} Re_1\big)+\sum_{i=2}^r \big(1-(U^{\top} Le_i)^{\top} V^{\top} Re_i\big)\nonumber\\
        &\ge \big(1-(U^{\top} Le_1)^{\top} V^{\top} Re_1\big)+\sum_{i=2}^r \big(1-\|U^{\top} Le_i\| \cdot\|V^{\top} Re_i\|\big)\nonumber\\
        &\ge 1-(U^{\top} Le_1)^{\top} V^{\top} Re_1\nonumber\\
        &\ge 1-\|V^{\top} Re_1\|\nonumber,
        \end{align}
        where the last inequality follows from $\|U^{\top} Le_1\|\le 1$ as $U,L$ are unitary. 

        By the definition of $t_1,t_2, \cdots, t_r$ as coordinates of $V^{\top} Re_1$,
        \begin{align}
        1-\|V^{\top} Re_1\|&= 1-\sqrt{\sum_{i=1}^r t_i^2}\nonumber\\
        &\ge 1-\sqrt{1-\alpha_1^2}\nonumber\\
        &\ge \frac{\alpha_1^2}{2}\ge \frac{\lambda^2}{32}\ge
        \frac{a^2\lambda^2}{32}\frac{p}{\mu r}\nonumber,\end{align}
        where the first inequality is due to the norm bound \eqref{equ:norm-bound} and the last inequality uses $|a|\le \inco$ in \eqref{equ:incoherence-discrepancy-u-v}.
        
        \item {\bf Case 2}: $\exists\  2\le k\le r$ such that $|s_k|\ge \frac{1}{4}\sqrt{\frac{p}{\mu r^2}}|\lambda|$. Recall that $\begin{bmatrix}
        s_1\\s_2\\ \cdots\\s_r
    \end{bmatrix}=U^{\top} L e_1$, i.e., $s_k=e_1^{\top} L^{\top} Ue_k$. Now consider $L^{\top} Ue_k=\begin{bmatrix}
        \widehat{s}_1\\\widehat{s}_2\\ \cdots\\\widehat{s}_r
        \end{bmatrix}$.
        Similarly as $\mathcal{E}_r$, we can define \begin{equation*}
        \widehat{\mathcal{E}}_r=\Bigg\{\begin{bmatrix}
        \widehat{s}_1\\\widehat{s}_2\\ \cdots\\\widehat{s}_r
    \end{bmatrix}\in \RR^{r}\Bigg| \frac{\widehat{s}_1^2}{1-b^2}+\sum_{i=2}^r \widehat{s}_i^2\le 1\Bigg\},
        \end{equation*} and prove that $L^{\top} Ue_k\in \widehat{\mathcal{E}}_r$.

        We have the following decomposition of $\|L^{\top} Ue_k\|^2$,
        \begin{align}
        \label{equ:case2-decomp}
        \|L^{\top} Ue_k\|^2=\sum_{i=1}^r \widehat{s}_i^2=-\frac{b^2 \widehat{s}_1^2}{1-b^2}+\frac{\widehat{s}_1^2}{1-b^2}+\sum_{i=2}^r \widehat{s}_i^2. \end{align}
        And because $L^{\top} Ue_k\in \widehat{\mathcal{E}}_r$, 
        it holds that $\frac{\widehat{s}_1^2}{1-b^2}+\sum_{i=2}^r \widehat{s}_i^2\le 1$.
        
        Plugging in \eqref{equ:case2-decomp} yields
        \begin{align}
        \|L^{\top} Ue_k\|^2&\le1-\frac{b^2 \widehat{s}_1^2}{1-b^2}\nonumber\\
        &\le 1-\frac{a^2 s_k^2}{1-a^2}\nonumber\\
        &\le 1-\frac{a^2|\lambda|^2}{16}\frac{p}{\mu r^2}\label{equ:norm-bound2},
        \end{align}
   where the second inequality follows from $\widehat{s}_1=e_1^{\top} L^{\top} Ue_k=s_k$ and $0\le a\le b$, the last inequality follows from the assumption of Case 2 that $|s_k|\ge \frac{1}{4}\sqrt{\frac{p}{\mu r^2}}|\lambda|$.

  By the same technique as \eqref{equ:local-technique}, we have 
        \begin{align}
        r-\langle U^{\top} L,V^{\top} R\rangle&=r-\langle L^{\top} U,R^{\top} V\rangle\nonumber\\&=\big(1-(L^{\top} Ue_k)^{\top} R^{\top} Ve_k\big)+\sum_{i\in [r]\setminus\{k\}} \big(1-(L^{\top} Ue_i)^{\top} R^{\top} Ve_i\big)\nonumber\\
        &\ge \big(1-(L^{\top} Ue_k)^{\top} R^{\top} Ve_k\big)+\sum_{i\in [r]\setminus\{k\}} \big(1-\|L^{\top} Ue_i\| \cdot\|R^{\top} Ve_i\|\big)\nonumber\\&\ge 1-(L^{\top} U e_k)^{\top} R^{\top} Ve_k.\nonumber
        \end{align}

        As $U,L$ are unitary, we have $\|R^{\top} V e_k\|\le 1$ and $1-(L^{\top} U e_k)^{\top} R^{\top} Ve_k\ge 1-\|L^{\top} U e_k\|$.

        Also because of the norm bound \eqref{equ:norm-bound2},
        
        \begin{align}
        1-\|L^{\top} U e_k\|\nonumber
        \ge 1-\sqrt{1-\frac{a^2|\lambda|^2}{16}\frac{p}{\mu r^2}}
        \ge \frac{a^2|\lambda|^2}{32}\frac{p}{\mu r^2}.\nonumber
        \end{align}

        Combining above we get $r-\langle U^{\top} L,V^{\top} R\rangle\ge \frac{a^2|\lambda|^2}{32}\frac{p}{\mu r^2}$.

        \item {\bf Case 3}: $\exists\  2\le k\le r$ such that $|s_k-t_k|\ge \frac{1}{4}\sqrt{\frac{p}{\mu r^2}}|\lambda|$.

        First we construct $w_s=\begin{bmatrix}
        s_k\\s_1\\ \cdots\\s_{k-1}\\s_{k+1}\\\cdots\\s_r\\\sqrt{1-\sum_{i=1}^r s_i^2}
    \end{bmatrix}$, $w_t=\begin{bmatrix}
        t_k\\t_1\\ \cdots\\t_{k-1}\\t_{k+1}\\\cdots\\t_r\\\sqrt{1-\sum_{i=1}^r t_i^2}
    \end{bmatrix}$ as two augmented unit vectors of $U^{\top} Le_1$ and $V^{\top} Re_1$.
    Then $\|w_s\|=\|w_t\|=1$. We get 
    \begin{align}
        (U^{\top} Le_1)^{\top} V^{\top} Re_1&\le w_s^{\top} w_t\nonumber\\&\le \sqrt{1-|s_k-t_k|^2}\nonumber\\&\le 1-\frac{|s_k-t_k|^2}{2},\label{equ:spurious-use-of lemma-5.2}
    \end{align}
where the second inequality is by applying Lemma \ref{lemma:norm-discrepancy} on $w_s$ and $w_t$.

By the same technique as \eqref{equ:local-technique}, we have 
        \begin{align}
        r-\langle U^{\top} L,V^{\top} R\rangle&=
        \big(1-(U^{\top} Le_1)^{\top} V^{\top} Re_1\big)+\sum_{i=2}^r \big(1-(U^{\top} Le_i)^{\top} V^{\top} Re_i\big)\nonumber\\
        &\ge \big(1-(U^{\top} Le_1)^{\top} V^{\top} Re_1\big)+\sum_{i=2}^r \big(1-\|U^{\top} Le_i\| \cdot\|V^{\top} Re_i\|\big)\nonumber\\
        &\ge 1-(U^{\top} Le_1)^{\top} V^{\top} Re_1.\nonumber
        \end{align}

        Because of \eqref{equ:spurious-use-of lemma-5.2}, we have
        \begin{align}
        1-(U^{\top} Le_1)^{\top} V^{\top} Re_1&\ge 1-\big(1-\frac{|s_k-t_k|^2}{2}\big)\nonumber\\
        &\ge\frac{|\lambda|^2}{32}\frac{p}{\mu r^2}\nonumber\ge\frac{a^2|\lambda|^2}{32}\frac{p}{\mu r^2},\nonumber
        \end{align}
        where the second inequality utilizes the assumption of Case 3 that $|s_k-t_k|\ge \frac{1}{4}\sqrt{\frac{p}{\mu r^2}}|\lambda|$ and the last inequality uses $a^2\le 1$. 

        \item {\bf Case 4}: $t_1\le 0$. Similar to the proof of Lemma \ref{lemma:discrepancy-ab}, we again exploit the geometric property of $s=U^{\top}Le_1$ and $t=V^{\top}Re_1$.

        We want to upper bound the value of $s^{\top} t$ subject to $s\in \mathcal{E}_r$ and $t\in \mathcal{T}_r, t_1\le 0$. Similar to \eqref{equ:reduce-to-r=2}, it suffices to consider $r=2$.

        Recall the definition of $\mathcal{E}_2$ which is an ellipse centered at $(ab,0)$. 
    \begin{align}
            s^{\top} t&=s_1 t_1+s_2 t_2\nonumber\\&\le (s_1-ab)t_1+s_2 t_2\nonumber\\&\le \sqrt{(s_1-ab)^2+s_2^2}\nonumber,
        \end{align}
        where the first inequality follows from the assumption of the case $t_1\le 0$ and the second inequality follows from $t_1^2+t_2^2\le 1$ for $t\in \mathcal{T}_2$.

        Observe that $\sqrt{(s_1-ab)^2+s_2^2}$ is the distance from a point on the ellipse to the center, thus should be no larger than the semi-major axis of the ellipse, which is $\sqrt{1-b^2}\le 1-\frac{b^2}{2}$.

        In conclusion,
        \begin{equation}
            \label{equ:discrepancy-t1<0}
        s^{\top} t\le 1-\frac{b^2}{2}\quad  \forall s\in \mathcal{E}_r, t\in \mathcal{T}_r, t_1\le 0.
        \end{equation}

        By the same technique as \eqref{equ:local-technique}, we have \begin{align}
        r-\langle U^{\top} L,V^{\top} R\rangle&=
        \big(1-(U^{\top} Le_1)^{\top} V^{\top} Re_1\big)+\sum_{i=2}^r \big(1-(U^{\top} Le_i)^{\top} V^{\top} Re_i\big)\nonumber\\
        &\ge \big(1-(U^{\top} Le_1)^{\top} V^{\top} Re_1\big).\nonumber
        \end{align}

        Now because $U^{\top} Le_1\in \mathcal{E}_r, V^{\top} Re_1\in \mathcal{T}_r$ and $t_1\le 0$, by 
        \eqref{equ:discrepancy-t1<0},
        \begin{align}\big(1-(U^{\top} Le_1)^{\top} V^{\top} Re_1\big)
        &\ge 1-(1-\frac{b^2}{2})=\frac{b^2}{2}\nonumber\\&\ge  \frac{a^2}{2}\ge \frac{a^2|\lambda|^2}{8}\frac{p}{\mu r}\nonumber
        \end{align}
        where the second inequality follows from $|a|\le |b|$, and the last inequality follows from \eqref{equ:incoherence-discrepancy-u-v}.

        \item {\bf Case 5}: Suppose all assumptions of Case 1-4 don't hold, we have \[\forall \ 2\le k\le r,\ |s_k-t_k|<\frac{1}{4}\sqrt{\frac{p}{\mu r^2}}|\lambda|,\  |s_k|<\frac{1}{4}\sqrt{\frac{p}{\mu r^2}}|\lambda|.\]

        Therefore, by the triangle inequality, it holds that \begin{equation}
        \label{equ:bounds-of-tk-k>=2}\forall \ 2\le k\le r, |t_k|<\frac{1}{2}\sqrt{\frac{p}{\mu r^2}}|\lambda|.\end{equation} Also, as the assumption of Case 1 doesn't hold, we have $|\alpha_1|<\frac{1}{4}|\lambda|$. 
        Revisiting the inequality \eqref{equ:compare-first-component}, we can deduce that \[|1-t_1|\ge \frac{1}{4}\invinco|\lambda|.\]

        As the assumption of Case 4 doesn't hold, we have $t_1\ge 0$. Putting together, we have  \begin{equation}
        \label{equ:bounds-of-t1}0\le t_1\le 1-\frac{1}{4}\invinco|\lambda|.\end{equation}

        Now we bound the norm of $V^{\top} Re_1=\begin{bmatrix}
        t_1\\t_2\\ \cdots\\t_r
    \end{bmatrix}$ from below as 
    \begin{align*}
        \|V^{\top} Re_1\|^2&= \sum_{i=1}^r t_i^2=t_1^2+\sum_{i=2}^{r} t_i^2\\
        &\le \big(1-\frac{1}{4}\invinco|\lambda|\big)^2+(r-1)\cdot \big(\frac{1}{2}\sqrt{\frac{p}{\mu r^2}}|\lambda|\big)^2
    \end{align*}
    where the inequality is due to both \eqref{equ:bounds-of-tk-k>=2} and \eqref{equ:bounds-of-t1}.

Let $g:=\sqrt{\frac{p}{\mu r}}|\lambda|$, by considering separately the cases of $r=1$ and $r\ge 2$, it is straightforward to show
    \begin{align*}
    (r-1)\cdot \big(\frac{1}{2r}\sqrt{\frac{p}{\mu r}}|\lambda|\big)^2
    \le \frac{g^2}{8}.
    \end{align*}

Putting together, we get
        \begin{align}
        \label{equ:case-5}
        \|V^{\top} Re_1\|^2&\le (1-\frac{g}{4})^2+\frac{g^2}{8}\nonumber\\&\le1-\frac{g}{2}+\frac{3g^2}{16}\le 1-\frac{g}{8}.
        \end{align}
    
The last inequality is because that by \eqref{equ:incoherence-discrepancy-u-v}, $g=\sqrt{\frac{p}{\mu r}}|\lambda|\le \sqrt{\frac{p}{\mu r}}\cdot 2\sqrt{\frac{\mu r}{p}}=2$.

Hence by the same technique as \eqref{equ:local-technique}, and by \eqref{equ:case-5}, we have 
        \begin{align*}
        r-\langle U^{\top} L,V^{\top} R\rangle&\ge 1-(U^{\top} Le_1)^{\top} V^{\top} Re_1\\
        &\ge 1-\|V^{\top} Re_1\|\\
        &\ge 1-\sqrt{1-\frac{g}{8}}.
        \end{align*}
        
        Further, by the definition of $g$ and $|\lambda|\le 2\inco$,
        \begin{align*}
        1-\sqrt{1-\frac{g}{8}}&\ge \frac{g}{16}=\frac{1}{16}\sqrt{\frac{p}{\mu r}}|\lambda|\ge \frac{a^2|\lambda|^2}{32}\frac{p}{\mu r^2}.
        \end{align*}        Therefore, it holds that $r-\langle U^{\top} L,V^{\top} R\rangle\ge \frac{a^2|\lambda|^2}{32}\frac{p}{\mu r^2}$.
    
    \end{itemize}

    Combining the above four cases, we prove that with $c=\frac{1}{32}$,
    \[r-\langle U^{\top} L,V^{\top} R\rangle\ge ca^2|\lambda|^2\frac{p}{\mu r^2}=c \frac{p}{\mu r^2} a^2(r_1-v_1)^2.\]

\end{proof}

\subsection{Proof of Lemma \ref{lemma:orthogonal-case}}

\begin{proof}
    We prove this lemma in two steps. We first prove the case of $r\ge 2$.
    
    We argue that there is no loss of generality in assuming $u=ae_1$ and $l=be_1$. To see that, notice that the lemma is invariant up to joint rotation.
    Specifically, let \[\tU=US,\tV=VS^{\top},\tL=LT,\tR=RT^{\top}.\]

    Accordingly, \[\tu=S^{\top} u,\tv=Sv,\tl=T^{\top} l,\tr=Tr\]
where $S,T\in \mathcal{O}_r$ are arbitrary orthogonal matrices. It follows that \[\langle U^{\top} L,V^{\top} R\rangle=\langle \tU^{\top} \tL,\tV^{\top} \tR\rangle, u^{\top} v=\tu^{\top} \tv, l^{\top} r=\tl^{\top} \tr.\]

    In addition, the incoherence parameter $\mu$ stays unchanged after the procedure. Thus, we can rotate these matrices such that the first row of $U$ and $L$ align with $e_1$ while preserving the lemma statement, i.e., $u=ae_1$ and $l=be_1$. Furthermore, we can assume $b\ge a\ge 0$, as $\langle U^{\top} L,V^{\top} R\rangle=\langle L^{\top} U,R^{\top} V\rangle$ and we can swap the order of $U$ and $L$, $V$ and $R$.

    In the subsequent proof, we aim to get 
    \[r-\langle U^{\top} L,V^{\top} R\rangle \ge c\frac{p}{\mu r^2} \cdot (av_1-br_1)^2.\]
    To achieve that we split $(br_1-av_1)^2$ into two terms as $(br_1-av_1)^2\le 2a^2(r_1-v_1)^2+2r_1^2(a-b)^2$ regarding the discrepancy of $a,b$ and $r_1,v_1$ respectively. 
    
    Now by Lemma \ref{lemma:discrepancy-ab} we have 
    \[r-\langle U^{\top} L,V^{\top} R\rangle \ge \frac{p}{2\mu r} \cdot r_1^2(a-b)^2\ge  \frac{p}{32\mu r^2}r_1^2(a-b)^2.\] 
    By Lemma \ref{lemma:discrepancy-rv} we have 
    \[r-\langle U^{\top} L,V^{\top} R\rangle\ge  \frac{p}{32\mu r^2} a^2(r_1-v_1)^2.\] Combining gives the result with $c=\frac{1}{128}$.
    \begin{align*}
        r-\langle U^{\top} L,V^{\top} R\rangle &\ge \frac{1}{2}\cdot\frac{p}{32\mu r^2}\big(r_1^2(a-b)^2+a^2(r_1-v_1)^2\big)\\
        &\ge \frac{p}{128\mu r^2}(av_1-br_1)^2.
        \end{align*}
    We are left with the simple case $r=1$.

    In fact, using Lemma \ref{lemma:norm-discrepancy} with $r=1$. We get 
    \[U^{\top} L\le 1-\frac{(u_1-l_1)^2}{2}\] and
    \[V^{\top} R\le 1-\frac{(v_1-r_1)^2}{2}.\]

    Therefore, as $r=1$,
    \begin{align*}
    r-\langle U^{\top} L,V^{\top} R\rangle&=1-(U^{\top} L)\cdot (V^{\top} R)\\
    &\ge 1-|U^{\top} L|\\
    &\ge 1-\sqrt{1-\frac{(u_1-l_1)^2}{2}}\ge \frac{(u_1-l_1)^2}{4}.\end{align*}

    Again by $\mu$-incoherence condition, it holds that
    $|v_1|\le \sqrt{\frac{\mu r}{p}}$ and 
        \[r-\langle U^{\top} L,V^{\top} R\rangle \ge \frac{(u_1-l_1)^2}{4}\ge \frac{p}{\mu r}\frac{v_1^2(u_1-l_1)^2}{4}.\]

    Similarly, \[r-\langle U^{\top} L,V^{\top} R\rangle\ge \frac{p}{\mu r}\frac{l_1^2(v_1-r_1)^2}{4}.\]

    Combining the above two inequalities yields 
    \begin{align*}
    r-\langle U^{\top} L,V^{\top} R\rangle&=1-(U^{\top} L)\cdot (V^{\top} R)\ge \frac{p}{\mu r}\frac{l_1^2(v_1-r_1)^2+l_1^2(v_1-r_1)^2}{8}
    \\
    &\ge \frac{p}{\mu r}\frac{(u_1 v_1-l_1 r_1)^2}{16}\\
    &\ge \frac{p}{\mu r^2}\frac{(u_1 v_1-l_1 r_1)^2}{16}.
    \end{align*}
\end{proof}

\subsection{Proof of Lemma \ref{lemma:main}}
\label{sec:proof-key-lemma}
We postpone the technical Lemma \ref{lemma:linear-programming} to the next subsection and state the proof of the key separation lemma.
\begin{proof}[Proof of lemma \ref{lemma:main}]
    Without loss of generality, we only need to compare $\|\Delta\|_F^2$ to the $\Delta_{1,1}^2$(we can permute the rows and columns). 
    
    Denote $u_i=U_{1,i},v_i=V_{1,i},l_i=L_{1,i},r_i=R_{1,i}$. It is straightforward to simplify

    \begin{align*}
    \|\Delta\|_F^2&=\|P-Q\|_F^2\\
    &=\|P\|_F^2+\|Q\|_F^2-2 \Tr(P^{\top} Q)\\
    &=\|\Sigma\|_F^2+\|\Lambda\|_F^2-2\Tr(V\Sigma U^{\top} L\Lambda R^{\top})\\
    &=\sum_{i=1}^{r}\sigma_i^2+\sum_{i=1}^{r}\lambda_i^2-2\Tr(R^{\top} V\Sigma U^{\top} L\Lambda).
    \end{align*}

    Let $X=U^{\top} L\odot V^{\top} R$ where $\odot$ is the Hadamard Product. Then we can simplify the trace term as \[\Tr(R^{\top} V\Sigma U^{\top} L\Lambda)=\langle A,X\rangle \]
    where $A$ is defined as \eqref{equ:definition-linear-coeff}.

    Note that Lemma \ref{lemma:orthogonal-case} holds if we replace $r$ by $1\le k\le r$. In specific, 
    \begin{align*}
    k-\langle U_{[1:k]}^\top L_{[1:k]},V_{[1:k]}^\top R_{[1:k]}\rangle\ge c\frac{p}{\mu k^2}\big[\sum_{i=1}^k (u_i v_i-l_i r_i)\big]^2\ge c\frac{p}{\mu r^2}\big[\sum_{i=1}^k (u_i v_i-l_i r_i)\big]^2.
    \end{align*}
    
    Hence it's not hard to draw the conclusion that $X\in \Omega$ where $\Omega$ is defined in \eqref{equ:definition-linear-feasible-set} with \begin{equation}
    \label{equ:definition-gamma}
    \gamma_k=c\frac{p}{\mu r^2}\big[\sum_{i=1}^k (u_i v_i-l_i r_i)\big]^2.
    \end{equation}

    Note that as long as $c<\frac{1}{6}$,
    \[\gamma_k\le c\frac{p}{\mu r^2}2r \frac{\mu r}{p}=2c< \frac{1}{3}.\]

    We can apply Lemma \ref{lemma:linear-programming} and so 
    \begin{align}
    \|\Delta\|_F^2
    &=\sum_{i=1}^{r}\sigma_i^2+\sum_{i=1}^{r}\lambda_i^2-2\Tr(R^{\top} V\Sigma U^{\top} L\Lambda)\nonumber\\
    &\ge \sum_{i=1}^{r}\sigma_i^2+\sum_{i=1}^{r}\lambda_i^2-2\langle A,X_0\rangle\nonumber\\
    &=\sum_{i=1}^r(1-\gamma_i-\gamma_{i-1})(\sigma_i-\lambda_i)^2+\sum_{i=1}^{r-1}\gamma_i\{(\sigma_i-\lambda_{i+1})^2+(\lambda_i-\sigma_{i+1})^2\}+\gamma_r(\sigma_r^2+\lambda_r^2)\nonumber \\
    &\ge \sum_{i=1}^r\frac{1}{3}(\sigma_i-\lambda_i)^2+\sum_{i=1}^{r-1}\gamma_i\{(\sigma_i-\lambda_{i+1})^2+(\lambda_i-\sigma_{i+1})^2\}+\gamma_r(\sigma_r^2+\lambda_r^2),
    \label{equ:frobenius}
    \end{align}
    where $A$ and $X_0$ are defined in Lemma \ref{lemma:linear-programming}, and the second equality follows from the direct calculation, and we let $\gamma_0$ for convenience.

    On the other hand, let $\phi_i=\left\{\begin{aligned}&\sigma_i, \quad\text{if}\;  2\ |\ i\\
    &\lambda_i,\quad \text{if} \ 2\nmid i\end{aligned}
    \right.$, we have by Cauchy-Schwarz inequality,
    
   \begin{align}
   \Delta_{1,1}^2=(P_{1,1}-Q_{1,1})^2&=\Big[\sum_{i=1}^r \big(\sigma_i u_i v_i-\lambda_i l_i r_i\big)\Big]^2\nonumber\\
   &\le 2\Big[\sum_{i=1}^r \big(\phi_i u_i v_i-\phi_i l_i r_i\big)\Big]^2+2r\sum_{i=1}^r (\sigma_i-\lambda_i)^2 (\frac{\mu r}{p})^2.
   \label{equ:first-element}
   \end{align}
The first term in the above display can be rewritten as 
    \begin{align*}
    \sum_{i=1}^r \big(\phi_i u_i v_i-\phi_i l_i r_i\big)&=\sum_{k=1}^r\Big\{
    (\phi_k-\phi_{k+1})\sum_{i=1}^k(u_iv_i-l_ir_i)\Big\}
    \end{align*}
where we define $\phi_{r+1}=0$ for convenience.

    \begin{align}
    \Big[\sum_{i=1}^r \big(\phi_i u_i v_i-\phi_i l_i r_i\big)\Big]^2&\le r\sum_{k=1}^{r}(\phi_k-\phi_{k+1})^2 \Big[\sum_{i=1}^k(u_iv_i-l_ir_i)\Big]^2\nonumber\\
    &=r\frac{\mu r^2}{cp}\sum_{k=1}^{r}(\phi_k-\phi_{k+1})^2 \gamma_k\nonumber\\
    &\le \frac{\mu r^3}{cp}\sum_{i=1}^{r-1}\{(\sigma_i-\lambda_{i+1})^2+(\lambda_i-\sigma_{i+1})^2\}^2 \gamma_i+\frac{\mu r^3}{cp}(\sigma_r^2+\lambda_r^2)\gamma_r\label{equ:expand},
    \end{align}
where the first inequality uses Cauchy-Schwarz inequality, the first equality uses the definition \eqref{equ:definition-gamma}, and the last inequality follows from the definition of $\phi$.

Therefore, choosing $c$ sufficiently small and combining \eqref{equ:frobenius},\eqref{equ:first-element},\eqref{equ:expand} gives
\begin{align*}
\Delta_{1,1}^2&\le \frac{2\mu r^3}{cp}\sum_{i=1}^{r-1}\{(\sigma_i-\lambda_{i+1})^2+(\lambda_i-\sigma_{i+1})^2\}^2 \gamma_i+\frac{2\mu r^3}{cp}(\sigma_r^2+\lambda_r^2)\gamma_r+2r\sum_{i=1}^r (\sigma_i-\lambda_i)^2 (\frac{\mu r}{p})^2
\\&\le \frac{2\mu r^3}{cp}\left(\sum_{i=1}^r\frac{1}{3}(\sigma_i-\lambda_i)^2+\sum_{i=1}^{r-1}\gamma_i\{(\sigma_i-\lambda_{i+1})^2+(\lambda_i-\sigma_{i+1})^2\}+\gamma_r(\sigma_r^2+\lambda_r^2)\right)
\\&\le\frac{2\mu r^3}{cp}\|\Delta\|_F^2,
\end{align*}
where the second inequality utilizes implicitly $p\ge 3\mu c$. In fact, for $c\le \frac{1}{2}$, if $p< 3\mu c$, it holds that $\frac{2\mu r^3}{cp}\ge 1$, rendering \eqref{equ:key-lemma} vacuous. Set $\Tilde{c}=\frac{2}{c}$ finishes the proof.

\end{proof}

\subsection{Proof of Lemma \ref{lemma:linear-programming}}
\begin{proof}[Proof of lemma \ref{lemma:linear-programming}]
    We first observe two simple facts. 
    \begin{enumerate}
    \item
    It is straightforward to observe that $\Omega$ is a polyhedron, and the problem is LP. 

    \item
    Because the $\epsilon$-ball $\{X\big|\|X\|_F\le \epsilon\}\subset \Omega$ if we choose $\epsilon$ small enough,
    we can deduce that $\Omega$ is non-degenerate.
    \end{enumerate}
    
    In the sequel, we prove that $X_0$ is an extreme point of the polyhedron $\Omega$. By LP theory, this can be achieved by finding $r^2$ hyperplanes defined by constraints whose joint intersection is exactly $X_0$. 

    In fact, define 
    \[A_i=\Big\{\sum_{k=1}^r X_{i,k}=1\Big\}\cup \Big\{\sum_{k=1}^{j-1}X_{i,k}-X_{i,j}+\sum_{k=j+1}^{r}X_{i,k}=1,\quad j=i+2,i+3,\cdots,r\Big\} \]

    \[B_i=\Big\{\sum_{k=1}^r X_{k,i}=1\Big\}\cup \Big\{\sum_{k=1}^{j-1}X_{k,i}-X_{j,i}+\sum_{k=j+1}^{r}X_{k,i}=1,\quad j=i+2,i+3,\cdots,r\Big\} \]

    and \[C_i=\Big\{\sum_{k=1}^i\sum_{j=1}^i X_{k,j}=i-\gamma_i \Big\}.\]

    Then let $D=\Big(\bigcup\limits_{i=1}^{r-1} A_i \Big)\bigcup\Big(\bigcup\limits_{i=1}^{r-1}B_i\Big)\bigcup\Big( \bigcup\limits_{i=1}^r C_i\Big)$, we can verify that $|D|=r^2$ and the intersection of them gives $X_0$.

    Finally, we prove that moving in every feasible direction from $X_0$ will not increase the value of the objective function, hence finishing the proof.

    To obtain that, let $\Delta^0$ be an arbitrary feasible direction, and we want to prove $\langle A,\Delta^0\rangle\le 0$ where $A$ is defined in \eqref{equ:definition-linear-coeff}. In later proof, we construct an operation $\calO$ and a property $\calP$ such that the initial $\Delta^0$ satisfies the property, the operation will not decrease the objective function, i.e., $\langle A,\calO(\Delta^0)\rangle\ge\langle A,\Delta^0\rangle$ while preserving the prescribed property $\calP$, and the operations will end up being a zero matrix in finite steps. Finally, we have $\calO^{I}(\Delta^0)$ will end up being a zero matrix for some integer $I$, where $\calO^I$ is the order-$I$ composition of $\calO$. And hence $\langle A,\Delta^0\rangle\le\langle A,\calO^{I}(\Delta^0)\rangle =0$.

    The property $\mathcal{P}$ defined for a $X\in\mathbb{R}^{r\times r}$ is sum of every upper left block of $X$ is greater or equal to 0, that is, for every $1\le k_1,k_2\le r$, it requires\[\sum_{i=1}^{k_1}\sum_{j=1}^{k_2} X_{i,j}\le 0.\]

    The operation $\mathcal{O}$ acting on $X$ is that we choose the nonzero element $X_{l_1,l_2}$ which has the lowest row index plus column index ($l_1+l_2$). If there is a tie, break the tie arbitrarily. The property $\mathcal{P}$ guarantees that $X_{l_1,l_2}<0$.
    \begin{enumerate}
    \item
    if $l_1<r, l_2< r$, then $\mathcal{O}(X)$ makes the following changes on $X$:
    \[
    X_{l_1,l_2}\rightarrow 0,\ X_{l_1,l_2+1}\rightarrow X_{l_1,l_2+1}+X_{l_1,l_2},\]\[ 
    X_{l_1+1,l_2}\rightarrow X_{l_1+1,l_2}+X_{l_1,l_2},\  
    X_{l_1+1,j+1}\rightarrow X_{l_1+1,j+1}-X_{l_1,l_2}.\]

    \item if $l_1=r$, $l_2<r$, then
    \[
    X_{l_1,l_2}\rightarrow 0,\ X_{l_1,l_2+1}\rightarrow X_{l_1,l_2+1}+X_{l_1,l_2}.\]
    \item if $l_1<r$, $l_2=r$, then
    \[
    X_{l_1,l_2}\rightarrow 0,\ X_{l_1+1,l_2}\rightarrow X_{l_1+1,l_2}+X_{l_1,l_2}.\]

    \item if $l_1=l_2=r$, then
    \[
    X_{l_1,l_2}\rightarrow 0.\]
    \end{enumerate}

Denote the smallest sum of row index and column index for a nonzero element of $X$ as $\calI(X)$.
It can be seen that $\calI(X)<\calI(\calO(X))$when $X\neq 0$, which implies that the iterative operations on $\Delta^0$ will end up being a zero matrix in finite steps.
To prove that initial $\Delta^0$ satisfies $\mathcal{P}$, we list three possibilities on $k_1,k_2$.

\begin{enumerate}
    \item if $k_1=k_2=k$, then by the block constraint  $\sum_{i=1}^k\sum_{j=1}^k X_{i,j}\le k-\gamma_k $ and that this constraint is tight on $X_0$, we have
    \[\sum_{i=1}^k\sum_{j=1}^k \Delta^0_{i,j}\le 0.\]

    \item if $k_1>k_2$, then consider the first $k_2$ column constraints,  by feasibility of the direction $\Delta^0$ and the tightness of these constraints on $X_0$, we have for every $1\le i\le k_2$, \[\sum_{i=1}^{k_1}\Delta^0_{i,j}\le 0,\]

    Summing up over $j$ yields\[\sum_{i=1}^{k_1}\sum_{j=1}^{k_2} \Delta^0_{i,j}\le 0.\]
    
    \item if $k_1<k_2$, the proof is similar to the second case by swapping the column and row.

    Recall the definition of $A$ in \eqref{equ:definition-linear-coeff}, the monotonicity of the operator $\mathcal{O}$ can be drawn by \[\langle A,\mathcal{O}(\Delta^0)-\Delta^0\rangle=|\Delta^0_{l_1,l_2}|(\sigma_{l_1}-\sigma_{l_1+1})(\lambda_{l_2}-\lambda_{l_2+1}) \ge 0\]
    where we define $\lambda_{r+1}:=\sigma_{r+1}:=0$ for convenience.

    Last but not least, $\mathcal{O}$ preserves $\mathcal{P}$. The only nontrivial case is when the block is specified by $(k_1,k_2)=(l_1,l_2)$. In that case, notice that by the choice of $(l_1,l_2)$, we have $\Delta^0_{i,j}=0$ when $i+j<l_1+l_2=k_1+k_2$, hence $\calO(\Delta^0)_{i,j}=0$ when $i+j<l_1+l_2=k_1+k_2$, and 
    \[\sum_{i=1}^{k_1}\sum_{j=1}^{k_2} \calO(\Delta^0)_{i,j}= 0.\]
\end{enumerate}
    
\end{proof}

\subsection{Optimality of Lemma \ref{lemma:main}}
\label{sec:optimality-key-lemma}
In this subsection, we demonstrate that Lemma \ref{lemma:main} is optimal in the sense that the right-hand side can not be improved in the order of dimension $p$, the incoherence parameter $\mu$. Moreover, it is expected to grow at least linearly with the rank 
$r$, though the precise relationship remains undetermined. To begin, we present a simplified version focusing solely on the dimension.

\begin{proposition}
\label{prop:optimality-key-lemma}
    For every $\epsilon>0$, there exists two rank-1 and $(1+\epsilon)$-incoherent matrices $P$ and $Q$, such that for $\Delta=P-Q$,
    \[\frac{\|\Delta\|_{\max}^2}{\|\Delta\|_F^2}\ge \frac{1}{2p}.\]
\end{proposition}

\begin{proof}
    Denote $u=\frac{1}{\sqrt{p}}\mathbf{1}_p$. Let $v=u+\frac{\epsilon}{\sqrt{p}} (e_{p-1}-e_p)$. Then $u,v\in\mathcal{O}_{p,1}$ and are $1+\epsilon$-incoherent.

    Let $P=uu^{\top}$, $Q=uv^{\top}$, then $P,Q$ are rank-1 and $(1+\epsilon)$-incoherent. However, $\Delta=\frac{\epsilon}{\sqrt{p}} u\cdot (e_{p-1}-e_p)$ has $2p$ nonzero elements, i.e.,$\|\Delta\|_0=2p$. Hence,
    \[\frac{\|\Delta\|_{\max}^2}{\|\Delta\|_F^2}\ge \frac{1}{\|\Delta\|_0}=\frac{1}{2p}.\]
\end{proof}

We can prove a more refined version below.
\begin{proposition}
\label{prop:optimality2-key-lemma}
For any $\mu>1,r\ge 1$ such that $\mu r< p$, there exists two rank-$r$ and $\mu$-incoherent matrices $P$ and $Q$ such that for $\Delta=P-Q$, it holds that
\[\frac{\|\Delta\|_{\max}^2}{\|\Delta\|_F^2}\ge \frac{\mu r}{2p}.\]
\end{proposition}

\begin{remark}
The lower bound in Proposition \ref{prop:optimality2-key-lemma} scales with $r$ instead of $r^3$ in Lemma \ref{lemma:main}, the optimal scaling with respect to $r$ is left to future work.
\end{remark}

\begin{proof}
Let $\tilde U\in \calO_{p,r}$ be a $(1+\mu)/2$-incoherent orthogonal matrix such that $\|\tilde U^\top e_1\|\in [\frac{1}{2}\inco,\frac{1}{2}]$. Denote $a=\|\tilde U^\top e_1\|$. Let $Q_1\in\calO_{r}$ such that $Q_1\tilde U=ae_1$. Denote $U=Q_1\tilde U$, then we also have $U\in \calO_{p,r}$. 

Denote $Ue_1=\begin{bmatrix}
    a\\\tilde u
\end{bmatrix}$, then $\|\tilde u\|=\sqrt{1-a^2}$. Now we start to define $V$. Pick $b$ such that $a<b<\min\{\inco,\frac{1}{2}\}$. We then set $Ve_1=\begin{bmatrix}
    b\\\tilde v
\end{bmatrix}$, where $\tilde v=\sqrt{\frac{1-b^2}{1-a^2}}\tilde u$, and $Ve_i=Ue_i$ for $2\le i\le r$. It is easy to verify that $U,V\in \calO_{p,r}$ and are both $\mu$-incoherent.

Let $P=UU^\top $ and $Q=VV^\top $. We have that 
\begin{align*}
\|U^\top  V\|_F^2&=(V^\top Ue_1)^2+\sum_{i=2}^r (V^\top Ue_i)^2\\&=r-1+\Big(ab+\sqrt{(1-a^2)(1-b^2)}\Big)^2
\end{align*}

We have that \begin{align*}
\|\Delta\|_F^2&=r-\langle U^\top  V,U^\top  V\rangle\\&=r-\|U^\top  V\|_F^2\\&=1-\Big(ab+\sqrt{(1-a^2)(1-b^2)}\Big)^2\\&=\Big(a\sqrt{1-b^2}-b\sqrt{1-a^2}\Big)^2\\&=\frac{(a-b)^2(a+b)^2}{\Big(a\sqrt{1-b^2}+b\sqrt{1-a^2}\Big)^2}\\&\le 2(a-b)^2,
\end{align*}
where we used $a,b\le \frac{1}{2}$ in the inequality.

On the other hand,
we have
\begin{align*}
\|\Delta\|_{\max}\ge |e_1^\top \Delta e_1|=|a^2-b^2|.
\end{align*}

Hence 
\begin{align*}
\frac{\|\Delta\|_{\max}^2}{\|\Delta\|_F^2}&\ge \frac{(a^2-b^2)^2}{2(a-b)^2}\\&=\frac{(a+b)^2}{2}\ge \frac{\mu r}{2p},
\end{align*}
where we use that $b>a>\frac{1}{2}\inco$. 

\end{proof}

\section{Proofs in Section \ref{sec:further-example}}
\label{sec:proof-further-example}

\subsection{Proof of Corollary \ref{cor:multitask-regression}}
\begin{proof}
    The multitask regression observation model satisfies $\sigma_{\min}^2$-RSC condition because $\|X\Delta\|_F^2=\operatorname{Tr}(\Delta^{\top} X^{\top} X\Delta)\ge \operatorname{Tr}(\Delta^{\top} \sigma_{\min}(X^{\top} X)\Delta)\ge\sigma_{\min}^2\|\Delta\|_F^2$. So we can take $\kappa=\sigma_{\min}^2$.
    In addition, we have 
    \[\|\mathfrak{X}(W)\|=\|X^{\top} W\|\le \|X\|\cdot \|W\|=\sigma_{\max}\|W\|\] and
    \[\|\mathfrak{X}(W)\|_{\max}=\|X^{\top} W\|_{\max}\le \|X^{\top}\|_{2,\infty}\cdot \|W\|_{2,\infty}=B\|W\|_{2,\infty}.\]
    A direct application of \ref{thm:main} gives
    \begin{align}
    \|\DL\|_F^2+\|\DS\|_F^2&\lesssim{\kappa^2}\left(\barr\|\mathfrak{X}^{\star}(W)\|^2+\bars \|\mathfrak{X}^{\star} (W)\|_{\max}^2\right)\\&\lesssim \frac{1}{\sigma_{\min}^4}\left(\barr\sigma_{\max}^2\|W\|^2+\bars B^2\|W\|_{2,\infty}^2\right).
    \end{align}
\end{proof}

\subsection{Proof of Theorem \ref{thm:robust-covariance-estimation}}

The proof utilized Proposition \ref{prop:truncation-concentration}, which along with its proof, is given at the end of this section.
\begin{proof}
    Let $W=\widehat{\Sigma}-\Sigma$,
    we have from the Proposition \ref{prop:truncation-concentration} that with probability at least $1-6p^2 \exp(-ct^2)$,  $\|W\|_{\max}\le \frac{t}{\sqrt{n}}$ and $\|W\|\le p\|W\|_{\max}\le \frac{pt}{\sqrt{n}}$.

    Applying Theorem \ref{thm:main}, we have that with probability at least $1-6p^2 \exp(-ct^2)$,
    \begin{align*}
    \|\widehat{\Sigma}_L-\Sigma_L\|^2_F+\|\widehat{\Sigma}_S-\Sigma_S\|^2_F&\lesssim r\|W\|^2+s \|W\|_{\max}^2\\&\lesssim r\frac{p^2 t^2}{n}+s\frac{t^2}{n}\lesssim \frac{rp^2 t^2}{n}.
    \end{align*}

    Therefore by setting $t\asymp \log p$, we have 
    \[\|\widehat{\Sigma}_L-\Sigma_L\|^2_F+\|\widehat{\Sigma}_S-\Sigma_S\|^2_F=\mathcal{O}_p(\frac{rp^2\log^2 p}{n}).\]
    
\end{proof}

\begin{proposition}
\label{prop:truncation-concentration}
Suppose Assumption \ref{asp:fourth-moment-bounded} holds. There exists a constant $c'$, such that the following holds for every $t>0$,
\[\PP(|\widehat{\sigma}_{ij}-\sigma_{ij}|\le \frac{t}{\sqrt{n}},\forall i,j)\ge 1-6p^2\exp(-c't^2).\]
\end{proposition}

\begin{remark}
    This proposition builds Gaussian tail bounds for the individually truncated covariance. The truncation level is carefully chosen to balance between bias and variance. If we allow data correlation, e.g., some $\alpha$-mixing conditions as in \cite{fan2011high} and \cite{fan2013large}, then there is no direct way to get around the heavy-tail problem because the truncation threshold $\tau$ is of order $\sqrt{n}$, which is too large for establishing exponential tails in \cite{merlevede2011bernstein}. A procedure to allow both generalizations is left to future work.
\end{remark}

\begin{proof}
    As $\EE(X_{1i}^4)\le C$, we have $\EE |X_{1i}|$ , $\EE X_{1i}^2 X_{1j}^2$, $\EE X_{1i}^2$ are all of constant order. Let $\max_{i,j}\{\EE |X_{1i}|,\EE X_{1i}^2 X_{1j}^2,\EE X_{1i}^2\}\le C_0$.
    Using the Winsorized concentration inequality established in \cite{fan2016shrinkage}, we have when taking $\tau_1\asymp \sqrt{n}\gtrsim \sqrt{n\EE X_{1i}^2}$, it holds that 
    \[\PP(|\widehat{\EE X_{\cdot i}}-\EE X_{1i}|>t\frac{C_0}{\sqrt{n}})\le 2\exp(-ct^2).\]
    Similarly for $X_{1i}X_{1j}$, when taking $\tau_2\asymp \sqrt{n}\gtrsim \sqrt{n\EE X_{1i}^2 X_{1j}^2}$, it holds that
    \[\PP(|\widehat{\EE X_{\cdot i}X_{\cdot j}}-\EE X_{1i}X_{1j}|>t\frac{C_0}{\sqrt{n}})\le 2\exp(-ct^2)\]

    Therefore combining the above two parts we have for some constant $c'$, \begin{align*}
    &\PP(|\widehat{\sigma}_{ij}-\sigma_{ij}|>t\frac{C_0}{\sqrt{n}})\\\le &\PP(|\EE X_{1j}|\cdot|\widehat{\EE X_{\cdot i}}-\EE X_{1i}|+|\widehat{\EE X_{\cdot i}}|\cdot|\widehat{\EE X_{\cdot j}}-\EE X_{1j}|+|\widehat{\EE X_{\cdot i}X_{\cdot j}}-\EE X_{1i}X_{1j}|>t\frac{C_0}{\sqrt{n}})\\\le &6\exp(-c't^2)
    \end{align*}
    Taking a union bound and absorbing $C_0$ into $c'$ renders the result.
\end{proof}

\subsection{Proof of Theorem \ref{thm:factor-loading-estimation}}

\begin{proof} 
By Weyl's theorem, we have 
\begin{equation}
    \label{equ:eigen-perturb-weyl}|\lambda_l-\widehat{\lambda}_l|\le \|\widehat{\Sigma}_L-\Sigma_L\|.\end{equation}

To prove \eqref{equ:factor-loading-estimation-1}, by the Davis-Kahan theorem~\citep{davis1970rotation} and the assumption on the minimum eigenvalue, it holds that there exists some $R\in \calO_{r,r}$ such that
\begin{align*}
\|VR-\hat V\|\le \frac{\sqrt{2}}{p}\|\widehat{\Sigma}_L-\Sigma_L\|
\end{align*}
Set $H=\Sigma^{-1/2}R\Sigma^{1/2}$. As $B=V\Sigma^{1/2}$ and $\hat B=\hat V\hat\Sigma^{1/2}$, it follows that
\begin{align*}
\|BH-\hat B\|&=\|VR\Sigma^{1/2}-\hat V\hat\Sigma^{1/2}\|\\&\le \|VR\Sigma^{1/2}-\hat V\Sigma^{1/2}\|+\|\hat V\Sigma^{1/2}-\hat V\hat\Sigma^{1/2}\|\\&\lesssim \sqrt{p}\|VR-\hat V\|+\|\Sigma^{1/2}-\hat \Sigma^{1/2}\|\\&\lesssim \frac{1}{\sqrt{p}}\|\widehat{\Sigma}_L-\Sigma_L\|+\frac{\|\widehat{\Sigma}_L-\Sigma_L\|}{\lambda_r^{1/2}}\\&\lesssim \frac{1}{\sqrt{\lambda_r}}\|\widehat{\Sigma}_L-\Sigma_L\|,
\end{align*}
where the last inequality follows from $\lambda_r\lesssim p$, as implied by the boundedness of of $\EE X_{1i}^2$.
Hence from Theorem \ref{thm:robust-covariance-estimation} we have with probability at least $1-n^{-c}$ that
\[\|BH-\hat B\|\lesssim \frac{rp^2\log^2 (p)\log n}{n\lambda_r},\]
which completes the proof of \eqref{equ:factor-loading-estimation-1}.

To prove \eqref{equ:factor-loading-estimation-2}, we again
invoke the Davis-Kahan theorem, \begin{equation}
    \label{equ:eigen-perturb-davis-kahan}\min_{e_l=\pm1}\|v_l-e_l\widehat{v}_l\|\lesssim \frac{1}{p}\|\widehat{\Sigma}_L-\Sigma_L\|\end{equation}
As the result is up to a rotational ambiguity, we can set the rotation matrix $H$ to be a diagonal matrix with $l^{th}$ diagonal element equal to $e_l$. It is straightforward to check that applying $H$ on $B$ helps align $B$ with $\hat B$. Therefore, without loss of generality, we can assume $e_l=1, \ \forall l$, and $H=\bI$. It boils down to bound $\sum_{j=1}^p\|b_j-\widehat{b}_j\|^2$.

Denote $v_{lj}$ as the $j^{th}$ coordinate of $v_l$, we have 
\begin{align*}
    \|b_j-\widehat{b}_j\|^2&=\sum_{l=1}^r(\lambda_l^{\frac{1}{2}}v_{lj}-\widehat{\lambda}_l^{\frac{1}{2}}\widehat{v}_{lj})^2\\&
\lesssim\sum_{l=1}^r(\lambda_l^{\frac{1}{2}}v_{lj}-\lambda_l^{\frac{1}{2}}\widehat{v}_{lj})^2+\sum_{l=1}^r(\lambda_l^{\frac{1}{2}}\widehat{v}_{lj}-\widehat{\lambda}_l^{\frac{1}{2}}\widehat{v}_{lj})^2\\
    &=\sum_{l=1}^r \lambda_l(v_{lj}-\widehat{v}_{lj})^2+\sum_{l=1}^r\widehat{v}_{lj}^2(\lambda_l^{\frac{1}{2}}-\widehat{\lambda}_l^{\frac{1}{2}})^2 
\end{align*}

By \eqref{equ:eigen-perturb-weyl} and \eqref{equ:eigen-perturb-davis-kahan}, it holds that \[\sum_{l=1}^r\sum_{j=1}^p \lambda_l(v_{lj}-\widehat{v}_{lj})^2\le r\lambda_1\frac{1}{p^2}\|\widehat{\Sigma}_L-\Sigma_L\|^2\lesssim r^2\frac{1}{p}\|\widehat{\Sigma}_L-\Sigma_L\|^2\] and 

\[\lambda_l^{\frac{1}{2}}-\widehat{\lambda}_l^{\frac{1}{2}}\le \frac{\|\widehat{\Sigma}_L-\Sigma_L\|}{\lambda_l^{\frac{1}{2}}+\widehat{\lambda}_l^{\frac{1}{2}}}\lesssim \frac{\|\widehat{\Sigma}_L-\Sigma_L\|}{\sqrt{p}}.\]

Also,  as $\Sigma_L=BB^{\top}$ satisfies the $\mu$-incoherence condition, we have that $|v_{lj}|\le \|v_j\|_2\le \sqrt{\frac{\mu r}{p}}$. Therefore we have
\[|\widehat{v}_{lj}|\le \|\widehat{v}_{l}-v_{l}\|+|v_{lj}|\le \frac{1}{p}\|\widehat{\Sigma}_L-\Sigma_L\|+\sqrt{\frac{\mu r}{p}},\]
which implies $|\widehat{v}_{lj}|\lesssim \sqrt{\frac{\mu r}{p}}$ for $n\gtrsim \frac{p\log^2 p}{\mu}$.

Putting together we get with probability at least $1-n^{-c}$,
\begin{align*}
\sum_{j=1}^p\|b_j-\widehat{b}_j\|^2&\lesssim r^2\frac{1}{p}\|\widehat{\Sigma}_L-\Sigma_L\|^2+rp\frac{\mu r}{p}\frac{1}{p}\|\widehat{\Sigma}_L-\Sigma_L\|^2\\&\lesssim r^2\frac{1}{p}\|\widehat{\Sigma}_L-\Sigma_L\|_F^2\\
&\lesssim \frac{r^3 p\log^2 (p)\log n}{n}.
\end{align*}

\end{proof}

\section{Extension to rectangular matrix estimand}
\label{sec:extension-rectangular}
In this section, we extend our main result Theorem \ref{thm:main} to rectangular matrix estimand. Specifically, the optimization problem \eqref{alg:main} remains the same except with a different parameter dimension\footnote{And therefore the constraint  $V\in \mathcal{O}_{p,\barr}^{\barmu}$ in \eqref{alg:main} is replaced with $V\in \mathcal{O}_{q,\barr}^{\barmu}$.}. A useful fact is that the key lemma continues to hold by reducing to the square matrix case.

\begin{assumption}
    \label{asp:L,S-nonsquare}
     The matrices $L^{\star}$ and  $S^{\star}$ statisfy the following conditions.
    \begin{enumerate}
    \item In the SVD $L^{\star}=U^{\star} \Sigma^{\star} V^{*T}$, we have $U^{\star}\in \mathcal{O}^{\mu}_{p,r},\ V^{\star}\in \mathcal{O}^{\mu}_{q,r}$;
    \item $\|S^{\star}\|_0\le s$;
    \item $s\le \bars\le \frac{\max\{p,q\}}{c\barmu\barr^4}$.
\end{enumerate}
\end{assumption}

\begin{corollary}
    Consider the case that $Y,L,S$ have dimension $p\times q$ instead of $p\times p$. Under Assumption \ref{asp:RSC},\ref{asp:RSC-parameters},\ref{asp:L,S}, \eqref{equ:error-bound} still holds.
\end{corollary}

The proof of the corollary essentially follows the proof of Theorem \ref{thm:main}.  It suffices to prove the following variant of the key lemma in rectangular case.

\begin{lemma}
\label{lemma:main-nonsquare}
    Let $P=U \Sigma V^{\top}$ and $Q=L \Lambda R^{\top}$ be the SVD of $P,Q\in \mathbb{R}^{p\times q}$, $U,L\in \mathcal{O}_{p,r},\ V,R\in \mathcal{O}_{q,r}$. Denote $\sigma_1\ge \sigma_2\ge\cdots\ge\sigma_r\ge 0$ and $\lambda_1\ge \lambda_2\ge\cdots\ge\lambda_r\ge 0$ as the diagonal elements of $\Sigma$ and $\Lambda$, respectively.
    Suppose $P$ and $Q$ both satisfy the $\mu$-incoherence condition \ref{def:incoherence}, then we have for $\Delta=P-Q$,
    \begin{equation}\frac{\|\Delta\|_{\max}^2}{\|\Delta\|_F^2}\le \frac{\Tilde{c}\mu r^4}{\max\{p,q\}}\end{equation}

where $\Tilde{c}$ is a fixed constant independent of both $P$ and $Q$.
\end{lemma}

\begin{proof}
    Without loss of generality, assume $p>q$.  Define padding matrices $\tP=[P,0_{p,p-q}],\ \tQ=[Q,0_{p,p-q}]$ and $\widetilde{\Delta}=\tP-\tQ$. It is straightforward to verify that the SVD of $\tP, \tQ$ are $\tP=U\Sigma \tV^\top $ and $\tQ=L\Sigma \tR^\top $, where $\tV=\begin{bmatrix}
        V\\0_{p-q,r}
    \end{bmatrix}$ and $\tR=\begin{bmatrix}
        R\\0_{p-q,r}
    \end{bmatrix}$.

    By the definition of the incoherence, we have the incoherence of $\tV$ and $V$ are equal and therefore, $\tV\in \mathcal{O}_{p,r}$. Similarly, we have $\tR\in \mathcal{O}_{p,r}$. Apply the key lemma \ref{lemma:main} with $P,Q$ replaced by $\tP,\tQ$, we get \[\frac{\|\Delta\|_{\max}^2}{\|\Delta\|_F^2}=\frac{\|\widetilde{\Delta}\|_{\max}^2}{\|\widetilde{\Delta}\|_F^2}\le\frac{\Tilde{c}\mu r^4}{p}=\frac{\Tilde{c}\mu r^4}{\max\{p,q\}}\] 
    where the last inequality follows from our assumption that $p>q$.
    
\end{proof}

\section{Other Technical Lemmas}
\label{sec:technical}
\subsection{Proof of Lemma \ref{lemma:new-consecutive-markov-chain}}

\begin{proof}
    We begin by proving that $Y$ is recurrent. In fact, for $(i,j),(k,l)\in [p]\times [p]\backslash Z$, we have $P_{i,j}>0,\ P_{k,l}>0$ by definition of $Z$. As $X$ is recurrent, there exists a $n_{j,k}$ such that $\PP(X_{n+1}=k\big|X_1=j)>0$, hence $\PP(X_{n+1}=k,X_{n+2}=l\big|X_1=j,X_0=i)>0$. Equivalently, $\PP(Y_{n+1}=(k,l)\big|Y_0=(i,j))>0$.

    For aperiodicity, for $(i,j)\in [p]\times [p]\backslash Z$, we have $P_{i,j}>0$. Therefore, for each $n$ such that $\PP(X_n=i\big|X_0=j)>0$, we have $\PP(X_{n+1}=i,X_{n+2}=j\big|X_0=i,X_1=j)>0$. Aperiodicity of $Y$ follows from the aperiodicity of $X$.

    Now recall the definition of mixing time which compares the state distribution starting from a fixed initial state or the stationary distribution in 1-norm. A little algebra shows that the stationary distribution of $Y$ is given by \begin{equation}
    \label{equ:stationary-new-consecutive-mc}\mu_{(k,l)}=\pi_{k}P_{k,l}.\end{equation} Consider an arbitrary fixed initial state $(i,j)\in [p]\times [p]\backslash Z$. We can calculate the $n$ step state distribution $e_{(i,j)}Q^n$ as 
    \[\big[e_{(i,j)}Q^n\big]_{(k,l)}=(P^n)_{j,k}P_{k,l}\]

    Therefore, 
    \begin{align*}\|\mu-e_{(i,j)}Q^n\|_1&=\sum_{k=1}^p\sum_{l=1}^p |(P^n)_{j,k}P_{k,l}-\pi_{k}P_{k,l}|\\
    &=\sum_{k=1}^p |(P^n)_{j,k}-\pi_{k}|\sum_{l=1}^p P_{k,l}\\
    &=\sum_{k=1}^p |(P^n)_{j,k}-\pi_{k}|\\
    &=\|e_j^{\top} P^n-\pi\|_1
    \end{align*}

    Therefore,
    \begin{align*}\tau_Y(\epsilon)&=\min \left\{n: \max_{(i,j)} \frac{1}{2}\left\|\mu-e_{(i,j)}Q^n\right\|_1 \leq\epsilon\right\}\\&=\min \left\{n: \max_{(i,j)} \frac{1}{2}\|e_j^{\top} P^n-\pi\|_1\leq\epsilon\right\}=\tau_X(\epsilon)
    \end{align*}
\end{proof}

\subsection{Gilbert-Varshamov bound}
\begin{lemma}
\label{lemma:gilbert-varshamov}[Gilbert-Varshamov bound]
    There exists a constant $c_0$, for each $m\ge 6$, there exists a subset $A$ of $\{-1,1\}^m$ consisting of at least $\exp(c_0 m)$ vertices such that each pair is greater than $\frac{m}{3}$ apart in Hamming distance.
\end{lemma}

\subsection{A simple lemma}
The following is a simple lemma regarding relations between elements of $P$ and $\pi$, whose proof follows from simple algebra and is omitted.
\begin{lemma}
    Suppose $P$ and $\pi$ correspond to the transition matrix and stationary distribution to a Markov chain. Then we have 
    \[P_{\min}\le \pi_{\min}\le \pi_{\max}\le P_{\max}, \]
    where $P_{\max}=\max P_{i,j}$, $P_{\min}=\min P_{i,j}$, $\pi_{\max}=\max \pi_i$, and
    $\pi_{\min}=\min \pi_i$.
\end{lemma}

\end{appendix}

\newpage
\bibliographystyle{apalike2}
\bibliography{main}

\end{document}